\documentclass[11pt]{article}

\usepackage[T1]{fontenc}
\usepackage[utf8]{inputenc}
\usepackage{amsmath,amssymb,amsthm}
\usepackage{mathtools}
\usepackage{bm}
\usepackage{graphicx}
\usepackage{mathrsfs}

\usepackage{array}
\usepackage{booktabs}
\usepackage{multirow}
\usepackage{longtable}
\usepackage{makecell}
\usepackage{supertabular}

\usepackage{algorithm}
\usepackage{algpseudocode}

\usepackage{pdflscape}
\usepackage{placeins}
\usepackage{float}

\usepackage{listings}
\usepackage{url}
\usepackage{hyperref}
\hypersetup{
	colorlinks = true,
	linkcolor  = blue,
	citecolor  = blue,
	urlcolor   = cyan,
	pdftitle   = {Your Title},
	pdfauthor  = {Your Name}
}

\usepackage{xcolor}
\usepackage{enumitem}
\usepackage{relsize}
\usepackage{cases}
\usepackage{tikz}
\usepackage{endnotes}

\usepackage{natbib}

\newcommand{\coloneq}{\coloneqq}

\newcommand{\halmos}{\ensuremath{\square}}

\newtheorem{theorem}{Theorem}
\newtheorem{lemma}{Lemma}
\newtheorem{corollary}{Corollary}
\newtheorem{proposition}{Proposition}

\newtheorem{remark}{Remark}
\newtheorem{assumption}{Assumption}



\allowdisplaybreaks[4]

\begin{document}
	\title{Efficient Group Lasso Regularized Rank Regression with Simulation-Based Tuning}

	\author{
		Meixia Lin \\
		School of Statistics and Data Science, Renmin University of China, P.R.\ China \\
		\href{mailto:lin_meixia@ruc.edu.cn}{\texttt{lin\_meixia@ruc.edu.cn}}
		\and
		Mengjiao Shi \\
		School of Mathematics and Statistics, Henan University, P.R.\ China \\
		\href{mailto:smj@henu.edu.cn}{\texttt{smj@henu.edu.cn}}
		\and
		Yunhai Xiao\textsuperscript{*} \\
		School of Mathematics and Statistics \& Center for Applied Mathematics of Henan Province,\\
		Henan University, P.R.\ China \\
		\href{mailto:yhxiao@henu.edu.cn}{\texttt{yhxiao@henu.edu.cn}}
		\thanks{* Corresponding author.}
		\and
		Qian Zhang \\
		Engineering Systems and Design, Singapore University of Technology and Design, Singapore \\
		\href{mailto:qian_zhang@sutd.edu.sg}{\texttt{qian\_zhang@sutd.edu.sg}}
	}
	
	\date{}
	\maketitle
	

	\begin{abstract}
		High-dimensional regression often suffers from heavy-tailed noise and outliers, which can severely  undermine the reliability of least-squares based methods. To improve robustness, we adopt a non-smooth Wilcoxon score based rank objective and incorporate the group sparsity regularization. By extending the tuning-free property originally developed for the rank Lasso, we introduce a simulation-based tuning rule and further establish a finite-sample error bound for the resulting estimator. To solve the associated optimization problem,
		we develop a proximal augmented Lagrangian method, for which we provide a novel convergence analysis by proving the metric subregularity of the underlying non-polyhedral KKT mapping, while enabling efficient semismooth Newton updates for the subproblems. Extensive numerical experiments demonstrate the robustness and effectiveness of our proposed estimator against several leading alternatives, and showcase the efficiency and scalability of our algorithm compared to the state-of-the-art baseline in both simulated and real-data settings.
		
	\end{abstract}

	\section{Introduction}
	Consider the standard linear model
	\begin{equation}\label{mod:linear regression}
		y = X \beta^* + \epsilon,
	\end{equation}
	where $y\in \mathbb{R}^n$ is the response vector, $X \in \mathbb{R}^{n \times p}$ is the design matrix with rows $X_i$ representing the covariate vector for the $i$-th observation, $\beta^*\in \mathbb{R}^p$ is the true coefficient vector, and $\epsilon\in \mathbb{R}^n$ denotes the random error vector. Ordinary least squares is the classical approach for estimating $\beta^*$ and works well under light-tailed noise like Gaussian errors. However, its reliance on the squared loss makes it highly sensitive to heavy-tailed errors and outliers. In particular, a single extreme value can dominate the loss function, leading to highly unstable estimates and compromised predictive performance. Yet, heavy-tailed errors are common in modern high-dimensional data, for example in genomics \citep{Wang2015} and neuroimaging \citep{Eklund2016}, underscoring the need for robust regression methods that remain reliable in such settings.
	
	To address this issue, a variety of robust estimation techniques have been developed. One line of research employs robust loss functions, which mainly rely on truncation or downweighting strategies. For example, the Huber estimator \citep{huber1973robust,sun2020adaptive} is a classic truncation-based  method, which clips gradients beyond a threshold to enforce linear growth, thereby diminishing the influence of outliers. In addition, Tukey's biweight estimation \citep{Beaton01051974,Huber2011} exemplifies the downweighting approach by adopting a bisquare function that reduces the weights of large residuals to zero. While these methods effectively alleviate the impact of outliers via gradient modification, they suffer from an identifiability issue as the global minimizers of the modified losses may not coincide with the true parameter vector \citep{Fan2016}. An alternative direction is quantile regression \citep{Koenker1978RegressionQ,Koenker_2005}, which replaces the squared loss with an asymmetric linear loss to estimate conditional quantiles, thereby providing robustness against heavy-tailed noise. However, its efficiency suffers when the specified quantile lies in a sparse region, making quantile estimation particularly challenging. To overcome this limitation, rank-based methods \citep{Hettmansperger2010} make use of rank-induced functionals, such as Jaeckel's dispersion measure \citep{Jaeckel1972}, to stabilize estimation at sparse quantiles. In particular, Wilcoxon score based rank estimators achieve strong consistency and asymptotic normality, with variance depending only on the rank-transformed error distribution. This robustness to heavy-tailed noise is achieved by minimizing the following rank loss:
	{\setlength{\abovedisplayskip}{2pt}
		{\setlength{\belowdisplayskip}{2pt}
			\begin{equation}\label{loss:Wilcoxon}
				L(X\beta-y):=\frac{1}{n(n-1)} \sum_{i=1}^{n}\sum_{j\neq i}\left|(X\beta-y)_i-(X\beta-y)_j\right|. 
			\end{equation}
			The pairwise differencing structure eliminates the intercept, making the estimator invariant to location shifts in $\epsilon_i$, so that no explicit centering assumption on the error distribution is needed. Moreover, since the distribution of $\epsilon_i - \epsilon_j$ is symmetric about zero when the errors are independent and identically distributed, the expected loss $E \left(|\epsilon_i - \epsilon_j - (X_i - X_j)^\intercal(\beta - \beta^*)|\right)$ is minimized at $\beta=\beta^*$. Replacing the expectation by its empirical counterpart yields \eqref{loss:Wilcoxon}. Thus minimizing \eqref{loss:Wilcoxon} provides a principled sample analogue whose minimizer consistently estimates $\beta^*$.
			
			
			In high-dimensional settings, imposing additional regularization on \eqref{loss:Wilcoxon} is essential for variable selection and for incorporating prior structural knowledge among coefficients. For instance, \cite{Wang2020ATR} introduced a Lasso regularized estimator, with a simulation-based parameter determination rule that adapts to both the design matrix and the unknown error distribution. Although the Lasso penalty effectively induces sparsity in the coefficients, it fails to encode structural constraints such as group sparsity. This capability is critical in a wide range of applications, from biomedical imaging \citep{janjusevic2025groupcdl} to deep learning \citep{lu2025growing}.
			
			To address this issue, we propose a group Lasso regularized rank regression estimator that incorporates the group Lasso penalty \citep{yuan2006model} to induce group-level sparsity:
			\begin{equation}\label{mod:rankgrouplasso0}
				\hat{\beta}=\underset{\beta\in \mathbb{R}^p}{\arg\min} \ \left\{
				L(X\beta-y) + \lambda\Psi(\beta)
				\right\},
			\end{equation}
			where the rank loss function $L(X\beta-y)$ is defined in \eqref{loss:Wilcoxon} and the group Lasso regularizer $\Psi:\mathbb{R}^p \to\mathbb{R}$ is defined as
			\begin{equation*}
				\Psi(\beta) = \sum_{l=1}^{g} w_l \|\beta_{\mathcal{G}_l}\|_2.
			\end{equation*}
			Here, $\{\mathcal{G}_l\}_{l=1}^g$ denotes a partition of the index set $\{1, \ldots, p\}$ into $g$ predefined, non-overlapping groups, with each group $\mathcal{G}_l$ assigned a positive weight $w_l > 0$. The notation $\beta_{\mathcal{G}_l}$ denotes the subvector of $\beta$ corresponding to the indices in $\mathcal{G}_l$. 
			The regularization parameter $\lambda > 0$ controls the trade-off between the goodness-of-fit and the group-wise sparsity. Moreover, we develop a simulation-based tuning rule for $\lambda$ by extending the framework of \cite{Wang2020ATR} from the Lasso to the group Lasso setting. The proposed parameter selection eliminates the need for cross-validation, while being both adaptive to the group structure and computationally efficient.
			Furthermore, we establish  a finite-sample error bound for the resulting estimator, showing validity even with heavy-tailed errors. Notably, this extension is nontrivial, as it requires adapting the pivotal tuning argument to the group lasso dual norm and developing the analysis under a group-structured sparsity framework.
			
			To make the proposed estimator \eqref{mod:rankgrouplasso0} practically useful, it is essential to design an efficient algorithm for the involved optimization problem. Existing progress on solving the regularized rank regression problems is rather limited. For the Lasso regularized rank regression, \cite{Wang2020ATR} reformulated the rank loss into a linear program with $\mathcal{O}(n^2)$ constraints by introducing slack variables, and then solved it by generic LP solvers. While conceptually straightforward, this approach suffers from poor scalability as the sample size $n$ grows. A more recent advance is the proximal–proximal majorization–minimization (PPMM) algorithm \citep{tang2023proximal}, originally developed for rank regression with difference-of-convex penalties. The method follows a standard linearization scheme, where the concave part of the regularizer is majorized by its affine approximation, yielding at each iteration a convex subproblem. These subproblems are then solved using a nested triple-loop scheme: an outer proximal point iteration, a middle proximal point loop applied to the dual of each subproblem, and an inner semismooth Newton iteration. This complicated structure arises from the singularity of the generalized Hessian in the dual formulation of the outer subproblem, which prevents the direct application of Newton-type methods. As a special case, the framework applies to Lasso-regularized rank regression, where it was shown to significantly outperform standard solvers such as ADMM and Gurobi (see \citet[Section V.A.]{tang2023proximal}). While PPMM achieves reasonable numerical performance, the computational overhead associated with three nested loops limits its scalability in high-dimensional problems. 
			
			Motivated by these limitations, our goal is to
			design an efficient solver that addresses the regularized rank regression problem, with the group Lasso regularized case \eqref{mod:rankgrouplasso0} being a notable example, preserving the fast linear convergence properties of proximal point algorithms (or, equivalently, dual augmented Lagrangian methods) without the complexity of such a heavy triple-loop design. To this end, we develop a proximal augmented Lagrangian method (PALM) for solving the dual of \eqref{mod:rankgrouplasso0}, equipped with a semismooth Newton (SSN) solver for the resulting subproblems. By augmenting the dual with a proximal term, PALM regularizes the subproblems and eliminates the singularity issue, enabling effective SSN updates which can fully exploit second-order sparsity for highly efficient computation.
			Notably, existing superlinear convergence results for PALM \citep{li2020asymptotically} rely on a local error bound condition of the KKT residual mapping, which is not readily verifiable in our setting due to the non-polyhedral group Lasso regularizer. We instead establish the metric subregularity of the KKT residual mapping, under which superlinear convergence of the proposed algorithm follows. The proof of this property is highly nontrivial in the present setting and, to the best of our knowledge, has not been established in the existing literature. As a byproduct, our analysis also yields the metric subregularity of the KKT mapping associated with the group Lasso regularizer, which may be of independent interest for PALM-type methods in structured sparsity problems. 
			On the computational side, unlike subsampling or resampling strategies
			\citep{Wang2020ATR, Fan2020commentATR}, we reformulate the rank loss term as a weighted ordered Lasso, which reduces the per-iteration complexity from $\mathcal{O}(n^2)$ to $\mathcal{O}(n \log n)$ and simultaneously facilitates efficient computation of the proximal mapping and the generalized Hessian.  
			Together, these yield a practical and theoretically grounded solver for high-dimensional regularized rank regression with structured sparsity. 
			
			Extensive numerical experiments demonstrate that the proposed estimator a\-chieves strong statistical accuracy and robustness across a range of data-generating scenarios. In particular, we compare against a variety of competitive baselines, including least-squares and rank-based estimators with Lasso or group Lasso regularization, as well as robust group Lasso methods based on Huber and quantile losses. Regarding computational performance, we benchmark our proposed algorithm against the state-of-the-art PPMM method, and show that our algorithm achieves significantly improved computational efficiency.
			
			The remainder is organized as follows. In Section~\ref{sec:2}, we introduce a simulation-based tuning rule for determining the regularization parameter $\lambda$ in model~\eqref{mod:rankgrouplasso0}, and establish a finite-sample error bound. In Section~\ref{sec:3}, we develop an efficient inexact PALM framework for solving problem~\eqref{mod:rankgrouplasso0}, with a detailed convergence analysis. Numerical results are presented in Section~\ref{sec:4}, evaluating both the statistical performance of the proposed estimator, and the computational efficiency of PALM against state-of-the-art solvers. Finally, we conclude our work in Section~\ref{sec:5}.

			\vspace{0.1cm}
			\noindent \textbf{Notation:} For a positive integer $m$, let $[m] = \{1, \dots, m\}$. For an index set $\mathcal I$, let $A_{\mathcal I}$ denote the subvector or submatrix indexed by $\mathcal I$, with the precise meaning (row, column, or entry selection) clear from the context. 
			Denote by $\overline{\mathcal I}$ the complement of $\mathcal I$ and by $|\mathcal I|$ its cardinality. We denote the all-ones (all-zeros) vector in $\mathbb{R}^p$ by $\mathbf{1}_p$ ($\mathbf{0}_p$), and omit the subscript when it is clear from the context. Let $e_i$ denote the $i$-th canonical basis vector, i.e., the vector with $1$ in the $i$th entry and $0$ elsewhere. Denote the identity matrix in $\mathbb{R}^{p\times p}$ by $I_p$. Denote by $\mathbb{I}\{E\}$ the indicator function of an event $E$.
			Denote the Euclidean norm by $\|\cdot\|_2$, and for a positive definite matrix $M$, define $\|x\|_M := \sqrt{x^\intercal M x}$ and $\mathrm{dist}_M(x,S) := \inf_{y\in S} \|x - y\|_M$.
			For any set $C$, denote by $\operatorname{ri}(C)$ the relative interior of $C$. If $C$ is closed, the projection onto $C$ is defined by $\Pi_{C}(x)\coloneq \arg\min\{\|x-y\|: y\in C\}$. 
			Let $\operatorname{sign}(\cdot)$ denote 
			the element-wise sign function: for any scalar $t$, $\operatorname{sign}(t) = 1$ if $t>0$, 
			$0$ if $t=0$, and $-1$ if $t<0$. For a given point $x\in\mathbb{R}^p$, the open $\ell_q$ ball of radius $r>0$ centered at $x$ is defined as $\mathcal{B}_q^r(x) := \{y\in\mathbb{R}^p \mid \|y-x\|_q < r\}$ for $q\in\{1,2,\infty\}$, and we simply denote $\mathcal{B}_q^r := \mathcal{B}_q^r(\mathbf 0)$. 
			For a set-valued mapping $F$, its graph is defined by $\mathrm{gph}(F) := \{(x,y)\mid y\in F(x)\}.$
			The \emph{Moreau envelope} of a proper, closed, 
			and convex function $f: \mathbb{R}^n \to \mathbb{R} \cup \{+\infty\}$ with $\sigma > 0$ is
			\begin{equation}\label{def:moreau envelope}
				e_{\sigma}f(x) := \min_{y \in \mathbb{R}^n} \left\{ f(y) + \frac{1}{2\sigma} \|y - x\|_2^2 \right\}, 
				\quad x \in \mathbb{R}^n.
			\end{equation}
			Following \cite{Moreau1965Proximite}, the Moreau envelope is continuously differentiable with gradient 
			$\nabla e_{\sigma}f(x) = (x - \operatorname{Prox}_{\sigma f}(x))/\sigma$, where the proximal operator
			$\operatorname{Prox}_{\sigma f}(x)$ is the unique minimizer of \eqref{def:moreau envelope} and is nonexpansive ($1$-Lipschitz) \citep{Nocedal2006Numerical,Rockafellar1976}.
			
			\section{Simulation-Based Tuning with Theoretical Guarantee}\label{sec:2}
			We propose a simulation-based determination rule for the regularization parameter $\lambda$ in the group Lasso regularized rank regression model~\eqref{mod:rankgrouplasso0}, and then establish a finite sample bound on the resulting estimation error. Our result extends the analysis in \cite{Wang2020ATR}, which focuses on Lasso regularized rank regression, to a broader setting that incorporates structured sparsity via the group Lasso.
			
			\subsection{Simulation-Based Parameter}
			Define the estimation error $\gamma := \beta - \beta^*$. From \eqref{mod:linear regression}, for any $\beta \in \mathbb{R}^p$, we have 
			\[
			X\beta - y = X(\beta - \beta^*) - \epsilon = X\gamma - \epsilon.
			\]
			Then minimizing the objective in \eqref{mod:rankgrouplasso0} is equivalent to solving
			\begin{equation}
				\min_{\gamma\in \mathbb{R}^p} \ \left\{ F_{\lambda}(\gamma):=L_0(\gamma) + \lambda \Psi(\gamma + \beta^*)\right\}, \label{eq: gamma_prob}
			\end{equation}
			where the shifted loss function $L_0:\mathbb{R}^p \rightarrow \mathbb{R}$ is defined as
			\begin{equation}
				L_0(\gamma) := \frac{1}{n(n-1)} \sum_{i=1}^{n} \sum_{j \neq i} \left|(\epsilon_i - \epsilon_j) - (X_i - X_j)^\intercal \gamma\right|.\label{eq: def_L0}
			\end{equation}
			This reformulation plays a central role in deriving both the simulation-based parameter determination rule and the corresponding finite-sample error bound.
			
			Let $S_n$ denote the subgradient of $L_0(\cdot)$ evaluated at $\gamma = \mathbf 0$. The following result, adapted from \citet{Wang2020ATR}, characterizes the distribution of the random vector $S_n$, which is shown to be independent of the error distribution.
			
			\begin{lemma}\label{lemma:pivotol}
				Under the linear model \eqref{mod:linear regression}, the subgradient of $L_0(\cdot)$ in \eqref{eq: def_L0} at $\gamma = \mathbf 0$ is
				\begin{equation}
					S_n = -\frac{2}{n(n-1)} X^\intercal \xi, \label{eq: def_sn}
				\end{equation}
				where $\xi = 2r - (n+1)$ and $r$ is 
				uniformly distributed over all permutations of $[n]$. 
			\end{lemma}
			
			Motivated by this completely pivotal property of $S_n$, we propose the following 
			simulation-based tuning rule for problem \eqref{mod:rankgrouplasso0}:
			\begin{equation}
				\lambda^* = c_0 \cdot Q_{\Psi^d(S_n)}(1 - \alpha_0),\label{eq: opt_lambda}
			\end{equation} 
			where $c_0>1$ and $0<\alpha_0<1$ are user-specified constants, and $Q_{\Psi^d(S_n)}(1-\alpha_0)$ is the $(1-\alpha_0)$-quantile of the random variable $\Psi^d(S_n)$. Here, $\Psi^d(\cdot)$ denotes the dual norm of $\Psi(\cdot)$, which takes the form of
			\begin{equation}
				\Psi^d(y) := \sup_{\Psi(x) \leq 1} \langle x, y \rangle = \max_{1 \leq l \leq g} \frac{\|y_{\mathcal{G}_l}\|_2}{w_l}.\label{eq: psi_dual}
			\end{equation}
			Notably, the proposed $\lambda^*$ is distribution-free by Lemma~\ref{lemma:pivotol}. Moreover, it can be computed via Monte Carlo simulation (see Algorithm~\ref{alg:lambda_estimation}) with an overall computational complexity of $\mathcal O(Knp)$, where $K$ denotes the pre-specified number of Monte Carlo samples. 

			\begin{algorithm}
				\caption{Simulation-based choice of regularization parameter $\lambda$}
				\label{alg:lambda_estimation}
				\begin{algorithmic}[1]
					\Require Design matrix $X \in \mathbb{R}^{n \times p}$, predefined groups $\{\mathcal{G}_l\}_{l=1}^g$ with weights $\{w_l\}_{l=1}^g$, quantile level $0 < \alpha_0 < 1$, safety factor $c_0 > 1$, simulation number $K(=500)$.
					\vspace{-0.9\baselineskip}
					\Statex
					\For{$k = 1$ to $K$}
					\State Generate a random permutation $r^{(k)}$ of $[n]$ and compute $\xi^{(k)} = 2 r^{(k)} - (n + 1)$.
					\State Compute subgradient $S_n^{(k)} = -\frac{2}{n(n-1)} X^\intercal \xi^{(k)}$.
					\State Evaluate $\Psi^d(S_n^{(k)}) = \max_{1 \leq l \leq g} \left(\|(S_n^{(k)})_{\mathcal{G}_l}\|_2/w_l\right)$.
					\EndFor
					\State Let $q_{1 - \alpha_0}$ denote the empirical $(1 - \alpha_0)$-quantile of $\{\Psi^d(S_n^{(k)})\}_{k=1}^K$.
					\State \Return  $\lambda^*= c_0 \cdot q_{1 - \alpha_0}$.
				\end{algorithmic}
			\end{algorithm}
			
			\begin{remark}
				In Algorithm~\ref{alg:lambda_estimation}, the constant $c_0$ serves as  a safety factor for the regularization, whereas $\alpha_0$ specifies the confidence level for the established error bound, see Theorem \ref{th:consistency}. 
				In practice, values such as $c_0 = 1.01$ and $\alpha_0 = 0.1$ are simple and effective choices.
			\end{remark}

			\subsection{Finite-Sample Bound on Estimation Error}\label{sec:finite_bound}
			We now establish a finite-sample bound on the estimation error of the resulting estimator.  To facilitate the analysis, we introduce the following notations. Let $\mathcal{S}_0 := \{j \in [p] : (\beta^*)_j \neq 0\}$ denote the support of the true parameter $\beta^*$, with cardinality $q = |\mathcal{S}_0|$. 
			Define the set of active group indices and the associated group-expanded support as
			\begin{equation*}
				\mathcal{A}_0 := \{l \in [g] : \mathcal{G}_l \cap \mathcal{S}_0 \neq \emptyset\}
				\quad \text{and} \quad
				\Omega := \bigcup_{l \in \mathcal{A}_0} \mathcal{G}_l.
			\end{equation*}
			Define 
			$
			\Psi_{\Omega}(\beta) := \Psi(\mathcal{P}_{\Omega}(\beta))$ for any $\beta\in \mathbb{R}^p$,
			where $\mathcal{P}_{\Omega}(\beta)$ denotes the projection of $\beta$ onto $\Omega$, that is, a vector in $\mathbb{R}^p$ that agrees with $\beta$ on the coordinates in $\Omega$ and is zero elsewhere.
			Similarly, we define $\Psi_{\bar{\Omega}}(\beta) := \Psi(\mathcal{P}_{\bar{\Omega}}(\beta))$.
			This decomposition gives rise to the following key properties, which are essential for the finite-sample error analysis. For any $\beta\in \mathbb{R}^p$, we have
			\begin{align}
				&\Psi(\beta) = \Psi_{\Omega}(\beta) + \Psi_{\bar{\Omega}}(\beta), \label{eq:decomposition} \\
				&\Psi_{\Omega}(\beta)  \leq c_{|\Omega|}\|\beta\|_2, \quad \text{with} \quad c_{|\Omega|} = \sqrt{\sum\nolimits_{l\in \mathcal{A}_0}w_l^2}, \label{ineq:cp} \\
				&\Psi^d(\beta)  \leq c_p\|\beta\|_\infty, \quad \text{with} \quad c_p = \max\nolimits_{1\leq l\leq g} w_l^{-1}\sqrt{|\mathcal{G}_l|}. \label{ineq:co}
			\end{align}
			
			To establish finite-sample guarantees for the proposed estimator, we impose the following technical conditions on the design matrix and error distribution.
			\begin{assumption}\label{assump:error}
				We assume the following conditions.
				\begin{enumerate}[label=(A\arabic*)]
					\item \label{assumption_X} 
					\textbf{Design Matrix Condition.} 
					The covariates are empirically centered, i.e., $\sum_{i=1}^n X_i=\mathbf 0$, and uniformly bounded by a constant $b_1 > 0$, i.e., $|X_{ij}| \leq b_1$ 
					for all $i, j$. Furthermore, the design matrix satisfies the restricted eigenvalue (RE) condition
					\[
					\inf_{\gamma \in \mathbb{R}^p} \left\{ \frac{\gamma^\intercal X^\intercal X \gamma}{n\|\gamma\|_2^2} : \Psi_{\bar{\Omega}}(\gamma) \leq (\bar{c} - 1) \Psi_{\Omega}(\gamma) \right\} \geq b_2,
					\]
					where $b_2 > 0$ is a constant and $\bar{c} := 1 + \frac{c_0 + 1}{c_0 - 1}$, with $c_0$ 
					in \eqref{eq: opt_lambda}.
					\item \label{assumption_error} \textbf{Error Distribution Condition.} The errors $\epsilon_i$'s are i.i.d. with density function $f(\cdot)$. Let $\zeta_{ij} := \epsilon_i - \epsilon_j$ for $i \neq j$, and denote by $F^*(\cdot)$ and $f^*(\cdot)$ cumulative distribution function and probability density function of $\zeta_{ij}$, respectively. There exists a constant $b_3 > 0$ such that, for all $i,j$,
					\begin{equation*} 
						f^*(\zeta_{ij}) \geq b_3 \quad \text{uniformly for all }  |\zeta_{ij}| \leq   \frac{34 b_1^2 c_0 \bar{c}^2}{b_2 b_3} c_p^2 c_{|\Omega|}^2 \mathcal{E}(n, p).
					\end{equation*}
					Here $\mathcal{E}(n, p) := n^{-1}\log p + \sqrt{n^{-1}\log p}$, and the constants $c_p$ and $c_{|\Omega|}$ are defined in~\eqref{ineq:co} and~\eqref{ineq:cp}, respectively.
				\end{enumerate}
			\end{assumption}

			\begin{remark}
				Assumption~\ref{assump:error} is in line with conditions commonly imposed in high-dimensional regression with rank-based methods. The RE condition on the design matrix is standard in the analysis of $\ell_1$-type regularization \citep{BickelRitovTsybakov2009,Negahban2012}. 
				For the error distribution, it is typical in the analysis of Wilcoxon-type estimators to assume i.i.d.\ errors with a continuous density bounded away from zero near the origin, a condition satisfied by many common distributions such as Gaussian, Cauchy, Student's $t$, and $\chi^2(k)$ with $k \geq 3$ \citep{JureckovaSen1996,Wang2020ATR}.
				More detailed discussions on the roles of Assumptions~\ref{assumption_X}--\ref{assumption_error} and sufficient conditions for Assumption~\ref{assumption_error} are deferred to Appendix~\ref{appendix:assumption_discussion}.

			\end{remark}

			We are now ready to provide a finite-sample bound on the estimation error of the proposed group Lasso regularized rank regression estimator \eqref{mod:rankgrouplasso0} with the simulation-based parameter \eqref{eq: opt_lambda}; see Theorem \ref{th:consistency}. The proof is detailed in Appendix~\ref{sec: proof_thm1}. The overall proof strategy follows that of \cite{Wang2020ATR}, but its extension to the group Lasso setting is not immediate. In particular, the analysis must be carried out in terms of the dual norm $\Psi^d$, and requires a group-wise decomposition of the regularizer together with an error analysis over a group-structured cone. These additional ingredients are essential for extending the rank regression framework from coordinate sparsity to structured sparsity.
			\begin{theorem}
				\label{th:consistency}
				Suppose Assumption~\ref{assump:error} hold. If $p\geq 2/\alpha_0$, then the group Lasso regularized rank regression estimator $\hat{\beta}(\lambda^*)$, which is a solution to  \eqref{mod:rankgrouplasso0} with the regularization parameter $\lambda^*$ specified in~\eqref{eq: opt_lambda}, satisfies
				$$
				\|\hat{\beta}(\lambda^*)  -\beta^*\|_2 \leq    \frac{{17} b_1 c_0 \bar{c}}{b_2b_3}  c_{|\Omega|}c_p{\cal E}(n,p)
				$$ 
				with probability at least $1-\exp{(-n {\cal E}^2(n,p))}  - \alpha_0$. Here, the constants $c_{|\Omega|}$ and $c_p$ are defined in~\eqref{ineq:cp} and~\eqref{ineq:co}, respectively; the constants $c_0$ and $\alpha_0$ are given in \eqref{eq: opt_lambda}; and the constants $\bar{c}$, $b_1$, $b_2$, and $b_3$ are specified in Assumption~\ref{assump:error} .
			\end{theorem}
			
			In particular, when $\log p \ll n$ and the group Lasso weights are chosen as $w_l = \sqrt{|\mathcal{G}_l|}$, the error bound in Theorem \ref{th:consistency} simplifies to a more interpretable form, as stated in the following corollary.
			\begin{corollary}\label{eq: coro_errorbound}
				Suppose Assumption~\ref{assump:error} hold. If $e^n\gg p\geq 2/\alpha_0$, then the group Lasso regularized rank regression estimator $\hat{\beta}(\lambda^*)$, which is a solution to  \eqref{mod:rankgrouplasso0} with the regularization parameter $\lambda^*$ from~\eqref{eq: opt_lambda} and weights $w_l = |\mathcal{G}_l|^{1/2}$, satisfies
				\begin{equation*}
					\|\hat{\beta}(\lambda^*) -\beta^*\|_2 \leq   \frac{ 34 b_1 c_0 \bar{c}}{b_2b_3}   \sqrt{\frac{|\Omega|  \log p}{n}}, 
				\end{equation*}
				with probability at least $1-p^{-1}-\alpha_0$. 
			\end{corollary}
			
			Note that Theorem~\ref{th:consistency} recovers the results for the following estimators 
			as special cases.
			\begin{enumerate}
				\item[(1)] \textit{Lasso}: 
				$\Psi(\beta) = \|\beta\|_1$. 
				In this case,
				$
				c_{|\Omega|} = \sqrt{q}$, $c_p = 1$,
				and Corollary \ref{eq: coro_errorbound} yields an error bound that is essentially equivalent to the known rate for the rank Lasso estimator established in \cite{Wang2020ATR}.
				
				\item[(2)] \textit{Weighted Lasso}: 
				$\Psi(\beta) = \sum_{j=1}^p w_j |\beta_j|$ with $w_j > 0$, $j\in[p]$. 
				In particular, if $w_j = \mathcal{O}(1)$ uniformly, Theorem~\ref{th:consistency} yields the bound $\|\hat{\beta}(\lambda^*) - \beta^*\|_2 \lesssim \sqrt{q \log p / n}$ under $\log p \ll n$.
			\end{enumerate}
			
			Furthermore, our result in Theorem~\ref{th:consistency} extends beyond the group Lasso setting and applies to any regularizer $\Psi(\cdot)$ satisfying the following three properties:
			(i) subadditivity (i.e., the triangle inequality) 
			$\Psi(\alpha + \beta) \leq \Psi(\alpha) + \Psi(\beta)$;
			(ii) support decomposability as formalized in~\eqref{eq:decomposition}; and 
			(iii) norm compatibility as required in~\eqref{ineq:cp} and~\eqref{ineq:co}, without assuming a specific form of the associated constants. We omit the proof here, as it follows directly from a straightforward modification of our main argument in Appendix~\ref{sec: proof_thm1}.
			
			\section{An Efficient Superlinearly Convergent Algorithm}\label{sec:3}
			In this section, we design an efficient, superlinearly convergent algorithm to solve the general class of convex-regularized rank regression problem, with the group Lasso case \eqref{mod:rankgrouplasso0} serving as a representative example. Our approach is a proximal augmented Lagrangian method (PALM) applied to the dual formulation, where the added proximal term ensures that the subproblems are well-posed and enables efficient semismooth Newton updates. With this design, PALM avoids the triple-loop structure induced by the state-of-the-art PPMM \citep{tang2023proximal} when specialized to convex problems, and ensures both efficiency and scalability in high-dimensional problems.

			Problem \eqref{mod:rankgrouplasso0} can be equivalently written in the constrained form
			\begin{equation}\label{mod:primaleq_const}
				\min_{s\in \mathbb{R}^n,\beta\in \mathbb{R}^p}\ \left\{
				L(s) +  \lambda\Psi(\beta)\ \middle \vert\  X\beta-s-y=0 
				\right\}.
			\end{equation}
			The dual problem of \eqref{mod:rankgrouplasso0} then becomes
			\begin{equation*}
				\sup_{z \in \mathbb{R}^n} \inf_{\substack{s \in \mathbb{R}^n \\ \beta \in \mathbb{R}^p}} 
				\left\{ L(s) + \lambda\Psi(\beta) + \langle z, X\beta - s - y \rangle \right\}
				= \sup_{z \in \mathbb{R}^n} \left\{ -\langle y, z \rangle - L^*(z) - \lambda \Psi^*(-\lambda^{-1}X^{\intercal}z) \right\},
			\end{equation*}
			which leads to the equivalent minimization problem:
			\begin{equation}\label{mod:dual}
				\min_{z\in \mathbb{R}^n} \left\{ g(z):= \langle y,z\rangle +L^*(z) +  \lambda\Psi^*(-\lambda^{-1}X^{\intercal}z) \right\},
			\end{equation}
			where $L^*(\cdot)$ and $  \Psi^*(\cdot)$ are conjugate functions of $L(\cdot)$ and $ \Psi(\cdot)$, respectively. In addition, the Karush–Kuhn–Tucker (KKT) conditions associated with \eqref{mod:primaleq_const} and \eqref{mod:dual} are
			\begin{equation*}
				X\beta-s-y=0,\quad s={\rm Prox}_L(z+s),\quad \beta ={\rm Prox}_{\lambda \Psi}(\beta-X^{\intercal}z).
			\end{equation*}
			
			We adopt the framework in \citet[Examples 11.46 \& 11.57]{RockafellarWets2009} to give the augmented Lagrangian function of \eqref{mod:dual}.
			Specifically, we rewrite problem \eqref{mod:dual} as the minimization of the function $g(z)=\tilde{g}(z,\mathbf 0, \mathbf 0)$, where for any $(z,\nu,\phi)\in \mathbb{R}^n\times\mathbb{R}^n\times\mathbb{R}^p$, 
			$\tilde{g}(z, \nu, \phi)$ is defined as 
			\begin{equation*}
				\tilde{g}(z, \nu, \phi):= \langle y, z\rangle  +  L^*(z+\nu) +  \lambda\Psi^*(-\lambda^{-1}X^{\intercal}z +\lambda^{-1}\phi).
			\end{equation*}
			The associated Lagrangian function is then given by
			\begin{equation}\label{eq: def_l}
				\begin{aligned}
					& l(z;s,\beta) = \inf_{\nu\in\mathbb{R}^n,\phi\in\mathbb{R}^p}\Big\{\tilde{g}(z,\nu,\phi)-\langle s, \nu\rangle-\langle \beta ,\phi\rangle \Big\}  \\
					& = \!\!\!\!\inf_{\nu\in\mathbb{R}^n,\phi\in\mathbb{R}^p}\!\! \Big\{ \langle y,z\rangle + L^*(z+\nu) - \langle s, \nu\rangle +  \lambda\Psi^*(-\lambda^{-1}X^{\intercal} z + \lambda^{-1}\phi) -\langle \beta,\phi\rangle\Big\}  \\
					& = \langle y,z\rangle-L(s)- \lambda\Psi(\beta)+\langle s,z\rangle- \langle \beta,X^{\intercal}z\rangle. 
			\end{aligned}\end{equation}
			And, given $\sigma > 0$, the augmented Lagrangian function corresponding to \eqref{mod:dual} is
			\begin{equation*}\begin{aligned}
					&{\cal L}_{\sigma}(z;s,\beta) = \sup\limits_{\nu\in\mathbb{R}^n,\phi\in\mathbb{R}^p}\Big\{l(z;\nu,\phi)-\frac{1}{2\sigma}\|\nu-s\|_2^2- \frac{1}{2\sigma}\|\phi-\beta\|_2^2\Big\}\notag\\
					&\qquad = \langle y, z \rangle - e_{\sigma}L(s + \sigma z) + \frac{1}{2\sigma}\|s +  \sigma z\|_2^2 - \frac{1}{2\sigma}\|s\|_2^2 - \lambda e_{\sigma\lambda} \Psi(\beta - \sigma X^{\intercal} z) \\
					& \qquad \quad + \frac{1}{2\sigma}\|\beta - \sigma X^{\intercal} z\|_2^2 - \frac{1}{2\sigma}\|\beta\|_2^2, 
			\end{aligned}\end{equation*}
			where $e_{\sigma} L(\cdot)$ and $e_{\sigma\lambda} \Psi(\cdot)$ denote the Moreau envelopes
			as defined in \eqref{def:moreau envelope}.

			A direct application of the augmented Lagrangian method (ALM) to problem \eqref{mod:dual} leads to 
			\begin{equation*}
				\min_{z\in\mathbb{R}^n}  {\cal L}_{\sigma}(z;\tilde{s},\tilde{\beta}).
			\end{equation*}
			The first-order optimality condition for this problem is given by
			\begin{equation*}\begin{aligned}
					\nabla_{z} {\cal L}_{\sigma} (z;\tilde{s},\tilde{\beta}) =y + \sigma {\rm Prox}_{L}(\tilde{s}/\sigma + z) - \sigma \lambda X{\rm Prox}_{\Psi}\left(\tilde{\beta}/(\sigma \lambda)-  X^{\intercal}z / \lambda
					\right),
			\end{aligned}\end{equation*}
			following the fact that ${\rm Prox}_{\sigma f}(\cdot) = \sigma\, {\rm Prox}_{f}\left(\cdot/\sigma\right)$ for any $\sigma>0$ and any positively $1$-homogeneous convex function $f$.
			However, solving this equation reliably with Newton-type methods is challenging, as the associated generalized Hessian can be ill-conditioned or even singular.
			
			To address this while preserving the convergence properties of the ALM framework, we incorporate a proximal term into the subproblem, following the methodologies of the proximal augmented Lagrangian method (PALM) \citep{Rockafellar1976Augmented,li2020asymptotically}. The proposed inexact PALM approach, outlined in Algorithm \ref{alg:PALM}, improves the subproblems' conditioning and enhances the algorithmic stability.
			
			\vspace{-0.3cm}
			\begin{algorithm}[H]\small
				\caption{Proximal augmented Lagrangian method for the dual of \eqref{mod:rankgrouplasso0}}
				\label{alg:PALM}
				\begin{algorithmic}[1]
					\Require Initial point $(z^0,s^0, \beta^0) \in \mathbb{R}^n\times\mathbb{R}^n\times\mathbb{R}^p$; parameters $\sigma_0, \tau>0$; positive summable sequence $\{\delta_k\}$ with $\delta_k<1$; stopping tolerance $\mathrm{Tol} >0$.
					\State Let $k=0$.
					\Repeat
					\State \textbf{Step 1:} Compute $z^{k+1}$ by solving:
					\begin{equation}\label{mod:iPALM}
						z^{k+1} \approx \mathop{\arg\min}\limits_{z\in\mathbb{R}^n} \left\{ \psi_k(z) := {\mathcal{L}}_{\sigma_k}(z;s^k,\beta^k) + \frac{\tau}{2\sigma_k}\|z-z^k\|_2^2 \right\}
					\end{equation}
					such that
					\begin{equation}\label{stopping criteria}
						\|\nabla \psi_k(z^{k+1})\|_2 \leq \frac{\delta_k \min(1,\sqrt{\tau})}{\sigma_k} \min\left\{1,\left\|(\sqrt{\tau}z^{k+1}, s^{k+1}, \beta^{k+1}) - (\sqrt{\tau} z^k, s^k, \beta^k)\right\|_2\right\}.
					\end{equation}
					\State \textbf{Step 2:} Update
					$$
					s^{k+1} = \sigma_k {\rm Prox}_{ L}(  s^k/\sigma_k +  z^{k+1}), \qquad \beta^{k+1} = \sigma_k \lambda {\rm Prox}_{ \Psi}(\beta^k/(\sigma_k \lambda) -  X^\intercal z^{k+1}/\lambda).
					$$
					\State \textbf{Step 3:} Update $\sigma_{k+1}\geq \sigma_k$, and let $k\leftarrow k+1$.
					\Until The relative KKT residual of \eqref{mod:primaleq_const} and \eqref{mod:dual} satisfies
					\begin{equation}\label{eq: def_etakkt}
						\eta_{\mathrm{kkt}}(z^{k+1},s^{k+1},\beta^{k+1}) := \max\{\eta_{p}(s^{k+1},\beta^{k+1}),\eta_{d}(z^{k+1},s^{k+1},\beta^{k+1})\} \leq \mathrm{Tol},
					\end{equation}
					where $$\eta_{p}(s,\beta):= \frac{\|X\beta-s-y\|_2}{1+\|y\|_2}$$ and $$\eta_{d}(z,s,\beta):= \max\left\{  \frac{\|s-{\rm Prox}_L(z+s)\|_2}{1+\|s\|_2}, \frac{\|\beta -{\rm Prox}_{\lambda \Psi}(\beta-X^{\intercal}z)\|_2}{1+\|\beta\|_2}    \right\}.$$ 
					\State \textbf{Output:} An approximate solution $(s^{k+1},\beta^{k+1})$ to \eqref{mod:primaleq_const}, and an approximate solution $z^{k+1}$ to \eqref{mod:dual}.
				\end{algorithmic}
			\end{algorithm}			
			
			\vspace{-0.6cm}
			\begin{remark}
				In practical implementations, we set $\tau = 1$ and update $\sigma_{k+1} = 1.5\sigma_k$. The stopping tolerance is chosen as ${\rm Tol} = 10^{-6}$ to ensure a high solution accuracy.
			\end{remark}
			
			
			\subsection{Convergence Analysis}
			In this section, we establish the superlinear convergence of the proposed inexact PALM framework for solving the convex-regularized rank regression problem. Our analysis closely follows the proof framework of \citet{li2020asymptotically}, but requires a fundamental modification. Their convergence results rely on a local error bound condition associated with the KKT residual mapping, which is not readily verifiable in our setting and cannot be directly established using existing techniques. To overcome this difficulty, we replace the required error bound condition with the metric subregularity of the KKT residual mapping, which is sufficient to guarantee the superlinear convergence of PALM.
			Its verification, however, is highly nontrivial in the present setting due to the significant difficulty introduced by the group Lasso regularizer.
			Notably, although error bound results for the group lasso have been studied in the literature (see, e.g., \cite{zhou2017unified}), these results focus on the subdifferential system $\{\beta : \mathbf 0 \in \partial (\lambda\Psi)(\beta)\}$ and therefore cannot be directly applied to the KKT system arising in our problem, which takes the form $\{(z,s,\beta): \mathbf 0 \in X^\intercal z + \partial(\lambda\Psi)(\beta)\}$. 

			Define the KKT residual mapping associated with \eqref{mod:primaleq_const} and \eqref{mod:dual} as
			\begin{equation*}
				\begin{aligned}
					&\mathcal{T}(z,s,\beta) := \left\{(z',s',\beta') \mid (z', -s', -\beta') \in \partial l(z;s,\beta)\right\} \\
					&\qquad = \left\{(z',s',\beta') \mid z' = y + s - X\beta,\ s' \in -z + \partial L(s),\ \beta' \in X^{\intercal}z + \partial  (\lambda\Psi)(\beta)\right\}.
				\end{aligned}
			\end{equation*}
			The KKT solution set is therefore $\mathcal{T}^{-1}(\mathbf 0)$, i.e.,
			\begin{equation*}
				\begin{split}
					\mathrm{SOL}_{\mathrm{KKT}}
					&= \mathcal{T}^{-1}(\mathbf 0) \\
					&= \bigl\{(z,s,\beta) \mid \mathbf 0 = y + s - X\beta,
					\mathbf 0 \in -z + \partial L(s), \mathbf 0 \in X^{\intercal}z + \partial (\lambda\Psi)(\beta)\bigr\}.
				\end{split}
			\end{equation*}
			Since the objective function in \eqref{mod:rankgrouplasso0} is proper, closed, convex, and coercive, a minimizer exists, and hence $\mathrm{SOL}_{\mathrm{KKT}}$ is nonempty by \citet[Theorem 3.34]{ruszczynski2006nonlinear}. 
			Moreover, coercivity implies boundedness, and together with the closedness of $\mathrm{SOL}_{\mathrm{KKT}}=\mathcal T^{-1}(0)$ (see \citet[Exercise 12.8]{RockafellarWets2009}), it follows that $\mathrm{SOL}_{\mathrm{KKT}}$ is compact.

			To analyze the metric subregularity of $\mathcal{T}$, we exploit the fact that $\mathrm{SOL}_{\mathrm{KKT}}$ can be expressed as the intersection of three sets. Specifically, by defining
			\begin{align*}
				\mathcal{T}_X(z,s,\beta) 
				&\coloneq \{z' \mid z' = y + s - X\beta\},\\
				\mathcal{T}_L(z,s,\beta) 
				&\coloneq \{s' \mid s' \in -z + \partial L(s)\},\\
				\mathcal{T}_\Psi(z,s,\beta) 
				&\coloneq \{\beta' \mid \beta' \in X^\intercal z + \partial (\lambda \Psi)(\beta)\},
			\end{align*} 
			and 
			$$\mathrm{SOL}_X \coloneq \mathcal{T}_X^{-1}(\mathbf 0),\quad
			\mathrm{SOL}_L \coloneq \mathcal{T}_L^{-1}(\mathbf 0),\quad
			\mathrm{SOL}_\Psi \coloneq \mathcal{T}_\Psi^{-1}(\mathbf 0),
			$$
			the KKT solution set $\mathrm{SOL}_{\mathrm{KKT}}$ can be written as
			\begin{align}\label{eq:SOL_intersection}
				\mathrm{SOL}_{\mathrm{KKT}} = \mathrm{SOL}_X \cap \mathrm{SOL}_L \cap \mathrm{SOL}_\Psi.
			\end{align}
			This representation allows us to establish the metric subregularity of $\mathcal{T}$ by combining the metric subregularity of the individual mappings with a regularity condition of the solution sets. 
			
			Among the three components, the main difficulty lies in the mapping $\mathcal{T}_\Psi$, which corresponds to the KKT condition associated with the group lasso regularization. 
			To address this, we first establish the metric subregularity for $\mathcal{T}_\Psi$ in Proposition~\ref{prop:error_bound_psi}, which plays a key role in studying the metric subregularity for $\mathcal{T}$ and is also of independent interest in the convergence analysis of proximal augmented Lagrangian type algorithms for problems involving the group Lasso regularization. 
			
			\begin{proposition}\label{prop:error_bound_psi}
				Suppose that there exists $(\bar z,\bar s,\bar\beta)\in \mathrm{SOL}_\Psi$ satisfying the following conditions:
				\begin{enumerate}[label=(B\arabic*)]
					\item\label{condition_SC_Psi} \textbf{Strict complementarity $(\Psi)$:}
					\(
					-X^\top \bar z \in \operatorname{ri}\bigl(\partial(\lambda\Psi)(\bar\beta)\bigr);
					\)
					\item\label{condition_group_nondegeneracy} \textbf{Active-group nondegeneracy:} the collection $\{X_{\mathcal G_l}\bar\beta_{\mathcal G_l}\}_{l \in I_{\bar\beta}^+}$ is linearly independent, where $I_{\bar\beta}^+ := \{l : \bar\beta_{\mathcal G_l} \neq \mathbf 0\}$ is the active group index set.
				\end{enumerate}
				Then $\mathcal{T}_\Psi$ is metrically subregular at $(\bar z,\bar s,\bar\beta)$ for $\mathbf 0$, i.e., there exist constants $\kappa_\Psi > 0$ and $r_\Psi > 0$ such that
				\begin{align}\label{eq:EB_Psi_local_final}
					\operatorname{dist}\bigl((z,s,\beta), \mathrm{SOL}_\Psi\bigr)
					\le
					\kappa_\Psi\, \operatorname{dist}\bigl(\mathbf 0, \mathcal T_\Psi(z,s,\beta)\bigr),
					\quad
					\forall (z,s,\beta)\in \mathcal B_2^{r_\Psi}(\bar z,\bar s,\bar\beta).
				\end{align}
			\end{proposition}
			\begin{proof}
				Since the $s$-component is unrestricted in $\mathrm{SOL}_\Psi$, we have
				\begin{equation}
					\begin{split}
						\operatorname{dist}\bigl((z,s,\beta),\mathrm{SOL}_\Psi\bigr)
						&= \operatorname{dist}\bigl((z,\beta),\mathcal S_{\Psi}\bigr),\\
						\operatorname{dist}\bigl(\mathbf 0, \mathcal T_\Psi(z,s,\beta)\bigr) 
						&= \operatorname{dist}\bigl(-X^\intercal z, \partial (\lambda \Psi)(\beta)\bigr),
					\end{split}
					\label{eq: reduce_psi}
				\end{equation}
				where
				\[
				\mathcal S_{\Psi}:=\{(z,\beta)\mid \mathbf 0\in X^\top z+\partial(\lambda\Psi)(\beta)\}.
				\]
				It therefore suffices to establish a local error bound for $\mathcal S_{\Psi}$ around $(\bar z,\bar\beta)$.
				
				We first analyze the local behaviour of $(z,\beta)$ around $(\bar z,\bar \beta)$. Since the groups $\{\mathcal{G}_l\}_{l=1}^g$ are non-overlapping, the subdifferential of $\Psi(\cdot)$ admits the block decomposition
				\[
				\partial(\lambda\Psi)(\beta)=\prod_{l=1}^g \partial\bigl(c_l\|\beta_{\mathcal G_l}\|_2\bigr),
				\mbox{ where }c_l=\lambda w_l \mbox{ for } l \in [g].
				\]
				Denote the inactive group index set $I_{\bar\beta}^0:=\{l:\bar\beta_{\mathcal G_l}=\mathbf 0\}$. By strict complementarity assumption~\ref{condition_SC_Psi}, for each $l\in I_{\bar\beta}^0$, we have $\|X_{\mathcal G_l}^\top \bar z\|_2<c_l.$ 
				Therefore, there exist constants $\bar \eta>0$ and $\bar{r}>0$ such that $\|X_{\mathcal G_l}^\top \bar z\|_2\le c_l-2\bar \eta$, $l\in I_{\bar\beta}^0$, and for all $(z,\beta)\in\mathcal{B}^{\bar{r}}_2 (\bar{z}, \bar{\beta})$,  
				\begin{align}
					\|X_{\mathcal G_l}^\top z\|_2\le c_l-\bar \eta,  \quad \forall l\in I_{\bar\beta}^0,\qquad \mbox{and}\qquad \beta_{\mathcal G_l}\neq \mathbf 0,  \quad \forall l\in I_{\bar\beta}^+. \label{eq:active_stability}
				\end{align}
				%
				Based on these observations, we define the mapping $H:\mathcal{B}^{\bar{r}}_2 (\bar{z}, \bar{\beta})\to \mathbb R^p$ by
				\begin{equation}\label{eq: def_H}
					(H(z,\beta))_{\mathcal G_l}=H_l(z,\beta):=
					\begin{cases}
						\beta_{\mathcal G_l}, & l\in I_{\bar\beta}^0,\\[4pt]
						X_{\mathcal G_l}^\top z + c_l\,\dfrac{\beta_{\mathcal G_l}}{\|\beta_{\mathcal G_l}\|_2}, & l\in I_{\bar\beta}^+,
					\end{cases}
				\end{equation}
				which is continuously differentiable on $\mathcal B^{\bar{r}}_2(\bar z,\bar\beta)$ by \eqref{eq:active_stability}. Specifically, for any $(\mathrm d z,\mathrm d\beta)\in\mathbb R^n\times\mathbb R^p$ and $l\in[g]$, the derivative of $H$ at $(\bar{z},\bar{\beta})$ is given by
				\begin{equation}\label{eq:derivative_psi}
					(\mathrm{D}H(\bar z,\bar\beta)(\mathrm d z,\mathrm d\beta))_{\mathcal G_l} = \mathrm{D}H_l(\bar z,\bar\beta)(\mathrm d z,\mathrm d\beta)
					=
					\begin{cases}
						\mathrm d\beta_{\mathcal G_l}, & l\in I_{\bar\beta}^0,\\[4pt]
						X_{\mathcal G_l}^\top \mathrm d z + \mathcal H_l \mathrm d\beta_{\mathcal G_l}, & l\in I_{\bar\beta}^+,
					\end{cases}    
				\end{equation}
				where
				\begin{equation}\label{eq:def_psi_hessian}
					q_l \coloneq \frac{\bar\beta_{\mathcal G_l}}{\|\bar\beta_{\mathcal G_l}\|_2}, \quad \mathcal{H}_l = \frac{c_l}{\|\bar\beta_{\mathcal G_l}\|_2}
					\bigl(I-q_l q_l^\top\bigr),\quad l\in I_{\bar\beta}^+.
				\end{equation}
				We next show that $\mathrm{D}H(\bar z,\bar\beta)$ is surjective by verifying that the orthogonal complement of its range is trivial. Let $\xi\in \mathbb{R}^p$ be orthogonal to the range of $\mathrm{D}H(\bar z,\bar\beta)$, i.e.,
				\begin{align}
					\sum_{l=1}^g \langle \xi_{\mathcal G_l}, \mathrm{D}H_l(\bar z,\bar\beta)(\mathrm{d} z,\mathrm{d}\beta)\rangle =0,
					\qquad \forall (\mathrm{d} z,\mathrm{d}\beta).\label{eq: condition_xi}
				\end{align}
				First, by taking $\mathrm{d} z=\mathbf 0$ and choosing $\mathrm{d}\beta$ supported on some fixed $l\in I_{\bar\beta}^0$, the condition \eqref{eq: condition_xi} gives
				\[
				\langle \xi_{\mathcal G_l},\mathrm{d}\beta_{\mathcal G_l}\rangle=0,
				\qquad \forall \mathrm{d}\beta_{\mathcal G_l},
				\]
				which means 
				\[
				\xi_{\mathcal G_l}=\mathbf 0,
				\qquad \forall l\in I_{\bar\beta}^0.
				\]
				Next, we set $\mathrm{d} z=\mathbf 0$ and choose $\mathrm{d}\beta$ supported on some $l\in I_{\bar\beta}^+$, then \eqref{eq: condition_xi} implies
				\[
				\left\langle \xi_{\mathcal G_l},\mathcal{H}_l \mathrm{d}\beta_{\mathcal G_l}\right\rangle =0,
				\qquad \forall \mathrm{d}\beta_{\mathcal G_l},
				\]
				which yields $(I-q_l q_l^\top)\xi_{\mathcal G_l}=\mathbf 0$, and thus
				\begin{align}
					\xi_{\mathcal G_l}\in \operatorname{span}(q_l),
					\qquad \forall\ l\in I_{\bar\beta}^+.\label{eq: xi_q}
				\end{align}
				Finally, considering $\mathrm{d}\beta=\mathbf 0$ and varying $\mathrm{d} z$ arbitrarily yields
				\[
				\sum_{l\in I_{\bar\beta}^+} \langle \xi_{\mathcal G_l}, X_{\mathcal G_l}^\top \mathrm{d} z\rangle =0
				\qquad \forall \mathrm{d} z,
				\]
				that is,
				\[
				\sum_{l\in I_{\bar\beta}^+} X_{\mathcal G_l}\xi_{\mathcal G_l}=\mathbf 0.
				\]
				This, together with \eqref{eq:def_psi_hessian}, \eqref{eq: xi_q} and the active-group nondegeneracy assumption~\ref{condition_group_nondegeneracy}, we have
				$\xi_{\mathcal G_l}=\mathbf 0$ for all $l\in I_{\bar\beta}^+$. Combining the above gives $\xi=\mathbf 0$, and therefore $\mathrm{D}H(\bar z,\bar\beta)$ is surjective. By applying \citet[Example~9.44]{RockafellarWets2009}, where the required constraint qualification is satisfied due to the surjectivity of $\mathrm{D}H(\bar z,\bar\beta)$, 
				there exist constants $\kappa_H>0$ and $r_H>0$ such that
				\begin{align}
					\operatorname{dist}\bigl((z,\beta),H^{-1}(\mathbf 0)\bigr)
					\le
					\kappa_H \|H(z,\beta)\|_2,
					\qquad
					\forall (z,\beta)\in \mathcal B_2^{r_H}(\bar z,\bar\beta).\label{eq:smooth_EB}
				\end{align}

				Then we build the local equivalence of $\mathcal{S}_{\Psi}$ and $H^{-1}(\mathbf 0)$ around $(\bar z,\bar \beta)$. To begin, we claim that  
				\begin{equation}\label{eq:local_smooth_representation_psi}
					\mathcal{S}_{\Psi}\cap \mathcal{B}^{\bar r}_2 (\bar{z},\bar{\beta}) = \{(z,\beta)\in \mathcal{B}^{\bar r}_2 (\bar{z},\bar{\beta})\mid H(z,\beta)=\mathbf 0\},
				\end{equation}
				which we will prove by showing both inclusions. First, let $(z,\beta)\in \mathcal B^{\bar r}_2(\bar z,\bar\beta)$ satisfy $H(z,\beta)=\mathbf 0$.
				For $l\in I_{\bar\beta}^0$, we have that $H_l(z,\beta)=\beta_{\mathcal G_l} = \mathbf 0$ and $\|X_{\mathcal G_l}^\top z\|_2<c_l$ by \eqref{eq:active_stability}, which implies that
				\[
				-\,X_{\mathcal G_l}^\top z \in 
				\partial\bigl(c_l\|\beta_{\mathcal G_l}\|_2\bigr) = \mathcal B_2^{c_l};
				\]
				for $l\in I_{\bar\beta}^+$, from \eqref{eq:active_stability} we have $\beta_{\mathcal G_l}\neq \mathbf 0$, and hence $H_l(z,\beta)=\mathbf 0$ gives
				\[
				-\,X_{\mathcal G_l}^\top z =
				c_l \frac{\beta_{\mathcal G_l}}{\|\beta_{\mathcal G_l}\|_2}
				\in \partial\bigl(c_l\|\beta_{\mathcal G_l}\|_2\bigr) = \left\{c_l \frac{\beta_{\mathcal G_l}}{\|\beta_{\mathcal G_l}\|_2}
				\right\}.
				\]
				This means $\mathbf 0\in X^\top z+\partial(\lambda\Psi)(\beta)$, and thus $(z,\beta)\in \mathcal{S}_{\Psi}$. Conversely, let $(z,\beta)\in \mathcal{S}_{\Psi}\cap \mathcal{B}^{\bar r}_2 (\bar{z},\bar{\beta})$, which implies that
				\begin{align}
					\mathbf 0\in X_{\mathcal G_l}^\top z+\partial\bigl(c_l\|\beta_{\mathcal G_l}\|_2\bigr),
					\quad \forall\ l\in[g].\label{spsi_condition}
				\end{align}
				For $l\in I_{\bar\beta}^0$, the fact that $\|X_{\mathcal G_l}^\top z\|_2<c_l$ from \eqref{eq:active_stability} and the inclusion \eqref{spsi_condition} indicates that $\beta_{\mathcal G_l}=\mathbf 0$, and thus $H_l(z,\beta)=\mathbf 0$; for $l\in I_{\bar\beta}^+$, then by \eqref{eq:active_stability} that $\beta_{\mathcal G_l}\neq \mathbf 0$, and the inclusion \eqref{spsi_condition} reduces to
				\[
				\mathbf 0 = X_{\mathcal G_l}^\top z + c_l \frac{\beta_{\mathcal G_l}}{\|\beta_{\mathcal G_l}\|_2},
				\]
				which is equivalent to $H_l(z,\beta)=\mathbf 0$. This gives $H(z,\beta)=\mathbf 0$. Therefore, this proves the claim \eqref{eq:local_smooth_representation_psi}. 
				Building on this, for all $(z,\beta)\in\mathcal B_2^{\bar r /2}(\bar z,\bar\beta)$, by denoting its projection onto $H^{-1}(\mathbf 0)$ as $(z^*,\beta^*)$, we have
				\begin{align*}
					\|(z^*,\beta^*)-(\bar z,\bar\beta)\|_2
					&\le
					\|(z^*,\beta^*)-(z,\beta)\|_2+\|(z,\beta)-(\bar z,\bar\beta)\|_2 \\
					&= \operatorname{dist}\bigl((z,\beta),H^{-1}(\mathbf 0)\bigr) + \|(z,\beta)-(\bar z,\bar\beta)\|_2 
					\\
					&\leq 2\|(z,\beta)-(\bar z,\bar\beta)\|_2  < \bar r.
				\end{align*}
				Together with \eqref{eq:local_smooth_representation_psi}, we see that 
				\[
				(z^*,\beta^*)\in H^{-1}(\mathbf 0)\cap \mathcal B_2^{\bar r}(\bar z,\bar\beta)=
				\mathcal S_\Psi\cap \mathcal B_2^{\bar r}(\bar z,\bar\beta),
				\]
				which implies
				\begin{align}
					\operatorname{dist}\bigl((z,\beta),\mathcal S_\Psi\bigr) \le \|( z,\beta)-(z^*,\beta^*)\|_2
					=\operatorname{dist}\bigl((z,\beta),H^{-1}(\mathbf 0)\bigr),\quad \forall (z,\beta)\in\mathcal B_2^{\bar r/2}(\bar z,\bar\beta). \label{eq: dist_psi_H}
				\end{align}
				
				Finally, we are ready to analyze the error bound of $\mathcal S_{\Psi}$. 
				Take $r_{\Psi} = \min \{ r_H,\bar r/2\}$. Combining \eqref{eq:smooth_EB} with \eqref{eq: dist_psi_H}, we obtain
				\begin{align}
					\operatorname{dist}((z,\beta),\mathcal{S}_{\Psi})\leq \operatorname{dist}\bigl((z,\beta),H^{-1}(\mathbf 0)\bigr)
					\le
					\kappa_H \|H(z,\beta)\|_2,
					\qquad
					\forall (z,\beta)\in \mathcal B_2^{r_{\Psi}}(\bar z,\bar\beta).\label{eq: error_bound_Spsi}
				\end{align}
				It remains to investigate $\|H(z,\beta)\|_2$. Fix $(z,\beta)\in \mathcal B_2^{r_{\Psi}}(\bar z,\bar\beta)$.
				For each $l\in I_{\bar\beta}^+$, since $\beta_{\mathcal G_l}\neq \mathbf 0$ on $\mathcal B_2^{r_{\Psi}}(\bar z,\bar\beta)$ from \eqref{eq:active_stability}, we have
				\[
				\operatorname{dist}\Bigl(-X_{\mathcal G_l}^\top z,  \partial\bigl(c_l\|\beta_{\mathcal G_l}\|_2\bigr)\Bigr)
				=
				\left\|X_{\mathcal G_l}^\top z + c_l \frac{\beta_{\mathcal G_l}}{\|\beta_{\mathcal G_l}\|_2}\right\|_2
				=
				\|H_l(z,\beta)\|_2.
				\]
				For each $l\in I_{\bar\beta}^0$, if $\beta_{\mathcal G_l}\neq \mathbf 0$, then
				\begin{align*}
					\operatorname{dist}\Bigl(-X_{\mathcal G_l}^\top z, \partial\bigl(c_l\|\beta_{\mathcal G_l}\|_2\bigr)\Bigr)
					&= \left\|X_{\mathcal G_l}^\top z + c_l \frac{\beta_{\mathcal G_l}}{\|\beta_{\mathcal G_l}\|_2}\right\|_2 \\
					&\ge c_l - \|X_{\mathcal G_l}^\top z\|_2 \\
					&\ge \bar{\eta}
					= \frac{\bar{\eta}}{r_{\Psi}} r_{\Psi}
					\ge \frac{\bar{\eta}}{r_{\Psi}} \|H_l(z,\beta)\|_2.
				\end{align*}
				where the second inequality follows from \eqref{eq:active_stability}, the last inequality from $(z,\beta)\in \mathcal B_2^{r_{\Psi}}(\bar z,\bar\beta)$, $\bar\beta_{\mathcal G_l}=\mathbf 0$ and the definition of $H_l$; if $\beta_{\mathcal G_l}=\mathbf 0$, then $H_l(z,\beta)=\mathbf 0$ from \eqref{eq:active_stability}.
				Combining the above two cases, we see that for any $l\in[g]$, we have
				\[
				\|H_l(z,\beta)\|
				\le
				\max\left\{\frac{r_{\Psi}}{\bar \eta},1\right\} \,
				\operatorname{dist}\Bigl(- X_{\mathcal G_l}^\top z ,\partial\bigl(c_l\|\beta_{\mathcal G_l}\|_2\bigr)\Bigr),\qquad
				\forall (z,\beta)\in \mathcal B_2^{r_{\Psi}}(\bar z,\bar\beta)
				\]
				which means
				\[
				\|H(z,\beta)\|
				\le
				\max\left\{\frac{r_{\Psi}}{\bar \eta},1\right\}\,\operatorname{dist}\bigl(-X^\intercal z, \partial (\lambda \Psi)(\beta)\bigr),\qquad
				\forall (z,\beta)\in \mathcal B_2^{r_{\Psi}}(\bar z,\bar\beta).
				\]
				This, together with \eqref{eq: error_bound_Spsi} and \eqref{eq: reduce_psi}, yields for all $(z,s,\beta)\in \mathcal B_2^{r_{\Psi}}(\bar z,\bar s,\bar\beta)$,
				\begin{align*}
					\operatorname{dist}\bigl((z,s,\beta),\mathrm{SOL}_\Psi\bigr)
					= \operatorname{dist}\bigl((z,\beta),\mathcal{S}_{\Psi}\bigr)
					\le \kappa_\Psi\, \operatorname{dist}\bigl(\mathbf 0,\mathcal T_\Psi(z,s,\beta)\bigr),
				\end{align*}
				where $\kappa_\Psi := \kappa_H \max\left\{\frac{r_{\Psi}}{\bar \eta},1\right\}$. This completes the proof.
				\hfill~\halmos
			\end{proof}
			
			Having established the metric subregularity of $\mathcal{T}_\Psi$, we now proceed to prove that of $\mathcal{T}$.  
			In light of the decomposition \eqref{eq:SOL_intersection}, \citet[Theorem~3.4]{hesse2013nonconvex} reveals that the strong regularity of $\mathrm{SOL}_{\mathrm{KKT}}$ is closely related to the normal cones of its components $\mathrm{SOL}_X$, $\mathrm{SOL}_L$, and $\mathrm{SOL}_\Psi$.  To facilitate the subsequent characterization of these normal cones, we invoke the following lemma from \citet[Example~6.8]{RockafellarWets2009}, which provides an explicit description of the normal cone to a general set.
			\begin{lemma}\label{lemma:normal_cone}
				Let $C \subset \mathbb{R}^n$ and $\bar x \in C$. 
				Assume that there exists an open neighborhood $U$ of $\bar x$ and a continuously differentiable mapping 
				$F: U \to \mathbb{R}^m$ such that
				\(C \cap U = \{ x \in U \mid F(x) = \mathbf 0 \},\) 
				and the derivative $\mathrm{D}F(\bar x)$ is surjective. 
				Then the normal cone to $C$ at $\bar x$ takes the form of
				\[
				N_C(\bar x) = \{ \mathrm D F(\bar x)^* y \mid y \in \mathbb{R}^m \},
				\]
				where $\mathrm D F(\bar x)^*$ denotes the adjoint of $\mathrm D F(\bar x)$.
			\end{lemma}
			
			For notational simplicity, by defining $B\in\mathbb{R}^{\frac{n(n-1)}{2}\times n}$ such that each row indexed by a pair $(i,j)$ is given by $B_{(i,j),:}=e_i - e_j$, the rank loss function can be equivalently written as $L(s) = \frac{2}{n(n-1)} \|Bs\|_1$. 
			We then state our main result on the metric subregularity for the KKT residual mapping $\mathcal{T}$.
			\begin{proposition}\label{prop:error_bound}
				Suppose that there exists $(\bar z,\bar s,\bar\beta)\in \mathrm{SOL}_{\mathrm{KKT}}$ satisfying Assumption~\ref{condition_SC_Psi} 
				and the following:
				\begin{enumerate}[label=(B\arabic*),start=3]
					\item\label{condition_SC_L} \textbf{Strict complementarity $(L)$:}
					\( \bar{z}\in\mathrm{ri(\partial L(\bar{s}))}; \)
					\item\label{condition_rank_nondegeneracy} \textbf{Rank-induced nondegeneracy:} \\
					- the matrix $\bar B_{0} X_{\bar{\mathcal{A}}}$ is of full row rank;\\
					- the collection $\{\bar B_{0} X_{\mathcal{G}_l}  \bar{\beta}_{\mathcal{G}_l}\}_{l\in I_{\bar\beta}^+}$ is linearly independent, \\
					where $\bar B_{0}$ denotes the submatrix of $B$ formed by rows $(i,j)$ satisfying $\bar s_i=\bar s_j$, and $\bar{\mathcal{A}} \coloneq \bigcup_{l\in I_{\bar\beta}^+} \mathcal{G}_l$ denotes the set of active coordinates.
				\end{enumerate}
				Then $\mathcal{T}$ is metrically subregular at $(\bar z,\bar s,\bar\beta)$ for $\mathbf 0$, i.e., there exist constants $\kappa_\ell>0$ and $r_\ell>0$ such that
				\begin{equation}\label{eq:EB}
					\mathrm{dist}((z,s,\beta), \mathrm{SOL}_{\mathrm{KKT}}) 
					\leq \kappa_\ell \mathrm{dist}(0, \mathcal{T}(z, s, \beta)), 
					\quad \forall (z,s,\beta)\in \mathcal{B}_2^{r_\ell} (\bar{z},\bar{s}, \bar{\beta}).
				\end{equation}
			\end{proposition}
			\begin{proof}
				First, we establish that $\mathrm{SOL}_{\mathrm{KKT}} =\mathrm{SOL}_{X} \cap \mathrm{SOL}_{L} \cap \mathrm{SOL}_{\Psi}$
				is strongly regular (see \citet[Definition 3.2 \& Theorem 3.3]{hesse2013nonconvex}) at $(\bar z,\bar s,\bar\beta)$, i.e., there exist constants $\kappa_R>0$ and $r_R>0$ such that, for all $(z,s,\beta)\in \mathcal{B}_2^{r_R} (\bar{z},\bar{s}, \bar{\beta})$,
				\begin{align}\label{eq:regularity}
					& \mathrm{dist}((z,s,\beta), \mathrm{SOL}_{\mathrm{KKT}}) \nonumber\\
					&\quad \leq \kappa_R \max\{\mathrm{dist}((z,s,\beta), \mathrm{SOL}_{X}), \mathrm{dist}((z,s,\beta), \mathrm{SOL}_{L}), \mathrm{dist}((z,s,\beta), \mathrm{SOL}_{\Psi})\}.
				\end{align}
				Since $\mathrm{SOL}_{X}$, $\mathrm{SOL}_{L}$, and $\mathrm{SOL}_{\Psi}$ are closed, from \citet[Theorem~3.4]{hesse2013nonconvex} we obtain that $\mathrm{SOL}_{\mathrm{KKT}}$ is strongly regular at $(\bar z,\bar s,\bar\beta)$ if and only if, for any
				$v_X \in N_{\mathrm{SOL}_{X}}(\bar z,\bar s,\bar\beta)$, $v_L \in N_{\mathrm{SOL}_{L}}(\bar z,\bar s,\bar\beta)$ and $v_\Psi \in N_{\mathrm{SOL}_{\Psi}}(\bar z,\bar s,\bar\beta)$ such that
				\begin{equation*}
					v_X + v_L + v_\Psi = \mathbf 0,
				\end{equation*}
				they must satisfy
				\begin{equation}\label{eq:sum_normal_direction}
					v_X = v_L = v_\Psi = \mathbf 0.
				\end{equation}
				It therefore suffices to verify that \eqref{eq:sum_normal_direction} holds.

				To this end, we characterize the involved three normal cones via Lemma \ref{lemma:normal_cone}.
				
				\vspace{0.2cm}
				\noindent \textbf{Normal cone to $\mathrm{SOL}_{X}$.}
				A direct calculation yields
				\begin{equation}\label{eq:normal_X}
					N_{\mathrm{SOL}_{X}}(\bar{z},\bar{s},\bar{\beta})=\{(\mathbf 0,u,-X^\intercal u)\mid u\in\mathbb{R}^n\}.
				\end{equation}
				
				\vspace{0.2cm}
				\noindent \textbf{Normal cone to $\mathrm{SOL}_{L}$.} Let $\bar{J}_0 \coloneq \{(i,j)\mid B_{(i,j),:} \bar{s}= \mathbf 0\}$, $\bar{J}_1 \coloneq \{(i,j)\mid B_{(i,j),:} \bar{s} \neq \mathbf 0\}$, and denote $\bar{B}_{1}\coloneq B_{\bar{J}_1}$. 
				According to \citet[Theorem 23.9]{rockafellar1997convex}, we have $\partial L(s) = \frac{2}{n(n-1)} B^\intercal \partial \|Bs\|_1$. 
				Since $\bar z\in \partial L(\bar s)$, there exists $\bar{u} \in \partial \|B \bar{s}\|_1$ such that  
				$$
				\bar{z} = \frac{2}{n(n-1)}(\bar{B}_0^\intercal \bar{u}_{\bar{J}_0} + \bar{B}_{1}^\intercal \bar{u}_{\bar{J}_1}).
				$$
				By the strict complementarity condition~\ref{condition_SC_L}, we have
				$$
				\|\bar{u}_{\bar{J}_0}\|_\infty <1.
				$$ 
				We then claim that 
				\begin{equation}\label{eq:SL_localization}
					\mathrm{SOL}_{L} \cap \mathcal{B}^{r_L}_2(\bar{z},\bar{s},\bar{\beta})   = \{(z,s,\beta)\in\mathcal{B}^{r_L}_2(\bar{z},\bar{s},\bar{\beta}) \mid  z - \bar{z} \in \operatorname{range}(\bar{B}_{0}^\intercal),\    \bar{B}_{0} s=\mathbf{0}\},
				\end{equation}
				where 
				\begin{equation}\label{eq:estimate_rL}
					r_L =\min\bigl\{\frac{1}{2}\min_{(i,j)\in \bar{J}_1} \{|\bar s_i -\bar s_j|\}, \ \frac{2}{n(n-1)}  \frac{1-\|\bar{u}_{\bar{J}_0}\|_{\infty}}{\| (\bar{B}_0 \bar{B}_0^\intercal)^{-1} \bar{B}_0 \|_\infty} \bigr\}.
				\end{equation}
				Here, $(\bar{B}_0 \bar{B}_0^\intercal)^{-1}$ is well-defined as $\bar{B}_0$ is of full row rank from the rank-induced nondegeneracy assumption~\ref{condition_rank_nondegeneracy}. We prove the equation \eqref{eq:SL_localization} by showing both inclusions. 
				
				First, let 
				$({z},{s},{\beta}) \in \mathrm{SOL}_{L} \cap\mathcal{B}^{r_L}_2(\bar{z},\bar{s},\bar{\beta})$. Since $z\in \partial L(s)$, there exists ${u} \in \partial \|B {s}\|_1$ such that 
				$$
				{z} = \frac{2}{n(n-1)} B^\top {u} = \frac{2}{n(n-1)} (\bar{B}_0^\intercal {u}_{\bar{J}_0} + \bar{B}_{1}^\intercal {u}_{\bar{J}_1}).
				$$ 
				The fact that 
				\[
				|s_i-\bar s_i|\leq \|s-\bar s\|_2< r_L \leq \frac{1}{2}\min_{(i,j)\in \bar{J}_1} \{|\bar s_i -\bar s_j|\},\quad \forall i,
				\]
				gives
				\[
				|(s_i-s_j)-(\bar s_i-\bar s_j)|\leq |s_i-\bar s_i|+|s_j-\bar s_j| < \min_{(i,j)\in \bar{J}_1} \{|\bar s_i -\bar s_j|\},\quad \forall i,j,
				\]
				which means 
				\[
				{\rm sign}(s_i-s_j) = {\rm sign}(\bar s_i-\bar s_j),\quad \forall (i,j)\in \bar J_1.
				\]
				That is to say, $ {u}_{\bar{J}_1} = \bar{u}_{\bar{J}_1}$.
				Therefore, it can be seen that
				\begin{equation*}
					\begin{split}
						{z} - \bar{z} 
						&= \frac{2}{n(n-1)} (\bar{B}_0^\intercal {u}_{\bar{J}_0} + \bar{B}_{1}^\intercal {u}_{\bar{J}_1})
						-\frac{2}{n(n-1)}(\bar{B}_0^\intercal \bar{u}_{\bar{J}_{0}} + \bar{B}_1^\intercal \bar{u}_{\bar{J}_1}) \\
						&= \frac{2}{n(n-1)}\bar{B}_0^\intercal ({u}_{\bar{J}_{0}} - \bar{u}_{\bar{J}_{0}}),
					\end{split}
				\end{equation*}
				yielding ${z} - \bar{z}\in\operatorname{range}(\bar{B}_0^\intercal)$. 
				The above equality, together with the triangle inequality and the choice of $r_L$ in \eqref{eq:estimate_rL}, we have  
				\begin{align*} 
					\|u_{\bar{J}_0}\|_\infty &\leq \|\bar{u}_{\bar{J}_0}\|_\infty+ \| \frac{n(n-1)}{2}(\bar{B}_0\bar{B}_0^\intercal )^{-1} \bar{B}_0 ({z} - \bar{z}) \|_\infty\\
					&\leq 
					\|\bar{u}_{\bar{J}_0}\|_\infty+ \frac{n(n-1)}{2} \| (\bar{B}_0\bar{B}_0^\intercal )^{-1} \bar{B}_0\|_\infty \cdot \|{z} - \bar{z} \|_\infty < 1,
				\end{align*}
				which means $\bar{B}_0 s = \mathbf{0}$. 
				
				Conversely, let $(z,s,\beta)\in\mathcal{B}^{r_L}_2(\bar{z}, \bar{s},\bar{\beta})$ such that $z - \bar{z} \in \operatorname{range}(\bar{B}_{0}^\intercal)$ and $\bar{B}_{0} s=\mathbf{0}$. Then, there exists ${q}\in \mathbb{R}^{|\bar J_0|}$ such that $z-\bar{z}=\frac{2}{n(n-1)}\bar{B}_0^\intercal (q-\bar{u}_{\bar{J}_0})$. By the choice of $r_L$ in \eqref{eq:estimate_rL}, we see that
				$$
				\|q\|_{\infty} = \|\bar{u}_{\bar{J}_0}+\frac{n(n-1)}{2}(\bar{B}_0 \bar{B}_0^\intercal)^{-1} \bar{B}_0(z-\bar{z})\|_{\infty} < 1,
				$$
				which yields $q\in(-1,1)^{|\bar J_0|}\subset  \partial \|\bar{B}_0s\|_1$.
				In the other hand, the choice of $r_L$ also implies $\partial\|\bar{B}_1s\|_1=\{{\bar{u}_{\bar{J}_1}} \}$, as ${\rm sign}(s_i-s_j) = {\rm sign}(\bar s_i-\bar s_j)$, for all $(i,j)\in \bar J_1$. These, combining with 
				$$
				z=\frac{2}{n(n-1)}\bar{B}_0^\intercal (q-\bar{u}_{\bar{J}_0}) + \frac{2}{n(n-1)}(\bar{B}_0^\intercal \bar{u}_{\bar{J}_0} + \bar{B}_{1}^\intercal \bar{u}_{\bar{J}_1})
				= \frac{2}{n(n-1)}(\bar{B}_0^\intercal q + \bar{B}_{1}^\intercal \bar{u}_{\bar{J}_1}),
				$$
				we have $z\in\partial L(s)$, implying $(z,s,\beta)\in \mathrm{SOL}_L$. Hence, \eqref{eq:SL_localization} holds. 
				
				Therefore, the equality \eqref{eq:SL_localization} indicates that $\mathrm{SOL}_L$, in a neighborhood of $(\bar z,\bar s,\bar\beta)$, is an affine manifold with tangent subspace
				\[
				\operatorname{range}(\bar B_0^\intercal)\times \ker(\bar B_0)\times \mathbb R^p.
				\]
				Thus by Lemma \ref{lemma:normal_cone}, the normal cone to $\mathrm{SOL}_L$ at $(\bar z,\bar s,\bar\beta)$ is given by 
				\begin{equation}\label{eq:normal_L}
					N_{\mathrm{SOL}_L}(\bar z,\bar s,\bar\beta)
					=
					\ker(\bar B_0)\times \operatorname{range}(\bar B_0^\intercal)\times \{\mathbf 0\}.
				\end{equation}
				
				\vspace{0.2cm}
				\noindent \textbf{Normal cone to $\mathrm{SOL}_{\Psi}$.} 
				Note that the rank-induced nondegeneracy condition (Assumption~\ref{condition_rank_nondegeneracy}) implies the active-group nondegeneracy condition (Assumption~\ref{condition_group_nondegeneracy}) required in Proposition~\ref{prop:error_bound_psi}. Hence, the assumptions of Proposition~\ref{prop:error_bound_psi} are satisfied. By the proof of Proposition \ref{prop:error_bound_psi}, there exists $r_\Psi>0$ such that 
				\[
				\mathrm{SOL}_{\Psi}\cap \mathcal{B}^{r_\Psi}_2 (\bar{z},\bar{s},\bar{\beta}) = \{(z,s,\beta)\in \mathcal{B}^{r_\Psi}_2 (\bar{z},\bar{s},\bar{\beta})\mid H(z,\beta)=\mathbf 0\},
				\]
				where $H$ is defined in \eqref{eq: def_H} with a surjective derivative $\mathrm{D}H(\bar{z}, \bar{\beta})$  (see \eqref{eq:derivative_psi}) at $(\bar{z}, \bar{\beta})$.
				Consequently, by Lemma~\ref{lemma:normal_cone}, the normal cone to $\mathrm{SOL}_{\Psi}$ is given by 
				\begin{equation}
					\begin{aligned}
						N_{\mathrm{SOL}_{\Psi}}(\bar z,\bar s,\bar\beta)
						&= \mathrm{range}(\mathrm{D}H(\bar{z}, \bar{\beta})^*)\notag\\
						&= \Bigl\{
						(v_z,\mathbf 0,v_\beta)
						\ \Big|\
						v_z=X_{\bar{\mathcal{A}}}\eta,\ 
						(v_\beta)_{\bar{\mathcal{A}}}=\mathcal H\eta,\ 
						\eta\in\mathbb R^{|\bar{\mathcal{A}}|}
						\Bigr\},\label{eq:normal_Psi}
					\end{aligned}
				\end{equation}
				where $\mathcal{H}={\rm Diag}(\mathcal{H}_l)_{l\in I_{\bar\beta}^+}$ with each \(\mathcal H_l\) defined in \eqref{eq:def_psi_hessian}.
				
				\vspace{0.2cm}
				We are now ready to verify the regularity condition \eqref{eq:sum_normal_direction}. Based on the expressions of the normal cones in \eqref{eq:normal_X}, \eqref{eq:normal_L} and \eqref{eq:normal_Psi}, by letting 
				\begin{align*}
					v_X=&\ (\mathbf 0,u,-X^\intercal u),  \quad u \in \mathbb{R}^n,\\
					v_L=&\ (a,\bar{B}_{0}^\intercal \rho,\mathbf 0),\quad a\in\mathrm{ker}(\bar{B}_{0}),\  \rho \in\mathbb{R}^{|\bar J_0|}, \\
					v_\Psi=&\ ( X_{\bar{\mathcal{A}}} \eta,\ \mathbf 0,\ h), \quad  \text{with}\ h_{\bar{\mathcal{A}}} = \mathcal H \eta,\ \eta\in\mathbb R^{|\bar{\mathcal{A}}|},
				\end{align*}  
				the equation \eqref{eq:sum_normal_direction} gives
				\begin{align}
					a +  X_{\bar{\mathcal{A}}} \eta = \mathbf 0, \quad u+\bar{B}_{0}^\intercal \rho=\mathbf 0, \quad -X^\intercal u + h =\mathbf 0.\label{eq: regular_equation}
				\end{align}  
				Since \(a\in\ker(\bar B_0)\), applying \(\bar B_0\) to the first identity in \eqref{eq: regular_equation} yields
				\begin{equation}
					\bar B_0X_{\bar{\mathcal A}}\eta=0.
					\label{eq:B0XAeta}
				\end{equation}
				Substituting \(u=-\bar B_0^\top\rho\) into the third equation in \eqref{eq: regular_equation} and restricting to \(\bar{\mathcal A}\), we get
				\begin{equation}
					X_{\bar{\mathcal A}}^\top \bar B_0^\top \rho+\mathcal H\eta=0.
					\label{eq:rhoHeta}
				\end{equation}
				Taking the inner product of \eqref{eq:rhoHeta} with \(\eta\) and using \eqref{eq:B0XAeta}, we obtain
				\[
				0
				=
				\eta^\top X_{\bar{\mathcal A}}^\top \bar B_0^\top \rho+\eta^\top\mathcal H\eta
				=
				\rho^\top \bar B_0X_{\bar{\mathcal A}}\eta+\eta^\top\mathcal H\eta
				=
				\eta^\top\mathcal H\eta,
				\]
				which means \(\eta\in\ker(\mathcal H)\), as \(\mathcal H\) is positive semidefinite. This, together with \eqref{eq:B0XAeta} means
				\begin{equation*}
					\eta \in \ker(\bar{B}_{0} X_{\bar{\mathcal{A}}}) \cap \ker(\mathcal H). 
				\end{equation*}
				According to Assumption~\ref{condition_rank_nondegeneracy}, it follows that $\eta = \mathbf 0$. The equation \eqref{eq:rhoHeta} then reduces to
				\[
				X_{\bar{\mathcal{A}}}^\intercal \bar{B}_{0}^\intercal \rho = \mathbf 0,
				\]
				which means $\rho = \mathbf 0$ as $\bar{B}_{0} X_{\bar{\mathcal{A}}}$ has full row rank. This further implies $u = \mathbf 0$, $a = \mathbf 0$, $h = \mathbf 0$, and thus the desired equality \eqref{eq:sum_normal_direction}. 
				
				Finally, under \eqref{eq:regularity}, the metric subregularity for $\mathcal{T}$ can be reduced to establishing metric subregularity estimates for the individual mappings $\mathcal{T}_X$, $\mathcal{T}_L$, and $\mathcal{T}_{\Psi}$. Specifically, the metric subregularity for $\mathcal{T}_\Psi$ has been established in \eqref{eq:EB_Psi_local_final} (see Proposition~\ref{prop:error_bound_psi}); as for the mapping $\mathcal{T}_X$, it holds that
				\begin{align*}
					\mathrm{dist}((z,s,\beta), \mathrm{SOL}_X) \leq \| y+s-X\beta \| = \mathrm{dist}(\mathbf 0, \mathcal{T}_X(z,s,\beta)), \quad  \forall (z,s,\beta),
				\end{align*}
				since for any $(z,s,\beta)$ we always have $(z,X\beta -y, \beta)\in \mathrm{SOL}_X$;
				and for all $(z,s,\beta)$, we have
				\begin{align*}
					\operatorname{dist}((z,s,\beta),\mathrm{SOL}_L)
					&\leq \inf_{z'\in \partial L(s)} \|(z,s,\beta)-(z',s,\beta)\| \\
					&= \operatorname{dist}(z,\partial L(s)) = \operatorname{dist}(\mathbf 0,\mathcal T_L(z,s,\beta)).
				\end{align*}
				Combining the above metric subregularity estimates and the regularity property \eqref{eq:regularity}, we conclude that for all $(z,s,\beta)\in \mathcal B_2^{r_\ell}(\bar z,\bar s,\bar\beta)$, it holds that
				\begin{align*}
					& \mathrm{dist}((z,s,\beta), \mathrm{SOL}_{\mathrm{KKT}})) \\
					& \; \leq \kappa_\ell \max\{\mathrm{dist}(\mathbf 0, \mathcal{T}_X(z,s,\beta)), \operatorname{dist}(\mathbf 0,\mathcal T_L(z,s,\beta)), \operatorname{dist}\bigl(\mathbf 0, \mathcal T_\Psi(z,s,\beta)\bigr)\}\\
					& \; \leq \kappa_\ell \mathrm{dist}(\mathbf 0, \mathcal{T}(z,s,\beta)),
				\end{align*}
				where the constants $r_\ell=\min\{r_\Psi,r_R\}$ and $\kappa_\ell=\kappa_R \max\{1,\kappa_\Psi\}$ with $r_\Psi$, $\kappa_\Psi$ from Proposition~\ref{prop:error_bound_psi} and $r_R$, $\kappa_R$ from \eqref{eq:regularity}.
				This completes the proof.
				~\hfill\halmos
			\end{proof}
			
			With the metric subregularity for the KKT residual mapping $\mathcal{T}$, we next present the convergence properties of Algorithm~\ref{alg:PALM}. Our analysis builds upon the general proof framework of \citet[Theorem 2.3 \& Theorem 2.5]{li2020asymptotically}, but departs from it in a key aspect. While their results rely on an error bound condition associated with $\mathcal{T}$ (see \citet[Equation (3.9)]{li2020asymptotically}), we show that this condition can be relaxed to the metric subregularity, which we have established in Proposition~\ref{prop:error_bound}. This relaxation allows us to extend their convergence framework to our setting. With this modification, the overall proof follows a similar line of argument. We therefore focus on the essential differences, highlighting where the metric subregularity is invoked, and refer to \citet{li2020asymptotically} for the parts that carry over without changes.

			\begin{theorem}\label{thm:convergence}
				Let $\{(z^k, s^k, \beta^k)\}$ be the sequence generated by Algorithm \ref{alg:PALM}. Then the following conclusions hold.
				\begin{enumerate}
					\item[{\rm (1)}] The sequence $\{(z^k, s^k,\beta^k)\}$ converges  to a KKT solution $(\bar{z}, \bar{s},\bar{\beta})\in\mathrm{SOL}_{\mathrm{KKT}}$ of problem \eqref{mod:primaleq_const} and its dual \eqref{mod:dual}.
					\item[{\rm (2)}]  
					If Assumptions~\ref{condition_SC_Psi}, \ref{condition_SC_L} and \ref{condition_rank_nondegeneracy} hold at $(\bar{z}, \bar{s},\bar{\beta})$, then there exists $\bar{k}>0$ such that, for all $k > \bar{k}$,
					\begin{equation}\label{ineq: convergence rate}
						{\rm dist}_{\mathcal{M}}
						\left((z^{k+1}, s^{k+1}, \beta^{k+1}), \mathrm{SOL}_{\mathrm{KKT}}\right) \leq \mu_k {\rm dist}_{\mathcal{M}}
						\left((z^k, s^k, \beta^k), \mathrm{SOL}_{\mathrm{KKT}}\right),
					\end{equation}
					where $\mathcal{M} := {\rm Diag}(\tau I_n, I_n, I_p)$ and 
					$$\mu_k = \frac{1}{1 - \delta_k}\bigg( \delta_k + \frac{\kappa_\ell \gamma(1 + \delta_k)  }{\sqrt{\sigma_k^2 + \kappa_\ell^2 \gamma^2}} \bigg) 
					\to \mu_{\infty} \coloneq  \frac{ \kappa_\ell {\gamma}}{ \sqrt{\sigma_\infty^2 + \kappa_\ell^2 \gamma^2}}, $$ 
					with $\sigma_\infty := \lim_{k \to \infty} \sigma_k$, $\gamma =  \max\{1, \tau\}$, and $\kappa_\ell$ from \eqref{eq:EB}.
				\end{enumerate}
			\end{theorem} 
			\begin{proof}
				Part (1):
				Each iteration of Algorithm~\ref{alg:PALM} can be equivalently viewed as an inexact preconditioned proximal point step: 
				\begin{equation}
					(z^{k+1}, s^{k+1}, \beta^{k+1}) \approx \mathcal{P}_k(z^{k}, s^{k}, \beta^{k}) \coloneq (\mathcal{M} + \sigma_k\mathcal{T})^{-1}\mathcal{M} (z^k, s^k, \beta^k), \label{eq: def_Pk}
				\end{equation}
				with the inexactness criterion \eqref{stopping criteria}. 
				Therefore, by \citet[Proposition 3.2 \& Theorems 2.3]{li2020asymptotically}, the sequence $\{(z^k,s^k,\beta^k)\}$ converges to some KKT solution $(\bar{z}, \bar{s}, \bar{\beta})\in\mathrm{SOL}_{\mathrm{KKT}}$.

				Part (2): We follow the proof of \citet[Theorem 2.5]{li2020asymptotically}, with necessary modifications.
				First, by the definition of $\mathcal{P}_k$ in \eqref{eq: def_Pk}, for any $(z^{k},  s^{k}, \beta^{k})$, we have 
				$$
				\mathcal{M}( z^{k},  s^{k}, \beta^{k})\in (\mathcal{M} + \sigma_k\mathcal{T})(\mathcal{P}_k( z^{k},  s^{k}, \beta^{k}))=\mathcal{M} \mathcal{P}_k( z^{k},  s^{k}, \beta^{k}) + \sigma_k\mathcal{T}(\mathcal{P}_k( z^{k},  s^{k}, \beta^{k})). 
				$$
				By denoting $\mathcal{Q}_k = \mathcal{I}-\mathcal{P}_k$, we have 
				\begin{equation}
					\sigma_k^{-1} \mathcal{M} \mathcal{Q}_k ( z^{k},  s^{k}, \beta^{k}) \in \mathcal{T}(\mathcal{P}_k( z^{k},  s^{k}, \beta^{k})).\label{eq: Qinclusion}
				\end{equation}
				
				Next, we verify that the inclusion \eqref{eq: Qinclusion} falls within the scope of the metric subregularity result in Proposition~\ref{prop:error_bound}.
				Since $\{{\delta}_k\}$ is a positive summable sequence, there exists $k_0>0$ such that
				$$
				{r_\ell}\ {\min\{1,\sqrt{\tau}\}} - \sum_{k=k_0}^\infty {\delta}_k >0,
				$$ 
				where $r_\ell$ is given in \eqref{eq:EB}.
				Following from the convergence of the whole sequence $\{({z^k,s^k,\beta^k})\}$,  there exists $\bar{k}>k_0$ such that 
				$(z^{\bar{k}}, s^{\bar{k}}, \beta^{\bar{k}})$ satisfies
				\begin{equation}\label{eq:condition_kbar}
					\mathrm{dist}_\mathcal{M}((z^{\bar{k}}, s^{\bar{k}}, \beta^{\bar{k}}), \mathrm{SOL}_{\mathrm{KKT}}) < {r_\ell}\ {\min\{1,\sqrt{\tau}\}} - \sum_{k=k_0}^\infty {\delta}_k < {r_\ell}\ {\min\{1,\sqrt{\tau}\}} - \sum_{k=\bar{k}}^\infty \delta_k.
				\end{equation}
				In addition, using the same argument as in \citet[Theorem 2.3]{li2020asymptotically}, we obtain that 
				\begin{equation}\label{eq:dist_delta}
					\mathrm{dist}_\mathcal{M}((z^{k+1}, s^{k+1}, \beta^{k+1}), \mathrm{SOL}_{\mathrm{KKT}}) \leq \mathrm{dist}_\mathcal{M}((z^k, s^k, \beta^k), \mathrm{SOL}_{\mathrm{KKT}})+ \delta_k,\quad \forall k>0.
				\end{equation}
				Therefore, for any $k\geq \bar{k}$, summing \eqref{eq:dist_delta} from $\bar{k}$ to $k$ and invoking \eqref{eq:condition_kbar}, we obtain
				\begin{align*}
					\mathrm{dist}_\mathcal{M}((z^{k+1}, s^{k+1}, \beta^{k+1}), \mathrm{SOL}_{\mathrm{KKT}})
					&\leq \mathrm{dist}_\mathcal{M}((z^{\bar{k}}, s^{\bar{k}}, \beta^{\bar{k}}), \mathrm{SOL}_{\mathrm{KKT}})
					+ \sum_{k=\bar{k}}^k \delta_k \\
					&< r_\ell \min\{1,\sqrt{\tau}\}.
				\end{align*}
				Hence, for all $k > \bar{k}$, we have
				\begin{align*}
					& \operatorname{dist}_{\mathcal{M}}(\mathcal{P}_k( z^{k},  s^{k}, \beta^{k}), \mathrm{SOL}_{\mathrm{KKT}}) \leq   \|\mathcal{P}_k( z^{k},  s^{k}, \beta^{k}) - \Pi_{\mathrm{SOL}_{\mathrm{KKT}}}(z^{k},  s^{k}, \beta^{k}) \|_\mathcal{M} \\   
					& \quad =   \|\mathcal{P}_k( z^{k},  s^{k}, \beta^{k}) - \mathcal{P}_k (\Pi_{\mathrm{SOL}_{\mathrm{KKT}}}(z^{k},  s^{k}, \beta^{k}) ) \|_\mathcal{M} \\
					&\quad \leq \operatorname{dist}_{\mathcal{M}}(( z^{k},  s^{k}, \beta^{k}), \mathrm{SOL}_{\mathrm{KKT}}) \leq {r_\ell}\ {\min\{1,\sqrt{\tau}\}},
				\end{align*}
				where $\Pi_{\mathrm{SOL}_{\mathrm{KKT}}}$ denotes the projection onto the solution set $\mathrm{SOL}_{\mathrm{KKT}}$, which means
				\begin{align*} 
					&\mathrm{dist}(\mathcal{P}_k( z^{k},  s^{k}, \beta^{k}),\mathrm{SOL}_{\mathrm{KKT}}) \leq \frac{1}{\min\{1,\sqrt{\tau}\}}\mathrm{dist}_\mathcal{M}(\mathcal{P}_k( z^{k},  s^{k}, \beta^{k}),\mathrm{SOL}_{\mathrm{KKT}})  <  r_\ell.
				\end{align*}
				Then, applying Proposition~\ref{prop:error_bound} to the inclusion \eqref{eq: Qinclusion}, we have for all $k > \bar{k}$,  
				\begin{equation*}
					\mathrm{dist}(\mathcal{P}_k( z^{k},  s^{k}, \beta^{k}), \mathrm{SOL}_{\mathrm{KKT}}) 
					\leq \kappa_\ell \sigma_k^{-1} \|\mathcal{M} \mathcal{Q}_k ( z^{k},  s^{k}, \beta^{k})\|_2.
				\end{equation*}
				
				The remaining arguments follow standard convergence rate analysis for the preconditioned PPA and are omitted for brevity; see \citet[Theorem 2.5]{li2020asymptotically}. 
				This completes the proof.
				~\hfill \halmos \end{proof}
			\begin{remark}
				From the expression of $\mu_\infty$ given in Theorem \ref{thm:convergence}, we have $\mu_\infty<1$. That is, Theorem~\ref{thm:convergence} establishes the $Q$-linear convergence rate of the primal-dual sequence $\{(z^k, s^k, \beta^k)\}$. Moreover, if $\sigma_\infty = \infty$, then $\mu_\infty = 0$, and the convergence improves to asymptotically $Q$-superlinear.
			\end{remark}

			
			\subsection{A Semismooth Newton Method for Subproblem \texorpdfstring{\eqref{mod:iPALM}}{the Subproblem}}
			While PALM in Algorithm \ref{alg:PALM}
			enjoys favorable convergence properties, its practical efficiency critically depends on whether the subproblem \eqref{mod:iPALM} can be solved efficiently at each iteration.
			In this subsection, we develop a semismooth Newton (SSN) method to solve \eqref{mod:iPALM}, with implementation details and complexity analysis given in Appendix~\ref{sec:implementation}.
			
			Observing that \(\psi_k(\cdot)\) in~\eqref{mod:iPALM} is strongly convex and differentiable, problem~\eqref{mod:iPALM} admits a unique optimal solution $\bar{z}^{k+1}$, which  satisfies the associated first-order optimality condition: 
			\begin{equation}\label{gradpsi}
				\nabla \psi_k(z) = y + \sigma_k {\rm Prox}_{ L}\left(\frac{s^k}{\sigma_k} + z\right) -\sigma_k\lambda X{\rm Prox}_{ \Psi}\left( \frac{\beta^k}{\sigma_k \lambda} - \frac{X^\intercal z}{\lambda} \right) + \frac{\tau}{\sigma_k}(z - z^{k}) = \mathbf 0.
			\end{equation}
			
			Since both ${\rm Prox}_L(\cdot)$ and ${\rm Prox}_{\Psi}(\cdot)$ are Lipschitz continuous, the nonlinear and nonsmooth equation~\eqref{gradpsi} can be efficiently solved using an SSN method. 
			Specifically, we define the following set-valued mapping:
			\begin{equation}\label{eq:Hessian}
				\hat{\partial}^2 \psi_k(z) := \sigma_k \partial{\rm Prox}_L\left(\frac{s^k}{\sigma_k} + z\right) + \sigma_k X \partial{\rm Prox}_\Psi\left( \frac{\beta^k}{\sigma_k \lambda} - \frac{X^\intercal z}{\lambda} \right)  X^{\intercal} + \frac{\tau}{\sigma_k}I_n,
			\end{equation}
			where $\partial{\rm Prox}_L(\cdot)$ and $\partial{\rm Prox}_{\Psi}(\cdot)$ denote the generalized Jacobians of the proximal operators ${\rm Prox}_L(\cdot)$ and ${\rm Prox}_{\Psi}(\cdot)$, respectively. 
			It can be proved that $\nabla \psi_k(\cdot)$ is strongly semismooth with respect to $\hat{\partial}^2 \psi_k(z)(\cdot)$, and all elements in $\hat{\partial}^2 \psi_k(z)$ are positive definite for any $z\in \mathbb{R}^n$. See Appendix~\ref{sec:implementation} for details.
			
			With these preparations in place, we present the SSN method 
			for solving the nonsmooth equation~\eqref{gradpsi}. The algorithm proceeds as follows.
			
			\begin{algorithm}[H]\small
				\caption{Semismooth Newton Method for solving \eqref{mod:iPALM}}
				\label{alg:ssn}
				\begin{algorithmic}[1]
					\Require $\bar{\mu} \in (0, 1/2)$, $\bar{\eta} \in (0, 1)$, $\bar\tau \in (0,1]$, $\bar{\delta} \in (0, 1)$, and $z^{(0)}=z^k$. 
					\State Let $j=0$.
					\Repeat
					\State Choose $\Lambda^{(j)}\in \partial\mathrm{Prox}_L(s^k/\sigma_k + z^{(j)})$ and $V^{(j)}\in \partial \mathrm{Prox}_{\Psi} \left( \beta^k/(\sigma_k\lambda) - X^{\intercal}z^{(j)}/\lambda\right)$. 
					\State Let $H^{(j)} = \sigma_k\Lambda^{(j)} + \sigma_k X V^{(j)} X^{\intercal} + \frac{\tau}{\sigma_k} I_n$. Solve the following linear system
					\begin{equation}\label{eq:Newton}
						H^{(j)} d = -\nabla \psi_k(z^{(j)})
					\end{equation}
					exactly or approximately by the conjugate gradient (CG) algorithm to find $d^{(j)}$ such that $
					\| H^{(j)} d^{(j)} + \nabla \psi_k(z^{(j)})\|_2 \leq \min\{\bar{\eta}, \| \nabla \psi_k(z^{(j)})\|_2^{1+\bar\tau}\}$.
					\State  Set $\alpha_{(j)} = \bar{\delta}^{m_{(j)}}$, where $m_{(j)}$ is the first nonnegative integer $m$ for which
					\begin{equation*}
						\psi_k(z^{(j)} + \bar{\delta}^{m} d^{(j)}) \leq \psi_k(z^{(j)}) + \bar{\mu} \bar{\delta}^{m} \langle \nabla \psi_k(z^{(j)}), d^{(j)} \rangle.
					\end{equation*}
					
					\State Set $z^{(j+1)} = z^{(j)} + \alpha_{(j)} d^{(j)}$, $j \leftarrow j+1$.
					\Until Condition \eqref{stopping criteria} is satisfied
					\State \Return An approximate solution $z^{(j)}$ to \eqref{mod:iPALM}.
				\end{algorithmic}
			\end{algorithm}
			
			\begin{remark}
				In practical implementations, we set $\bar{\mu} = 10^{-3}$, $\bar{\eta} = 10^{-1}$, $\bar{\tau} = 10^{-1}$, and $\bar{\delta} = 0.5$.
			\end{remark}

			The convergence result of the SSN method 
			in Algorithm \ref{alg:ssn} is provided in the following theorem, which is a direct consequence of \citet[Proposition 3.3 \& Theorem 3.4]{zhao2010newton} and \citet[Theorem~3]{Li2018}.
			\begin{theorem}
				Let $\{z^{(j)}\}$ be the sequence generated by Algorithm \ref{alg:ssn}. 
				Then $\{z^{(j)}\}$ converges to the unique optimal solution $ \bar{z}^{k+1}$ of problem \eqref{mod:iPALM} and
				$$
				\|z^{(j+1)} - \bar{z}^{k+1}\|_2 = \mathcal{O}(\|z^{(j)} - \bar{z}^{k+1}\|_2^{1+\bar\tau}),
				$$
				where $\bar{\tau}\in (0,1]$ is given in the algorithm.
			\end{theorem}
			\section{Numerical Experiments}\label{sec:4}
			This section presents comprehensive numerical experiments to evaluate the proposed group Lasso regularized rank regression model in terms of statistical performance, and to assess the proposed algorithm in terms of computational efficiency. We begin by examining the estimation accuracy and robustness of the proposed estimator with the simulation-based tuning under various data-generating processes, comparing its performance against baseline models. 
			On the computational side, we conduct two complementary studies. We first evaluate the scalability of the proposed PALM framework for solving the group Lasso regularized rank regression. Second, in the absence of dedicated solvers for the group Lasso case, we instead  benchmark our algorithm on the $\ell_1$-regularized setting against the state-of-the-art PPMM algorithm from \cite{tang2023proximal}, which has been shown to outperform standard solvers such as ADMM and Gurobi and thus provides a strong baseline. All experiments were performed in MATLAB on an Apple M3 system running macOS 15.3.1 with 24 GB of RAM. Throughout this section, the group Lasso weights are chosen as $w_l = \sqrt{|\mathcal{G}_l|}$.
			
			\subsection{Statistical Performance of Group Lasso Regularized Rank Regression}\label{subsec:rank_grouptable}

			We compare the performance of our proposed group Lasso regularized rank regression model \eqref{mod:rankgrouplasso0} (Rank\_GLasso) against five alternatives: the group Lasso regularized least squares regression model (LS\_GLasso) \citep{yuan2006model,Zhang2020} implemented in \texttt{SparseGroupLasso} package\footnote{\url{https://github.com/YangjingZhang/SparseGroupLasso}};
			the $\ell_1$-regularized rank regression model (Rank\_Lasso) \citep{Wang2020ATR}, solved by the state-of-the-art PPMM algorithm \citep{tang2023proximal}; 
			and three robust group Lasso estimators implemented in R package \texttt{hrqglas}\footnote{\url{https://CRAN.R-project.org/package=hrqglas}} \citep{Sherwood2022}: Huber\_GLasso, QR\_GLasso ($\tau=0.25$), and QR\_GLasso ($\tau=0.5$).
			The last three models combine group Lasso regularization with different loss functions: Huber\_GLasso uses the Huber loss, while QR\_GLasso($\tau$) uses the quantile check loss at quantile level $\tau$.

			Each model involves a regularization parameter $\lambda$. 
			The two rank-based methods admit data-driven simulation-based selection rules for $\lambda$: for Rank\_GLasso, $\lambda$ is computed by Algorithm~\ref{alg:lambda_estimation}; for Rank\_Lasso, we adopt the selection rule proposed in \citet{Wang2020ATR} (see \texttt{TFRE}\footnote{\url{https://github.com/yunanwu123/TFRE}}). 
			The remaining four methods require parameter tuning: for LS\_GLasso, we select $\lambda$ via $5$-fold cross-validation over 10 logarithmically spaced values between $\lambda_{\text{max}} = \max_{1\leq l\leq g}\|X_{{\cal G}_l}^{\intercal} y\|_2/\sqrt{|{\cal G}_l|}$ and $\lambda_{\text{min}}=10^{-4} \lambda_{\text{max}}$, with the optimal value chosen to minimize the validation mean squared error; for the three \texttt{hrqglas} estimators, $\lambda$ is selected using the built-in $5$-fold cross-validation routine.

			\paragraph{Data Generation.} To systematically evaluate the competing models, we adopt the linear regression framework \eqref{mod:linear regression} and generate synthetic data by specifying the distributions of covariates, true coefficient vector, and noise.
			The rows of the design matrix $X$ are sampled i.i.d. from
			the $p$-dimensional multivariate Gaussian distribution $\mathcal{N}(0,\Sigma)$, where the covariance structure $\Sigma$ is chosen from two settings:
			\begin{itemize}[leftmargin=3.5em, topsep=3pt, itemsep=1pt]
				\item[(C1)] Equi-correlation: $\Sigma_{ij} = 0.3$ for all $i \neq j$, and $\Sigma_{ij} = 1$ for $i=j$.
				\item[(C2)] Auto-regressive (AR(1)): $\Sigma_{ij} = 0.9^{|i-j|}$.
			\end{itemize}
			The true coefficient vector $\beta^* \in \mathbb{R}^p$ exhibits group sparsity. 
			We partition the covariates into $g = p/20$ non-overlapping groups $\mathcal{G} = \{{\cal G}_1,\dots,{\cal G}_g\}$, among which only 1\% of groups are active, i.e., contain nonzero coefficients. For simplicity, the active groups are taken to be the first $m$ groups in $\mathcal{G}$. We consider two signal patterns:
			\begin{itemize}[leftmargin=3.5em, topsep=3pt, itemsep=1pt]
				\item[(S1)]  Uniform signal: $ (\beta^*)_{{\cal G}_l} = \sqrt{3}\cdot \mathbf 1_{|{\cal G}_l|}, \quad l\in[m]$.
				\item[(S2)] Group-wise linear decaying signal: $(\beta^*)_{{\cal G}_l} = \left(2 - (l-1)/4\right) \cdot \mathbf 1_{|{\cal G}_l|}, \ l\in[m]$.
			\end{itemize}
			We consider six different noise distributions, including three Gaussian cases with different variances and three heavy-tailed alternatives:
			\begin{itemize}[leftmargin=3.5em, topsep=3pt, itemsep=1pt]
				\item[(E1)] Gaussian noise with low variance: $\epsilon \sim \mathcal{N}(0, 0.25)$;
				\item[(E2)] Gaussian noise with standard variance: $\epsilon \sim \mathcal{N}(0, 1)$;
				\item[(E3)] Gaussian noise with high variance: $\epsilon \sim \mathcal{N}(0, 2)$;
				\item[(E4)] Gaussian mixture with outliers: $\epsilon \sim 0.95\mathcal{N}(0, 1) + 0.05\mathcal{N}(0, 100)$, denoted as $\mathcal{N}_{\mathrm{mix}}$;
				\item[(E5)] Student’s t-distribution with 4 degrees of freedom: $\epsilon \sim  t_4$;
				\item[(E6)] Cauchy distribution: $\epsilon \sim \text{Cauchy}(0, 1)$.
			\end{itemize}
			
			\begin{table}[htbp]
				\centering
				\caption{Comparison of models on data generated with $X$ under covariance structure (C1) and $\beta^*$ under signal pattern (S1) across all error distributions (E1)--(E6).}
				\label{tab:group1.1_new}
				\setlength{\tabcolsep}{4pt}
				\begin{tabular}{c l c c c c c c}
					\toprule
					\textbf{Noise}  & \textbf{Method} & \textbf{$\lambda$} & \textbf{Time} & \textbf{$\ell_2$ Error} & \textbf{ME} & \textbf{FP} & \textbf{FN} \\
					\midrule
					\multirow{6}{*}{$\mathcal{N}(0, 0.25)$}
					& Rank\_Lasso & 0.217 & \textbf{00:05} & 1.36e\texttt{+}0 & 2.45e-1 & \textbf{305} & 0 \\
					& Rank\_GLasso & 0.092 & 00:07 & 3.34e-1 & 5.67e-2 & 514 & 0 \\
					& Huber\_GLasso & 0.009 & 02:27 & 4.72e\texttt{+}0 & 8.43e-1 & 3023 & 0 \\
					& QR\_GLasso ($\tau$=0.25) & 0.029 & 02:40 & 3.38e-1 & 6.45e-2 & 405 & 0 \\
					& QR\_GLasso ($\tau$=0.5) & 0.069 & 02:52 & 3.27e-1 & 7.46e-2 & 480 & 0 \\
					& LS\_GLasso & 17.902 & 00:59 & \textbf{3.10e-1 } & \textbf{5.06e-2 }& 472 & 0 \\
					\addlinespace
					\multirow{6}{*}{$\mathcal{N}(0, 1)$}
					& Rank\_Lasso & 0.220 & \textbf{00:03} & 6.97e\texttt{+}0 & 2.17e\texttt{+}0 & \textbf{347} & 0 \\
					& Rank\_GLasso & 0.095 & 00:06 & 7.51e-1 & 2.90e-1 & 538 & 0 \\
					& Huber\_GLasso & 0.009 & 02:38 & 5.65e\texttt{+}0 & 1.38e\texttt{+}0 & 2986 & 0 \\
					& QR\_GLasso ($\tau$=0.25) & 0.031 & 02:50 & 7.81e-1 & 3.64e-1 & 504 & 0 \\
					& QR\_GLasso ($\tau$=0.5) & 0.056 & 02:36 & 7.40e-1 & 3.62e-1 & 440 & 0 \\
					& LS\_GLasso & 49.096 & 01:04 & \textbf{6.73e-1} & \textbf{2.61e-1} & 498 & 0 \\
					\addlinespace
					\multirow{6}{*}{$\mathcal{N}(0, 2)$}
					& Rank\_Lasso & 0.216 & \textbf{00:04} & 9.40e\texttt{+}0 & 3.86e\texttt{+}0 & 337 & 1 \\
					& Rank\_GLasso & 0.090 & 00:05 & 8.98e-1 & 4.58e-1 & 572 & 0 \\
					& Huber\_GLasso & 0.009 & 02:37 & 5.97e\texttt{+}0 & 2.39e\texttt{+}0 & 2885 & 0 \\
					& QR\_GLasso ($\tau$=0.25) & 0.039 & 02:46 & 9.43e-1 & 5.73e-1 & 478 & 0 \\
					& QR\_GLasso ($\tau$=0.5) & 0.073 & 02:45 & \textbf{8.14e-1} & 4.88e-1 & \textbf{320} & 0 \\
					& LS\_GLasso & 47.153 & 01:06 & 8.77e-1 & \textbf{4.34e-1} & 616 & 0 \\
					\addlinespace
					\multirow{6}{*}{$\mathcal{N}_{\text{mix}}$}
					& Rank\_Lasso & 0.218 & \textbf{00:04} & 9.01e\texttt{+}0 & 5.09e\texttt{+}0 & \textbf{337} & 1 \\
					& Rank\_GLasso & 0.094 & 00:06 & \textbf{7.83e-1} & \textbf{3.59e-1} & 668 & 0 \\
					& Huber\_GLasso & 0.009 & 02:43 & 5.28e\texttt{+}0 & 2.80e\texttt{+}0 & 2885 & 0 \\
					& QR\_GLasso ($\tau$=0.25) & 0.068 & 02:47 & 9.29e-1 & 6.74e-1 & 560 & 0 \\
					& QR\_GLasso ($\tau$=0.5) & 0.042 & 02:48 & \textbf{8.02e-1} & 4.13e-1 & 691 & 0 \\
					& LS\_GLasso & 46.857 & 01:10 & 1.24e\texttt{+}0 & 9.07e-1 & 617 & 0 \\
					\addlinespace
					\multirow{6}{*}{$ t_4$}
					& Rank\_Lasso & 0.221 & \textbf{00:04} & 4.13e\texttt{+}0 & 2.54e\texttt{+}0 & \textbf{327} & 0 \\
					& Rank\_GLasso & 0.097 & 00:06 & \textbf{8.41e-1} & \textbf{4.08e-1} & 563 & 0 \\
					& Huber\_GLasso & 0.009 & 02:44 & 5.40e\texttt{+}0 & 2.55e\texttt{+}0 & 3055 & 0 \\
					& QR\_GLasso ($\tau$=0.25) & 0.071 & 02:50 & 9.09e-1 & 7.20e-1 & 420 & 0 \\
					& QR\_GLasso ($\tau$=0.5) & 0.046 & 02:48 & 8.37e-1 & \textbf{4.08e-1} & 458 & 0 \\
					& LS\_GLasso & 49.416 & 01:08 & 8.81e-1 & 4.75e-1 & 486 & 0 \\
					\addlinespace
					\multirow{6}{*}{$\text{Cauchy}(0, 1)$}
					& Rank\_Lasso & 0.216 & \textbf{00:03} & 1.19e\texttt{+}1 & 1.43e\texttt{+}1 & 330 & 12 \\
					& Rank\_GLasso & 0.092 & 00:06 & 1.14e\texttt{+}0 & 6.43e-1 & 554 & 0 \\
					& Huber\_GLasso & 0.008 & 02:49 & 6.36e\texttt{+}0 & 8.39e\texttt{+}0 & 3284 & 0 \\
					& QR\_GLasso ($\tau$=0.25) & 0.060 & 02:52 & 1.15e\texttt{+}0 & 1.10e\texttt{+}0 & \textbf{340} & 0 \\
					& QR\_GLasso ($\tau$=0.5) & 0.052 & 02:53 & \textbf{9.73e-1} & \textbf{5.67e-1} & 577 & 0 \\
					& LS\_GLasso & 1005.356 & 01:28 & 6.47e\texttt{+}0 & 3.53e\texttt{+}1 & 456 & 0 \\
					\bottomrule
				\end{tabular}
			\end{table}
			
			\begin{table}[htbp]
				\centering
				\caption{Comparison of models on data generated with $X$ under covariance structure (C1) and $\beta^*$ under signal pattern (S2) across all error distributions (E1)–(E6).}
				\label{tab:group1.2_new}
				\setlength{\tabcolsep}{4pt}
				\begin{tabular}{c l c c c c c c}
					\toprule
					\textbf{Noise}  & \textbf{Method} & \textbf{$\lambda$} & \textbf{Time} & \textbf{$\ell_2$ Error} & \textbf{ME} & \textbf{FP} & \textbf{FN} \\
					\midrule
					\multirow{6}{*}{$\mathcal{N}(0, 0.25)$}
					& Rank\_Lasso & 0.217 & \textbf{00:05} & 1.36e\texttt{+}0 & 2.45e-1 & \textbf{305} & 0 \\
					& Rank\_GLasso & 0.092 & 00:07 & 3.34e-1 & 5.65e-2 & 514 & 0 \\
					& Huber\_GLasso & 0.009 & 02:46 & 3.95e\texttt{+}0 & 6.25e-1 & 2964 & 0 \\
					& QR\_GLasso ($\tau$=0.25) & 0.042 & 02:52 & \textbf{3.12e-1} & 6.58e-2 & 320 & 0 \\
					& QR\_GLasso ($\tau$=0.5) & 0.063 & 02:53 & 3.31e-1 & 7.24e-2 & 460 & 0 \\
					& LS\_GLasso & 16.838 & 01:03 & \textbf{3.12e-1} & \textbf{5.07e-2} & 466 & 0 \\
					\addlinespace
					\multirow{6}{*}{$\mathcal{N}(0, 1)$}
					& Rank\_Lasso & 0.220 & \textbf{00:04} & 6.56e\texttt{+}0 & 2.03e\texttt{+}0 & \textbf{340} & 1 \\
					& Rank\_GLasso & 0.095 & 00:06 & 7.47e-1 & 2.89e-1 & 538 & 0 \\
					& Huber\_GLasso & 0.009 & 02:43 & 5.06e\texttt{+}0 & 1.23e\texttt{+}0 & 2846 & 0 \\
					& QR\_GLasso ($\tau$=0.25) & 0.033 & 02:52 & 7.72e-1 & 3.61e-1 & 506 & 0 \\
					& QR\_GLasso ($\tau$=0.5) & 0.055 & 02:50 & 7.37e-1 & 3.60e-1 & 440 & 0 \\
					& LS\_GLasso & 45.734 & 01:10 & \textbf{6.76e-1} & \textbf{2.58e-1} & 534 & 0 \\
					\addlinespace
					\multirow{6}{*}{$\mathcal{N}(0, 2)$}
					& Rank\_Lasso & 0.216 & \textbf{00:04} & 8.24e\texttt{+}0 & 3.50e\texttt{+}0 & \textbf{337} & 1 \\
					& Rank\_GLasso & 0.090 & 00:05 & 8.92e-1 & 4.54e-1 & 590 & 0 \\
					& Huber\_GLasso & 0.009 & 02:45 & 5.18e\texttt{+}0 & 2.18e\texttt{+}0 & 2672 & 0 \\
					& QR\_GLasso ($\tau$=0.25) & 0.041 & 02:56 & 9.24e-1 & 5.60e-1 & 440 & 0 \\
					& QR\_GLasso ($\tau$=0.5) & 0.073 & 02:53 & \textbf{8.03e-1} & 4.79e-1 & 360 & 0 \\
					& LS\_GLasso & 44.240 & 01:11 & 8.84e-1 & \textbf{4.38e-1} & 636 & 0 \\
					\addlinespace
					\multirow{6}{*}{$\mathcal{N}_{\text{mix}}$}
					& Rank\_Lasso & 0.218 & \textbf{00:04} & 8.32e\texttt{+}0 & 4.54e\texttt{+}0 & \textbf{338} & 1 \\
					& Rank\_GLasso & 0.094 & 00:06 & \textbf{7.77e-1} & \textbf{3.56e-1} & 687 & 0 \\
					& Huber\_GLasso & 0.009 & 02:49 & 4.72e\texttt{+}0 & 2.54e\texttt{+}0 & 2907 & 0 \\
					& QR\_GLasso ($\tau$=0.25) & 0.067 & 02:53 & 9.19e-1 & 6.57e-1 & 600 & 0 \\
					& QR\_GLasso ($\tau$=0.5) & 0.040 & 02:56 & 7.98e-1 & 4.06e-1 & 744 & 0 \\
					& LS\_GLasso & 122.444 & 01:14 & 1.02e\texttt{+}0 & 7.10e-1 & 500 & 0 \\
					\addlinespace
					\multirow{6}{*}{$ t_4$}
					& Rank\_Lasso & 0.221 & \textbf{00:03} & 4.12e\texttt{+}0 & 2.54e\texttt{+}0 & \textbf{327} & 0 \\
					& Rank\_GLasso & 0.097 & 00:06 & 8.32e-1 & 4.01e-1 & 562 & 0 \\
					& Huber\_GLasso & 0.009 & 02:47 & 4.79e\texttt{+}0 & 2.29e\texttt{+}0 & 2981 & 0 \\
					& QR\_GLasso ($\tau$=0.25) & 0.074 & 02:52 & 8.74e-1 & 6.96e-1 & 400 & 0 \\
					& QR\_GLasso ($\tau$=0.5) & 0.048 & 02:51 & \textbf{8.21e-1} & \textbf{3.99e-1} & 459 & 0 \\
					& LS\_GLasso & 46.660 & 01:09 & 8.82e-1 & 4.75e-1 & 481 & 0 \\
					\addlinespace
					\multirow{6}{*}{$\text{Cauchy}(0, 1)$}
					& Rank\_Lasso & 0.216 & \textbf{00:03} & 1.16e\texttt{+}1 & 1.29e\texttt{+}1 & 347 & 18 \\
					& Rank\_GLasso & 0.092 & 00:07 & 1.13e\texttt{+}0 & 6.35e-1 & 574 & 0 \\
					& Huber\_GLasso & 0.007 & 02:52 & 5.91e\texttt{+}0 & 8.34e\texttt{+}0 & 3368 & 0 \\
					& QR\_GLasso ($\tau$=0.25) & 0.060 & 02:51 & 1.14e\texttt{+}0 & 1.08e\texttt{+}0 & \textbf{340} & 0 \\
					& QR\_GLasso ($\tau$=0.5) & 0.052 & 02:55 & \textbf{9.64e-1} & \textbf{5.59e-1} & 557 & 0 \\
					& LS\_GLasso & 943.661 & 01:31 & 6.53e\texttt{+}0 & 3.41e\texttt{+}1 & 474 & 0 \\
					\bottomrule
				\end{tabular}
			\end{table}

			\paragraph{Estimation accuracy metric.} We evaluate the estimation performance of the competing models using the following three metrics. The first is the $\ell_2$ error, defined as $\|\hat{\beta} - \beta^*\|_2$, which measures the Euclidean distance between the estimated coefficient vector $\hat{\beta}$ and the true coefficient vector $\beta^*$. The second is the model error (ME), given by $(\hat{\beta} - \beta^*)^\intercal \Sigma_X (\hat{\beta} - \beta^*)$, where $\Sigma_X$ is the sample covariance matrix of the design matrix $X$. This metric quantifies the prediction discrepancy with respect to the observed design. The third metric assesses the support recovery performance through the number of false positives (FP) and false negatives (FN). Specifically, FP is the count of incorrectly selected variables, computed as $\sum_{j=1}^p \mathbb{I}\{\beta_j^* = 0 \text{ and } \hat{\beta}_j \neq 0\}$, while FN is the count of missed true variables, computed as $\sum_{j=1}^p \mathbb{I}\{\beta_j^* \neq 0 \text{ and } \hat{\beta}_j = 0\}$.
			
			To provide a comprehensive comparison, we consider all combinations of the covariance structures (C1–C2) for $X$, the signal patterns (S1–S2) for $\beta^*$, and the error distributions (E1–E6). 
			The corresponding results with $n=500$ and $p=8000$ are summarized in Tables~\ref{tab:group1.1_new}--\ref{tab:group2.2_new}, each focusing on a $(\text{covariance}, \text{signal})$ setting under a variety of noise distributions. 
			For each setting, we compare the six models by presenting their estimation accuracy metrics, computational time, and the selected regularization parameter~$\lambda$. 
			The reported computational time includes both the cost of selecting~$\lambda$ and the final fitting time. In each table, the best (i.e., smallest) values for time, $\ell_2$ error, model error (ME), and false positives (FP) are highlighted in bold for ease of comparison.

			\begin{table}[htbp]
				\centering
				\caption{Comparison of models on data generated with $X$ under covariance structure (C2) and $\beta^*$ under signal pattern (S1) across all error distributions (E1)–(E6).}
				\label{tab:group2.1_new}
				\setlength{\tabcolsep}{4pt}
				\begin{tabular}{c l c c c c c c}
					\toprule
					\textbf{Noise}  & \textbf{Method} & \textbf{$\lambda$} & \textbf{Time} & \textbf{$\ell_2$ Error} & \textbf{ME} & \textbf{FP} & \textbf{FN} \\
					\midrule
					\multirow{6}{*}{$\mathcal{N}(0, 0.25)$}
					& Rank\_Lasso & 0.223 & \textbf{00:02} & 7.25e-1 & 1.07e-1 &   1 & 0 \\
					& Rank\_GLasso & 0.153 & \textbf{00:02} & 4.53e-1 & 5.71e-2 &   \textbf{0} & 0 \\
					& Huber\_GLasso & 0.030 & 02:05 & \textbf{4.27e-1} & 5.23e-2 & 530 & 0 \\
					& QR\_GLasso ($\tau$=0.25) & 0.039 & 12:51 & 5.01e-1 & 6.17e-2 &   \textbf{0} & 0 \\
					& QR\_GLasso ($\tau$=0.5) & 0.039 & 14:35 & 5.17e-1 & 6.12e-2 &  36 & 0 \\
					& LS\_GLasso & 14.117 & 00:44 & 5.45e-1 & \textbf{4.65e-2} & 430 & 0 \\
					\addlinespace
					\multirow{6}{*}{$\mathcal{N}(0, 1)$}
					& Rank\_Lasso & 0.222 & \textbf{00:02} & 1.61e\texttt{+}0 & 4.41e-1 &   \textbf{1} & 0 \\
					& Rank\_GLasso & 0.152 & \textbf{00:02} & 6.51e-1 & 1.74e-1 &  19 & 0 \\
					& Huber\_GLasso & 0.035 & 02:16 & \textbf{5.38e-1} & 1.45e-1 & 153 & 0 \\
					& QR\_GLasso ($\tau$=0.25) & 0.041 & 07:14 & 8.74e-1 & 2.78e-1 &  20 & 0 \\
					& QR\_GLasso ($\tau$=0.5) & 0.038 & 08:27 & 8.25e-1 & 1.57e-1 &  58 & 0 \\
					& LS\_GLasso & 36.192 & 00:54 & 7.25e-1 & \textbf{1.19e-1} &  97 & 0 \\
					\addlinespace
					\multirow{6}{*}{$\mathcal{N}(0, 2)$}
					& Rank\_Lasso & 0.226 & \textbf{00:02} & 2.38e\texttt{+}0 & 9.10e-1 &   \textbf{0} & 0 \\
					& Rank\_GLasso & 0.153 & \textbf{00:02} & 9.13e-1 & 2.96e-1 &   \textbf{0} & 0 \\
					& Huber\_GLasso & 0.027 & 02:21 & \textbf{7.78e-1} & \textbf{2.52e-1} & 774 & 0 \\
					& QR\_GLasso ($\tau$=0.25) & 0.034 & 07:29 & 1.16e\texttt{+}0 & 3.34e-1 &  57 & 0 \\
					& QR\_GLasso ($\tau$=0.5) & 0.037 & 07:56 & 1.15e\texttt{+}0 & 2.84e-1 & 162 & 0 \\
					& LS\_GLasso & 35.810 & 01:00 & 1.18e\texttt{+}0 & 3.05e-1 & 700 & 0 \\
					\addlinespace
					\multirow{6}{*}{$\mathcal{N}_{\text{mix}}$}
					& Rank\_Lasso & 0.223 & \textbf{00:02} & 1.93e\texttt{+}0 & 7.07e-1 &   \textbf{0} & 0 \\
					& Rank\_GLasso & 0.151 & 00:03 & 7.56e-1 & 2.61e-1 &   \textbf{0} & 0 \\
					& Huber\_GLasso & 0.013 & 02:15 & \textbf{7.28e-1} & 5.78e-1 & 2242 & 0 \\
					& QR\_GLasso ($\tau$=0.25) & 0.034 & 07:12 & 9.20e-1 & 2.70e-1 &  80 & 0 \\
					& QR\_GLasso ($\tau$=0.5) & 0.037 & 09:30 & 9.71e-1 & \textbf{2.21e-1} & 115 & 0 \\
					& LS\_GLasso & 99.537 & 01:03 & 1.59e\texttt{+}0 & 9.05e-1 & 319 & 0 \\
					\addlinespace
					\multirow{6}{*}{$ t_4$}
					& Rank\_Lasso & 0.224 &\textbf{ 00:02} & 2.10e\texttt{+}0 & 6.28e-1 &   1 & 0 \\
					& Rank\_GLasso & 0.151 & \textbf{00:02} & 7.82e-1 & 2.34e-1 &   \textbf{0} & 0 \\
					& Huber\_GLasso & 0.032 & 02:22 & \textbf{5.98e-1} & \textbf{1.50e-1} & 171 & 0 \\
					& QR\_GLasso ($\tau$=0.25) & 0.035 & 07:52 & 1.02e\texttt{+}0 & 2.62e-1 &  39 & 0 \\
					& QR\_GLasso ($\tau$=0.5) & 0.041 & 07:52 & 9.53e-1 & 2.11e-1 &  40 & 0 \\
					& LS\_GLasso & 35.526 & 01:00 & 1.21e\texttt{+}0 & 3.35e-1 & 646 & 0 \\
					\addlinespace
					\multirow{6}{*}{$\text{Cauchy}(0, 1)$}
					& Rank\_Lasso & 0.221 & \textbf{00:02} & 4.83e\texttt{+}0 & 3.37e\texttt{+}0 &   \textbf{0} & 0 \\
					& Rank\_GLasso & 0.152 & 00:05 & 9.89e-1 & 6.63e-1 &   \textbf{0} & 0 \\
					& Huber\_GLasso & 0.040 & 02:30 & \textbf{7.21e-1} & \textbf{3.10e-1} &  41 & 0 \\
					& QR\_GLasso ($\tau$=0.25) & 0.032 & 07:11 & 1.44e\texttt{+}0 & 1.12e\texttt{+}0 & 199 & 0 \\
					& QR\_GLasso ($\tau$=0.5) & 0.046 & 06:43 & 1.02e\texttt{+}0 & 3.60e-1 &   \textbf{0} & 0 \\
					& LS\_GLasso & 702.484 & 01:24 & 4.38e\texttt{+}0 & 5.69e\texttt{+}1 & 746 & 0 \\
					\bottomrule
				\end{tabular}
			\end{table}

			\begin{table}[htbp]
				\centering
				\caption{Comparison of models on data generated with $X$ under covariance structure (C2) and $\beta^*$ under signal pattern (S2) across all error distributions (E1)–(E6).}
				\label{tab:group2.2_new}
				\setlength{\tabcolsep}{4pt}
				\begin{tabular}{c l c c c c c c}
					\toprule
					\textbf{Noise}  & \textbf{Method} & \textbf{$\lambda$} & \textbf{Time} & \textbf{$\ell_2$ Error} & \textbf{ME} & \textbf{FP} & \textbf{FN} \\
					\midrule
					\multirow{6}{*}{$\mathcal{N}(0, 0.25)$}
					& Rank\_Lasso & 0.223 & \textbf{00:0}2 & 7.25e-1 & 1.07e-1 & 1 & 0 \\
					& Rank\_GLasso & 0.153 & \textbf{00:02} & 4.45e-1 & 5.63e-2 & \textbf{0} & 0 \\
					& Huber\_GLasso & 0.030 & 02:02 & \textbf{4.14e-1} & 5.06e-2 & 529 & 0 \\
					& QR\_GLasso ($\tau$=0.25) & 0.039 & 11:54 & 4.92e-1 & 6.15e-2 & \textbf{0} & 0 \\
					& QR\_GLasso ($\tau$=0.5) & 0.039 & 15:30 & 5.16e-1 & 6.04e-2 & 36 & 0 \\
					& LS\_GLasso & 14.055 & 00:42 & 5.38e-1 & \textbf{4.60e-2} & 411 & 0 \\
					\addlinespace
					\multirow{6}{*}{$\mathcal{N}(0, 1)$}
					& Rank\_Lasso & 0.222 & \textbf{00:02} & 1.61e\texttt{+}0 & 4.41e-1 & \textbf{1} & 0 \\
					& Rank\_GLasso & 0.152 & \textbf{00:02} & 6.21e-1 & 1.72e-1 & 19 & 0 \\
					& Huber\_GLasso & 0.032 & 02:14 & \textbf{5.33e-1} & 1.33e-1 & 199 & 0 \\
					& QR\_GLasso ($\tau$=0.25) & 0.041 & 07:50 & 8.23e-1 & 2.69e-1 & 20 & 0 \\
					& QR\_GLasso ($\tau$=0.5) & 0.035 & 07:43 & 8.22e-1 & 1.49e-1 & 78 & 0 \\
					& LS\_GLasso & 35.694 & 00:53 & 6.99e-1 & \textbf{1.17e-1} & 96 & 0 \\
					\addlinespace
					\multirow{6}{*}{$\mathcal{N}(0, 2)$}
					& Rank\_Lasso & 0.226 & 00:02 & 2.38e\texttt{+}0 & 9.10e-1 & \textbf{0} & 0 \\
					& Rank\_GLasso & 0.153 & 00:02 & 8.82e-1 & 2.89e-1 & \textbf{0} & 0 \\
					& Huber\_GLasso & 0.028 & 02:40 & \textbf{7.56e-1} & \textbf{2.47e-1} & 575 & 0 \\
					& QR\_GLasso ($\tau$=0.25) & 0.036 & 08:14 & 1.12e\texttt{+}0 & 3.26e-1 & 57 & 0 \\
					& QR\_GLasso ($\tau$=0.5) & 0.038 & 06:46 & 1.10e\texttt{+}0 & 2.75e-1 & 139 & 0 \\
					& LS\_GLasso & 35.748 & 01:00 & 1.16e\texttt{+}0 & 3.01e-1 & 699 & 0 \\
					\addlinespace
					\multirow{6}{*}{$\mathcal{N}_{\text{mix}}$}
					& Rank\_Lasso & 0.223 & \textbf{00:02} & 1.93e\texttt{+}0 & 7.07e-1 & \textbf{0} & 0 \\
					& Rank\_GLasso & 0.151 & 00:03 & 7.34e-1 & 2.56e-1 & \textbf{0} & 0 \\
					& Huber\_GLasso & 0.013 & 02:12 & \textbf{7.26e-1} & 5.03e-1 & 2324 & 0 \\
					& QR\_GLasso ($\tau$=0.25) & 0.034 & 07:40 & 8.93e-1 & 2.63e-1 & 80 & 0 \\
					& QR\_GLasso ($\tau$=0.5) & 0.041 & 09:07 & 9.07e-1 & \textbf{2.19e-1} & 88 & 0 \\
					& LS\_GLasso & 97.747 & 01:04 & 1.55e\texttt{+}0 & 9.04e-1 & 338 & 0 \\
					\addlinespace
					\multirow{6}{*}{$ t_4$}
					& Rank\_Lasso & 0.224 & \textbf{00:02} & 2.10e\texttt{+}0 & 6.28e-1 & 1 & 0 \\
					& Rank\_GLasso & 0.151 & 00:03 & 7.47e-1 & 2.28e-1 & \textbf{0} & 0 \\
					& Huber\_GLasso & 0.032 & 02:40 & \textbf{5.74e-1} & \textbf{1.49e-1} & 192 & 0 \\
					& QR\_GLasso ($\tau$=0.25) & 0.035 & 07:32 & 9.91e-1 & 2.57e-1 & 40 & 0 \\
					& QR\_GLasso ($\tau$=0.5) & 0.041 & 08:20 & 9.11e-1 & 2.05e-1 & \textbf{40} & 0 \\
					& LS\_GLasso & 35.548 & 00:59 & 1.15e\texttt{+}0 & 3.25e-1 & 642 & 0 \\
					\addlinespace
					\multirow{6}{*}{$\text{Cauchy}(0, 1)$}
					& Rank\_Lasso & 0.221 & \textbf{00:02} & 4.83e\texttt{+}0 & 3.37e\texttt{+}0 & \textbf{0} & 0 \\
					& Rank\_GLasso & 0.152 & 00:05 & 9.54e-1 & 6.54e-1 & \textbf{0} & 0 \\
					& Huber\_GLasso & 0.038 & 02:31 & \textbf{6.91e-1} & \textbf{2.90e-1} & 42 & 0 \\
					& QR\_GLasso ($\tau$=0.25) & 0.032 & 06:38 & 1.40e\texttt{+}0 & 1.12e\texttt{+}0 & 182 & 0 \\
					& QR\_GLasso ($\tau$=0.5) & 0.046 & 06:13 & 9.86e-1 & 3.51e-1 & \textbf{0} & 0 \\
					& LS\_GLasso & 662.720 & 01:30 & 4.61e\texttt{+}0 & 6.51e\texttt{+}1 & 751 & 0 \\
					\bottomrule
				\end{tabular}
			\end{table}

			The simulation results in Tables~\ref{tab:group1.1_new}--\ref{tab:group2.2_new} demonstrate the clear advantages of the proposed Rank\_GLasso across all settings. In terms of computational efficiency, Rank\_GLasso and Rank\_Lasso are orders of magnitude faster than than the competing methods included in our comparison (LS\_GLasso and the three \texttt{hrqglas} estimators), all of which rely on cross-validation to select  regularization parameter, as the rank-based methods determine the regularization parameter directly from data without costly tuning. This speed advantage is consistent across all covariance structures, signal patterns and noise conditions.
			
			Regarding estimation accuracy, Rank\_GLasso consistently achieves competitive performance to the best-performing method across all four tables. Under covariance structure (C1), its $\ell_2$ error and model error are consistently close to the best results across settings, where the top performer may vary among Rank\_GLasso, finely tuned LS\_GLasso and QR\_GLasso ($\tau=0.5$). Under (C2), its accuracy is only slightly below the top-performer Huber\_GLasso, but consistently outperforms the remaining estimators.
			Notably, Rank\_GLasso exhibits strong robustness across all noise settings, including both Gaussian noise and heavy-tailed or contaminated distributions (e.g., Cauchy). Overall, Rank\_GLasso provides reliable and accurate estimates across diverse data-generating mechanisms.
			
			In support recovery, Rank\_GLasso strikes an effective balance, consistently achieving zero false negatives while keeping false positives at a moderate level. In particular, under the AR(1) (Table~\ref{tab:group2.1_new}--\ref{tab:group2.2_new}) covariance structure, it frequently attains zero false positives, demonstrating the benefit of group-level sparsity in eliminating irrelevant variables. Rank\_Lasso achieves even fewer false positives but at the expense of higher estimation error and occasional missed signals.}
		
		Overall, the results indicate that Rank\_GLasso achieves an excellent balance between computational efficiency and estimation accuracy. It avoids the costly parameter tuning required by LS\_GLasso and the three \texttt{hrqglas} estimators, resulting in computational time comparable to the fastest solver Rank\_Lasso, while delivering estimation accuracy close to the best-performing methods across a wide range of settings.

		\subsection{Scalability of PALM for Group Lasso Regularized Rank Regression}
		In this subsection, we investigate the scalability of our proposed PALM algorithm for solving the group Lasso regularized rank regression model \eqref{mod:rankgrouplasso0}. We follow the same data generation procedure as in Section~\ref{subsec:rank_grouptable}, where the design matrix $X$ is generated under structure (C1), and the true coefficient vector $\beta^*$ follows signal pattern (S1). Here, we slightly modify (S1) by setting $g = p/100$ instead of $p/20$, so as to allow for much larger dimensions while keeping the number of groups at a reasonable scale. Experiments are carried out for each of the noise distributions (E1)–(E6).
		
		We present two types of experiments to systematically evaluate the scalability of PALM. In the first set of experiments, we fix the sample size $n$ and increase the dimension $p$ from moderately large ($p=50000$) to extremely high-dimensional ($p=400000$). To illustrate the role of sample size, we consider both $n=500$ and $n=2000$. We conduct experiments under all six error distributions (E1)–(E6), and the results are summarized in Figure~\ref{fig:scaling_p}. Figure~\ref{fig:scaling_p}(a) shows that for $n=500$, computational time grows nearly linearly with dimension $p$ and stays below 70 seconds under all noise distributions. Even at the largest scale ($p=400000$), PALM remains robust and efficient, demonstrating its practical applicability in ultra-high dimensional settings. For $n=2000$ (Figure~\ref{fig:scaling_p}(b)), corresponding to a large-sample regime, the computational time remains consistently within 300 seconds as $p$ increases from $50000$ to $400000$, demonstrating strong scalability in high dimensions. This efficiency is achieved through a tailored computational strategy for solving the linear systems in Algorithm~\ref{alg:ssn}, which exploits second-order sparsity to significantly reduce the computational cost (see Appendix~\ref{sec:implementation}), so that the overall cost is governed primarily by $n$ and the problem sparsity $r$ rather than the ambient dimension $p$.
		\begin{figure}[H]
			\centering
			\includegraphics[width=0.4\textwidth]{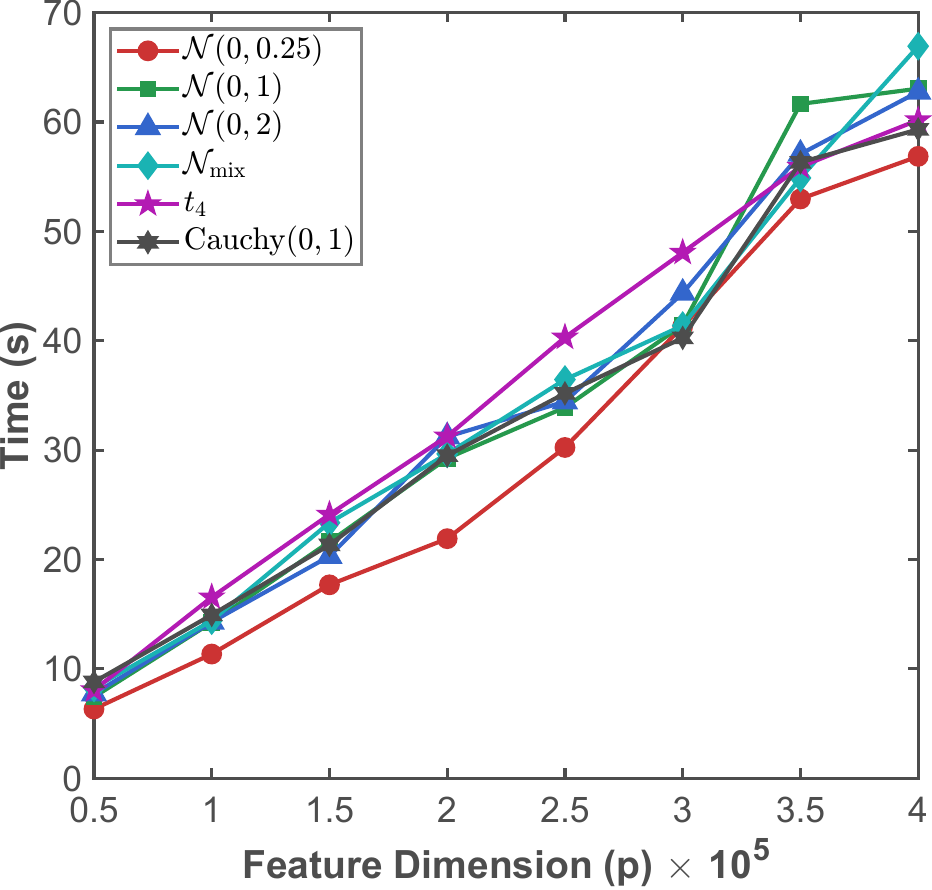}
			\quad 
			\includegraphics[width=0.4\textwidth]{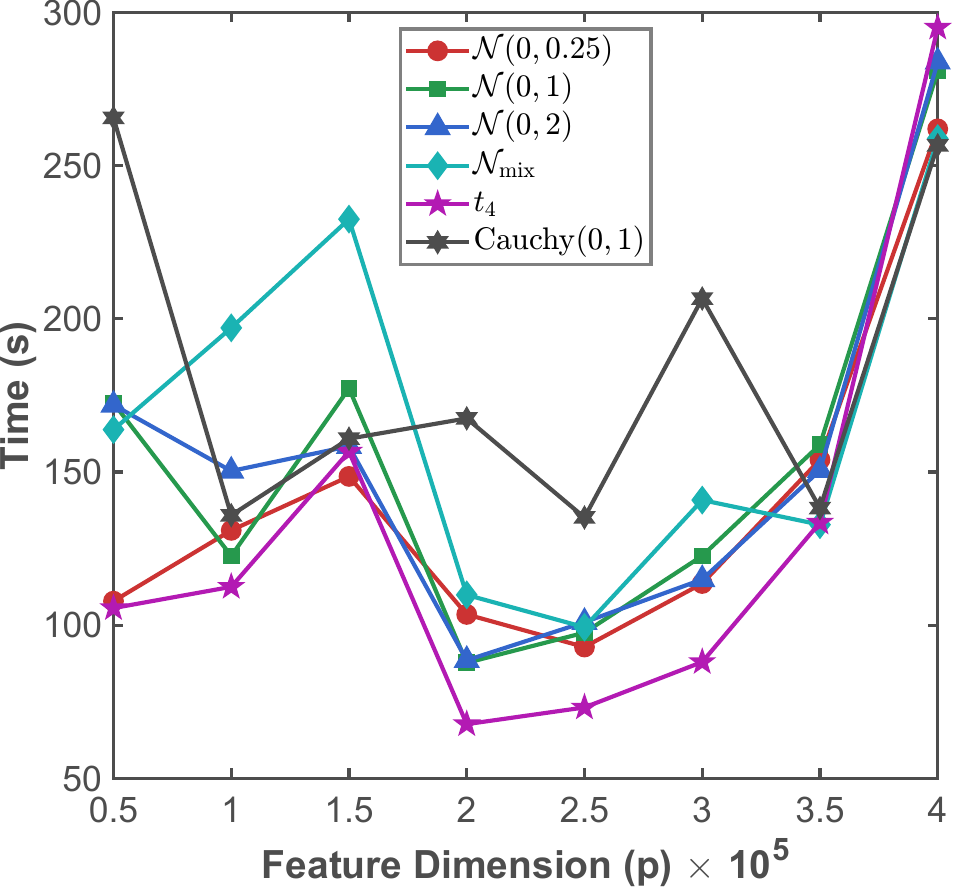}
			\caption{Computational time of PALM versus the dimension $p$ under different sample size scales. Left panel (a): fixed sample size $n=500$; right panel (b): fixed sample size $n=2000$.}
			\label{fig:scaling_p}
		\end{figure}
		
		In the second set of experiments, we fix the dimension $p$ and increase the sample size $n$ to examine the effect of sample growth. Two representative dimensions are considered, $p=8000$ and $p=100000$, with $n$ varying from 250 to 2000. Figure~\ref{fig:scaling_n} summarizes the outcomes obtained under six error distributions (E1)–(E6). It can be seen  that the computational time  increases approximately linearly with $n$ for both $p=8000$ and $p=100000$, as expected, while remaining low overall; even for the largest problem with $p=100000$ and $n=2000$, the runtime is under 90 seconds.

		\begin{figure}[H]
			\centering
			\includegraphics[width=0.4\textwidth]{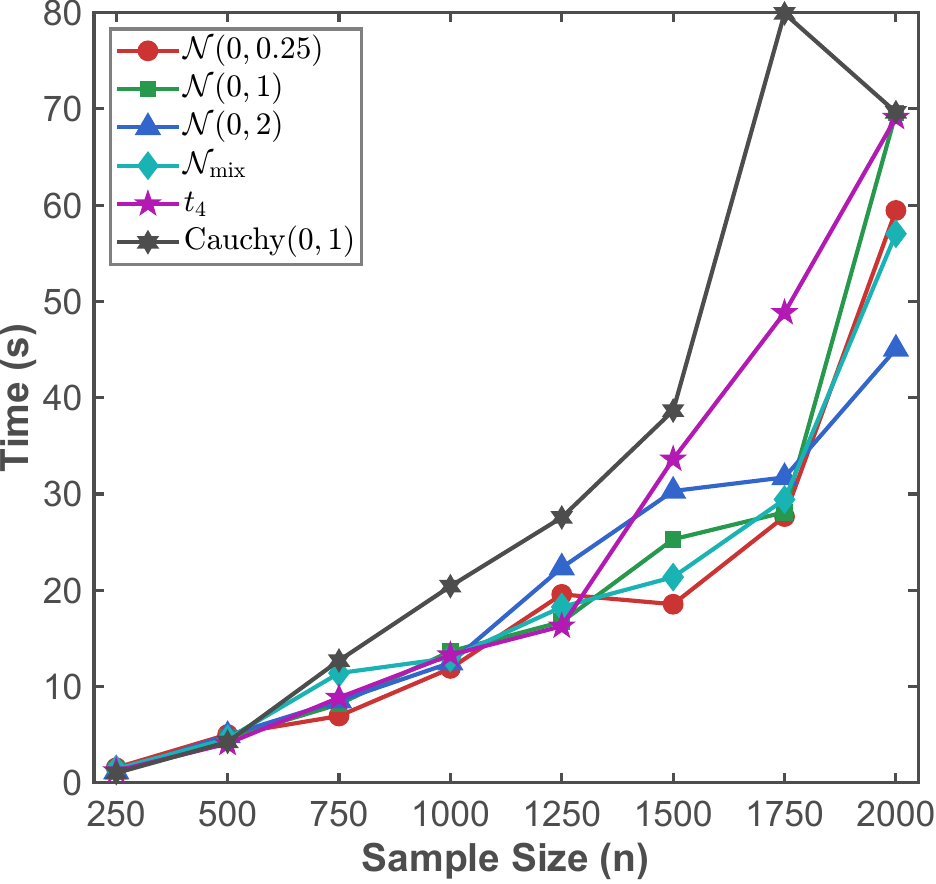}
			\quad
			\includegraphics[width=0.4\textwidth]{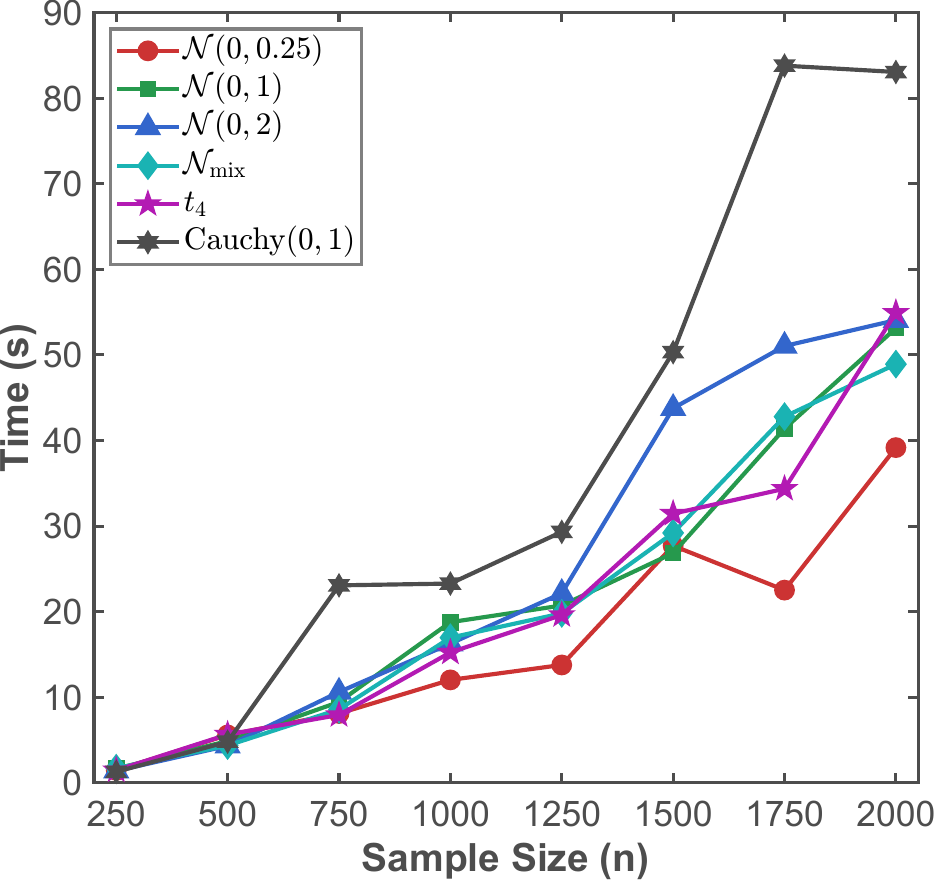}
			\caption{Computational time of PALM versus the sample size $n$ under different dimension scales. Left panel (a): fixed dimension $p=8000$; right panel (b): fixed  dimension $p=100000$.} 
			\label{fig:scaling_n}
		\end{figure}	
		
		Overall, the proposed PALM framework exhibits robust scalability and maintains low computational cost in high-dimensional settings, illustrating its practical utility for high-dimensional group-sparse regression.
		\subsection{\texorpdfstring{Benchmarking PALM on $\ell_1$-Regularized Rank Regression}{Benchmarking PALM on L1-Regularized Rank Regression}} 
		
		Since no existing algorithms are tailored to the group Lasso regularized rank regression problem, a direct comparison is not available. To facilitate a meaningful benchmark, we instead consider the $\ell_1$-regularized setting, obtained from \eqref{mod:rankgrouplasso0} by setting $\Psi(\beta)=\|\beta\|_1$. This allows us to compare with the state-of-the-art solver and also illustrates the applicability of our algorithm to general nonsmooth convex regularizers. Specifically, we compare our proposed PALM framework with the PPMM algorithm of \cite{tang2023proximal}, which has been shown to outperform standard solvers such as ADMM and Gurobi for $\ell_1$-regularized rank regression and thus serves as a strong baseline.

		\paragraph{Data Generation.}
		To evaluate the performance of our algorithm for $\ell_1$-regularized rank regression, we largely follow the data generation framework described in Section~\ref{subsec:rank_grouptable} on group Lasso regularized setting. The noise distributions remain the same (settings (E1)--(E6)). For the covariates, we introduce an additional design (C3) from \cite{tang2023proximal} to facilitate direct comparison with PPMM. For the coefficients, we add elementwise-sparse designs (S3)--(S4) to match the $\ell_1$ setting.
		
		The rows of the design matrix $X$ are sampled i.i.d. from $\mathcal{N}(0,\Sigma)$ with
		\begin{itemize}[leftmargin=3.5em, topsep=3pt, itemsep=1pt]
			\item[(C3)] Equi-correlation: $\Sigma_{ij} = 0.5$ for all $i \neq j$, and $\Sigma_{ii} = 1$.
		\end{itemize}
		
		The true coefficient vector $\beta^* \in \mathbb{R}^p$ exhibits elementwise sparsity with:
		\begin{itemize}[leftmargin=3.5em, topsep=3pt, itemsep=1pt]
			\item[(S3)] Sparse signal: 
			$$
			\beta^* = (\underbrace{\sqrt{3},\ldots,\sqrt{3}}_{k\ \text{entries}},0,\ldots,0)^\intercal, \text{~where~}
			k = \begin{cases}
				3 & \text{if } p \leq 8000,\\
				\lfloor 0.001p \rfloor & \text{if } p > 8000.
			\end{cases}
			$$
			
			\item[(S4)] Decaying signal: 
			$
			\beta^* = (\underbrace{2,\ldots,2}_{4m},\underbrace{1.75,\ldots,1.75}_{3m},\ldots,\underbrace{0.25,\ldots,0.25}_{3m},0,\ldots,0)^\intercal,
			$
			where the scaling factor
			$
			m = \begin{cases}
				1 & \text{if } p \leq 8000,\\
				\lceil p/2500 \rceil & \text{if } p > 8000.
			\end{cases}
			$
		\end{itemize}

		To access solution quality, we report both the KKT residual $\eta_{\text{kkt}}$ in \eqref{eq: def_etakkt} and the relative duality gap defined as
		\begin{equation*}
			\text{Relgap} := \frac{|\text{pobj} - \text{dobj}|}{1 + |\text{pobj}| + |\text{dobj}|},
		\end{equation*}
		where
		\[
		\text{pobj} := L(X\beta - y) + \lambda\Psi(\beta) \qquad \text{and} \qquad
		\text{dobj} := -\langle y,z\rangle -L^*(z) - \lambda \Psi^*(-\lambda^{-1}X^{\intercal}z)
		\]
		are the primal and dual objective values, respectively. 
		
		Since the implementation of PPMM is not publicly available, we reproduced the algorithm based on the description in \cite{tang2023proximal}. As shown in Table~\ref{tab:setting2_2000_8000}, which corresponds exactly to the setting in \cite{tang2023proximal}, our implementation of PPMM achieves running times that are significantly faster than those reported in the original paper. This improvement not only supports the validity of our reproduction but also provides a strong baseline for fair comparison.

		\begin{table}[H]
			\centering
			\caption{Comparison of PALM and PPMM for solving $\ell_1$-regularized rank regression with a $2000 \times 8000$ design matrix under covariance structure (C3), evaluated across different signal patterns (S3)--(S4) and error distributions (E1)--(E6).}
			\label{tab:setting2_2000_8000}
			\begin{tabular}{@{}l l c ccc ccc@{}}
				\toprule
				\multirow{2}{*}{\textbf{Signal}} & \multirow{2}{*}{\textbf{Noise}} & \multirow{2}{*}{\textbf{$\lambda$}} & \multicolumn{2}{c}{\textbf{Time}} & \multicolumn{2}{c}{\textbf{$\eta_{\mathrm{kkt}}$}} & \multicolumn{2}{c}{\textbf{Relgap}} \\
				\cmidrule(lr){4-5} \cmidrule(lr){6-7} \cmidrule(lr){8-9}
				& & & PALM & PPMM & PALM & PPMM & PALM & PPMM \\
				\midrule
				\multirow{6}{*}{(S3)}
				& $\mathcal{N}(0, 0.25)$ & 0.104 & 00:04 & 00:17 & 9.84e-7 & 9.83e-7 & 2.27e-8 & 5.02e-10 \\   
				& $\mathcal{N}(0, 1)$ & 0.102 & 00:04 & 00:13 & 9.90e-7 & 8.74e-7 & 5.94e-8 & 9.23e-10 \\   
				& $\mathcal{N}(0, 2)$ & 0.102 & 00:04 & 00:10 & 9.73e-7 & 9.88e-7 & 9.05e-8 & 9.21e-9 \\   
				& $\mathcal{N}_{\mathrm{mix}}$ & 0.100 & 00:04 & 00:14 & 9.75e-7 & 9.94e-7 & 3.44e-7 & 1.27e-9 \\   
				& $t_4$ & 0.101 & 00:04 & 00:15 & 9.54e-7 & 9.23e-7 & 6.34e-8 & 4.47e-9 \\   
				& $\mathrm{Cauchy}(0,1)$ & 0.103 & 00:11 & 00:13 & 9.37e-7 & 4.98e-7 & 1.27e-6 & 5.48e-8 \\   
				\midrule
				\multirow{6}{*}{(S4)}
				& $\mathcal{N}(0, 0.25)$ & 0.105 & 00:07 & 00:28 & 9.99e-7 & 9.20e-7 & 2.38e-9 & 3.17e-9 \\   
				& $\mathcal{N}(0, 1)$ & 0.103 & 00:07 & 00:25 & 9.78e-7 & 8.89e-7 & 8.66e-9 & 1.51e-9 \\   
				& $\mathcal{N}(0, 2)$ & 0.102 & 00:09 & 00:29 & 9.78e-7 & 9.10e-7 & 2.39e-8 & 2.57e-9 \\   
				& $\mathcal{N}_{\mathrm{mix}}$ & 0.100 & 00:12 & 00:31 & 9.44e-7 & 8.98e-7 & 3.56e-8 & 3.43e-9 \\   
				& $t_4$ & 0.100 & 00:09 & 00:26 & 9.52e-7 & 8.24e-7 & 7.34e-9 & 1.78e-8 \\   
				& $\mathrm{Cauchy}(0,1)$ & 0.103 & 00:21 & 00:25 & 7.28e-7 & 7.20e-7 & 2.96e-7 & 8.45e-8 \\   
				\bottomrule
			\end{tabular}
		\end{table}	
		
		\begin{table}[htbp]
			\centering
			\caption{Comparison of PALM and PPMM for solving $\ell_1$-regularized rank regression with design matrix $X$ under covariance structure (C2) and signal pattern (S3), evaluated across varying problem dimensions and error distributions (E1)--(E6).}
			\label{tab:bigsettingC2S3}
			\setlength{\tabcolsep}{4pt}
			\begin{tabular}{@{}l l c ccc ccc@{}}
				\toprule
				\multirow{2}{*}{\textbf{Noise}} & \multirow{2}{*}{\textbf{Dimension}} & \multirow{2}{*}{\textbf{$\lambda$}} & \multicolumn{2}{c}{\textbf{Time}} & \multicolumn{2}{c}{\textbf{$\eta_{\mathrm{kkt}}$}} & \multicolumn{2}{c}{\textbf{Relgap}} \\
				\cmidrule(lr){4-5} \cmidrule(lr){6-7} \cmidrule(lr){8-9}
				& & & PALM & PPMM & PALM & PPMM & PALM & PPMM \\
				\midrule
				\multirow{3}{*}{$\mathcal{N}(0, 0.25)$} 
				& (2000, 20000) & 0.116 & 00:06 & 00:50 & 6.54e-7 & 3.70e-7 & 1.17e-7 & 1.12e-11 \\   
				& (2000, 40000) & 0.119 & 00:13 & 01:22 & 3.72e-7 & 4.27e-7 & 4.48e-9 & 5.32e-10 \\   
				& (4000, 40000) & 0.085 & 00:26 & 04:02 & 4.48e-7 & 1.32e-7 & 1.87e-7 & 5.79e-8 \\   
				\addlinespace
				\multirow{3}{*}{$\mathcal{N}(0, 1)$} 
				& (2000, 20000) & 0.117 & 00:05 & 00:35 & 6.04e-7 & 5.96e-7 & 1.29e-7 & 1.35e-9 \\   
				& (2000, 40000) & 0.121 & 00:11 & 01:09 & 5.26e-7 & 4.38e-7 & 9.76e-9 & 6.23e-9 \\   
				& (4000, 40000) & 0.086 & 00:17 & 03:54 & 5.42e-7 & 1.59e-7 & 4.45e-8 & 3.63e-7 \\   
				\addlinespace
				\multirow{3}{*}{$\mathcal{N}(0, 2)$} 
				& (2000, 20000) & 0.117 & 00:06 & 00:40 & 8.76e-7 & 5.99e-7 & 3.32e-7 & 4.47e-7 \\   
				& (2000, 40000) & 0.122 & 00:13 & 01:12 & 5.87e-7 & 5.45e-7 & 4.28e-8 & 2.46e-8 \\   
				& (4000, 40000) & 0.085 & 00:17 & 02:58 & 2.64e-7 & 1.92e-7 & 5.96e-8 & 6.39e-7 \\   
				\addlinespace
				\multirow{3}{*}{$N_{\mathrm{mix}}$} 
				& (2000, 20000) & 0.116 & 00:04 & 00:32 & 8.96e-7 & 6.57e-7 & 2.11e-7 & 1.59e-8 \\   
				& (2000, 40000) & 0.121 & 00:08 & 01:05 & 9.19e-7 & 4.59e-7 & 2.21e-8 & 3.97e-9 \\   
				& (4000, 40000) & 0.085 & 00:17 & 03:05 & 9.12e-7 & 1.10e-7 & 6.26e-8 & 1.74e-7 \\   
				\addlinespace
				\multirow{3}{*}{$t_4$} 
				& (2000, 20000) & 0.116 & 00:05 & 00:34 & 7.52e-7 & 5.99e-7 & 2.65e-7 & 6.79e-8 \\   
				& (2000, 40000) & 0.121 & 00:07 & 00:54 & 8.34e-7 & 4.58e-7 & 2.72e-7 & 4.63e-9 \\   
				& (4000, 40000) & 0.085 & 00:16 & 03:16 & 3.08e-7 & 1.17e-7 & 4.05e-8 & 1.22e-7 \\   
				\addlinespace
				\multirow{3}{*}{$\mathrm{Cauchy}(0,1)$} 
				& (2000, 20000) & 0.118 & 00:06 & 00:36 & 8.49e-7 & 5.15e-7 & 1.82e-6 & 1.59e-7 \\   
				& (2000, 40000) & 0.121 & 00:12 & 01:07 & 4.75e-7 & 3.10e-7 & 6.29e-8 & 1.69e-8 \\   
				& (4000, 40000) & 0.085 & 00:20 & 03:40 & 2.70e-7 & 5.59e-7 & 9.08e-7 & 1.18e-5 \\   
				\bottomrule
			\end{tabular}
		\end{table}

		Table~\ref{tab:setting2_2000_8000} first reports results under the same experimental setting as in \cite{tang2023proximal}. We observe that PALM consistently runs faster than PPMM while attaining comparable accuracy. Across both signal patterns and all noise types, PALM requires only 4 to 21 seconds, while PPMM requires substantially longer runtimes (ranging from 10 to 31 seconds). This represents a significant speedup, with PALM often finishing in half the time of PPMM or even less. Under the more complex signal pattern (S4), which is less sparse than (S3), both algorithms require longer computational time, as expected. Despite the increased problem difficulty, PALM maintains a clear efficiency advantage. These results confirm the practical efficiency of PALM under the same settings used in prior work, while ensuring a fair and reproducible comparison.

		We then proceed to test more challenging large-scale problems under varying distribution of covariates, signal patterns, and error distributions, with results summarized in Tables~\ref{tab:bigsettingC2S3} and \ref{tab:bigsettingC3S4}. The results reveal that PALM maintains a consistent and robust efficiency advantage as the problem size grows. Under the (C2)-(S3) setting (Table~\ref{tab:bigsettingC2S3}), when the dimension expands from $(2000, 20000)$ to the more challenging $(4000, 40000)$, the runtime of PALM increases moderately from seconds to under half a minute for all noise types. In contrast, PPMM's runtime increases substantially from less than one minute to over 4 minutes, indicating significantly poorer scalability. This difference in computational efficiency becomes more pronounced under the denser signal pattern (S4) in Table~\ref{tab:bigsettingC3S4}. For the largest problem $(4000, 40000)$, PALM robustly obtains a solution within 6--8 minutes across all noise distributions, while PPMM's runtime varies widely from 35 to 60 minutes, highlighting its sensitivity to problem difficulty. Despite these substantial differences in computational time, both algorithms consistently achieve high-quality solutions with $\eta_{\rm{kkt}}$ below $10^{-6}$, underscoring PALM's superior scalability without compromising accuracy.
		
		\begin{table}[htbp]
			\centering
			\caption{Comparison of PALM and PPMM for solving $\ell_1$-regularized rank regression with design matrix $X$ under covariance structure (C3) and signal pattern (S4), evaluated across varying problem dimensions and error distributions (E1)--(E6).}
			\label{tab:bigsettingC3S4}
			\setlength{\tabcolsep}{4pt}
			\begin{tabular}{@{}l l c ccc ccc@{}}
				\toprule
				\multirow{2}{*}{\textbf{Noise}} & \multirow{2}{*}{\textbf{Dimension}} & \multirow{2}{*}{\textbf{$\lambda$}} & \multicolumn{2}{c}{\textbf{Time}} & \multicolumn{2}{c}{\textbf{$\eta_{\mathrm{kkt}}$}} & \multicolumn{2}{c}{\textbf{Relgap}} \\
				\cmidrule(lr){4-5} \cmidrule(lr){6-7} \cmidrule(lr){8-9}
				& & & PALM & PPMM & PALM & PPMM & PALM & PPMM \\
				\midrule
				\multirow{3}{*}{$\mathcal{N}(0, 0.25)$} 
				& (2000, 20000) & 0.106 & 01:11 & 02:57 & 9.13e-7 & 4.84e-7 & 1.10e-8 & 8.42e-10 \\   
				& (2000, 40000) & 0.108 & 01:52 & 05:21 & 7.51e-7 & 9.11e-7 & 2.35e-6 & 2.87e-6 \\   
				& (4000, 40000) & 0.075 & 07:07 & 36:27 & 4.81e-7 & 3.39e-8 & 2.28e-8 & 1.67e-8 \\   
				\addlinespace
				\multirow{3}{*}{$\mathcal{N}(0, 1)$} 
				& (2000, 20000) & 0.103 & 01:18 & 02:20 & 9.70e-7 & 4.98e-7 & 1.11e-9 & 6.48e-8 \\   
				& (2000, 40000) & 0.110 & 01:52 & 04:53 & 9.30e-7 & 9.53e-7 & 2.78e-6 & 2.79e-6 \\   
				& (4000, 40000) & 0.077 & 07:52 & 34:58 & 7.28e-7 & 6.02e-7 & 3.23e-8 & 2.10e-6 \\   
				\addlinespace
				\multirow{3}{*}{$\mathcal{N}(0, 2)$} 
				& (2000, 20000) & 0.105 & 01:18 & 02:11 & 5.36e-7 & 4.40e-7 & 1.65e-8 & 4.67e-7 \\   
				& (2000, 40000) & 0.111 & 01:53 & 04:38 & 9.44e-7 & 8.60e-7 & 2.54e-6 & 2.53e-6 \\   
				& (4000, 40000) & 0.078 & 07:43 & 59:14 & 3.40e-7 & 2.13e-7 & 5.73e-9 & 6.87e-7 \\   
				\addlinespace
				\multirow{3}{*}{$N_{\mathrm{mix}}$} 
				& (2000, 20000) & 0.105 & 01:06 & 02:11 & 7.59e-7 & 4.99e-7 & 5.01e-8 & 1.61e-8 \\   
				& (2000, 40000) & 0.108 & 01:23 & 04:12 & 9.55e-7 & 8.65e-7 & 2.33e-6 & 2.40e-6 \\   
				& (4000, 40000) & 0.077 & 06:39 & 40:04 & 5.41e-7 & 2.47e-7 & 1.20e-7 & 2.07e-6 \\   
				\addlinespace
				\multirow{3}{*}{$t_4$} 
				& (2000, 20000) & 0.106 & 01:01 & 01:57 & 7.54e-7 & 5.49e-7 & 2.83e-8 & 2.33e-7 \\   
				& (2000, 40000) & 0.108 & 01:24 & 04:36 & 8.23e-7 & 8.52e-7 & 1.98e-6 & 2.50e-6 \\   
				& (4000, 40000) & 0.078 & 06:55 & 37:47 & 3.64e-7 & 5.55e-8 & 6.56e-9 & 2.54e-7 \\   
				\addlinespace
				\multirow{3}{*}{$\mathrm{Cauchy}(0,1)$} 
				& (2000, 20000) & 0.108 & 01:05 & 02:00 & 5.47e-7 & 6.91e-7 & 1.36e-7 & 2.18e-5 \\   
				& (2000, 40000) & 0.107 & 01:25 & 03:37 & 3.15e-7 & 8.14e-7 & 8.63e-7 & 9.56e-7 \\   
				& (4000, 40000) & 0.076 & 06:51 & 47:55 & 5.12e-7 & 1.40e-7 & 2.02e-7 & 1.06e-5 \\   
				\bottomrule
			\end{tabular}
		\end{table}

		Finally, we evaluate both algorithms on real-world datasets from the KEEL repository~\citep{AlcalFdez2011KEELDS}. Following \cite{tang2023proximal}, the features of each dataset are expanded into a higher-dimensional space using polynomial basis functions 
		for scalability evaluation. The order of expansion is encoded in the dataset name; for example, ``baseball5'' denotes a fifth-order expansion. The comparative results are presented in Table~\ref{tab:real_results}. As shown, both algorithms successfully solve all instances to high accuracy.
		PALM demonstrates significantly superior computational efficiency across all datasets. A notable example is the ``compactiv51tra4'' dataset, where PALM converges within one minute, while PPMM requires over 30 minutes, representing a speedup of more than 30 times. Similarly, on the ``puma32h51tst3'' dataset, PALM solves the problem in just 4 seconds, while PPMM takes nearly 5 minutes. Even on smaller datasets, PALM maintains a consistent 2- to 7-fold speed advantage. This highlights the practical strength and scalability of PALM algorithm when applied to real problems.
		
		\begin{table}[htbp]
			\caption{Comparison of PALM and PPMM for solving $\ell_1$-regularized rank regression on real datasets.}
			\label{tab:real_results}
			\setlength{\tabcolsep}{4pt} 
			\begin{tabular}{@{}l l c ccc ccc@{}}
				\toprule
				\multirow{2}{*}{\textbf{Data}} & \multirow{2}{*}{\textbf{Dimension}} & \multirow{2}{*}{\textbf{$\lambda$}} & \multicolumn{2}{c}{\textbf{Time}} & \multicolumn{2}{c}{\textbf{$\eta_{\mathrm{kkt}}$}} & \multicolumn{2}{c}{\textbf{Relgap}} \\
				\cmidrule(lr){4-5} \cmidrule(lr){6-7} \cmidrule(lr){8-9}
				& & & PALM & PPMM & PALM & PPMM & PALM & PPMM \\
				\midrule
				baseball5 & (337, 17067) & 2.1e-1 & 00:04 & 00:14 & 6.12e-7 & 3.85e-7 & 7.84e-7 & 3.32e-5 \\
				compactiv51tra4 & (6553, 12649) & 5.4e-2 & 00:54 & 30:34 & 3.69e-7 & 7.51e-7 & 8.96e-7 & 7.22e-1 \\
				concrete7 & (1030, 6434) & 1.2e-1 & 00:03 & 00:09 & 9.76e-7 & 7.74e-7 & 4.67e-7 & 2.10e-8 \\
				dee10 & (365, 8007) & 1.9e-1 & 00:02 & 00:05 & 9.26e-7 & 9.77e-7 & 3.05e-8 & 2.22e-7 \\
				friedman10 & (1200, 3002) & 1.1e-1 & 00:01 & 00:07 & 8.78e-7 & 1.00e-6 & 3.27e-7 & 9.89e-8 \\
				puma32h51tst3 & (1639, 6544) & 1.2e-1 & 00:04 & 04:43 & 5.99e-7 & 9.14e-7 & 4.28e-9 & 3.05e-11 \\
				wizmir7 & (1461, 11439) & 9.2e-2 & 00:03 & 00:27 & 9.78e-7 & 5.09e-7 & 2.08e-8 & 1.14e-7 \\
				\bottomrule
			\end{tabular}
		\end{table}
		
		\section{Conclusion}\label{sec:5}
		This paper proposed a group Lasso regularized Wilcoxon rank regression estimator for robust high-dimensional regression with grouped coefficients, together with a simulation-based determination rule for the regularization parameter. We established a finite sample error bound for the estimator, with a theoretical framework that extends beyond group regularization to general norm-based penalties under certain structural conditions. To solve the resulting optimization problem, we developed a proximal augmented Lagrangian method (PALM). Since existing convergence theory is not directly applicable, we establish its superlinear convergence via a novel analysis by proving the metric subregularity of the underlying KKT mapping. The PALM subproblems are then efficiently solved by a semismooth Newton method.
		Numerical experiments demonstrate that the proposed estimator attains superior estimation accuracy and reliable group selection across diverse data-generating settings, while the PALM algorithm exhibits strong scalability and consistently outperforms the state-of-the-art alternative on synthetic and real data.

		
	\section*{Declarations}
	\noindent\textbf{Funding.} The work of Yunhai  Xiao is supported by the National Natural Science Foundation of China (Grant No. 12471307 and 12271217).
	
	\noindent\textbf{Availability of data and materials.} Only public datasets, in \href{https://sci2s.ugr.es/keel/datasets.php}{KEEL-dataset repository}, are used.
	
	\noindent\textbf{Competing interests.} The authors declare that they have no conflict of interest.
	
	\bibliographystyle{spbasic}
	\bibliography{ref2} 
	
	\appendix 
	\section{Discussion of Assumption~\ref{assump:error}}
	\label{appendix:assumption_discussion}
	
	We further discuss the roles of Assumptions~\ref{assumption_X} and
	\ref{assumption_error}, and provide sufficient conditions under which
	Assumption~\ref{assumption_error} holds.
	
	\begin{enumerate}[label=(a\arabic*)]
		\item Assumption~\ref{assumption_X} imposes a RE condition 
		over a structured cone based on the decomposition of $\Psi(\cdot)$ across the active and inactive groups.
		The choice of parameter in \eqref{eq: opt_lambda}
		guarantees that the estimation error $\hat{\gamma}(\lambda^*) := \hat{\beta}(\lambda^*) - \beta^*$ lies in this cone with probability $1 - \alpha_0$ (see Lemma~\ref{lemma:Psi<gamma with prob}).
		
		\item Assumption~\ref{assumption_error} requires the pairwise error difference density $f^*(\cdot)$ to be bounded away from zero near the origin. Let
		$
		K :=  \left(34 b_1^2 c_0 \bar{c}^2/b_2\right) c_p^2 c_{|\Omega|}^2 \mathcal{E}(n, p),
		$
		then Assumption~\ref{assumption_error} is equivalent to requiring the existence of some $M > 0$ such that
		\begin{equation}
			M f^*(\zeta) \geq K \quad \text{for all } |\zeta| \leq M.\label{eq: Mfstar}
		\end{equation}
		This condition holds for many common error distributions. For example, if the error is unimodal, such as Gaussian, Cauchy, Student's $t$, and $\chi^2(k)$ with $k \geq 3$, then $f^*(\cdot)$ is symmetric about zero, non-decreasing on $(-\infty,0]$ and non-increasing on $[0,\infty)$, see \citet[Theorem 2.2]{PURKAYASTHA199897}. Consequently, the infimum of $\zeta \mapsto M f^*(\zeta)$ over $[-M, M]$ is attained at $|\zeta| = M$, so condition~\eqref{eq: Mfstar} simplifies to
		$
		M f^*(M) \geq K$,
		in which case such an $M$ exists as long as $K$ is sufficiently small.
	\end{enumerate}

	\section{Proof of the Finite-Sample Error Bound (Theorem \ref{th:consistency})}
	\label{sec: proof_thm1}
	This section presents the proof of Theorem~\ref{th:consistency}, which establishes a finite-sample error bound for the proposed estimator~\eqref{mod:rankgrouplasso0} under the simulation-based regularization parameter~\eqref{eq: opt_lambda}. The argument relies on some key ingredients, each stated as a supporting lemma. 
	The first lemma ensures that the estimation error lies within a structured cone with high probability as follows.
	\begin{lemma}\label{lemma:Psi<gamma with prob}
		For the group Lasso regularized rank regression estimator $\hat{\beta}(\lambda^*)$ in \eqref{mod:rankgrouplasso0} with the regularization parameter $\lambda^*$ specified in~\eqref{eq: opt_lambda}, the estimation error $\hat{\gamma}(\lambda^*) := \hat{\beta}(\lambda^*) - \beta^*$ satisfies
		\[
		\operatorname{pr}\left(\Psi_{\bar\Omega}(\hat{\gamma}(\lambda^*)) \leq \frac{c_0 + 1}{c_0 - 1} \Psi_{\Omega}(\hat{\gamma}(\lambda^*))\right) = 1 - \alpha_0.
		\]
	\end{lemma}
	\begin{proof}
		Denote $\hat{\gamma}=\hat{\gamma}(\lambda^*)$. By the equivalence between the solutions of \eqref{mod:rankgrouplasso0} and \eqref{eq: gamma_prob} under the transformation $\beta = \gamma + \beta^*$, we have $\hat{\gamma} = \mathop{\arg\min}\limits_{\gamma\in \mathbb{R}^p}F_{\lambda^*}(\gamma)$, 
		where $F_{\lambda^*}(\cdot)$ is defined in \eqref{eq: gamma_prob}. Then $F_{\lambda^*}(\hat\gamma)\leq F_{\lambda^*}(0)$, which implies
		\begin{equation}
			\begin{aligned}
				L_0(\hat{\gamma})-L_0(0)
				&\leq \lambda^*\left(\Psi(\beta^*) -\Psi(\hat{\beta}(\lambda^*))\right) \\
				&\leq \lambda^*\Psi_{\Omega}(\beta^*)-\lambda^*\Psi_{\Omega}(\hat{\beta}(\lambda^*))- \lambda^*\Psi_{\bar\Omega}(\hat{\beta}(\lambda^*))
				\leq \lambda^*\Psi_{\Omega}(\hat{\gamma})- \lambda^*\Psi_{\bar\Omega}(\hat{\gamma}),
			\end{aligned}
			\label{ineq:Lgamma>L0}
		\end{equation}
		where the second inequality follows  from the decomposition \eqref{eq:decomposition}, and the last equality follows from the triangle inequality and the fact that $\Psi_{\bar{\Omega}}(\beta^*) = 0$.
		On the other hand, by the convexity of $L_0({\cdot})$, we have
		\begin{equation*}
			L_0(\hat{\gamma})-L_0(\mathbf 0)\geq \langle \hat{\gamma}, S_n\rangle 
			\geq -\Psi(\hat{\gamma})\cdot\Psi^d(S_n),
		\end{equation*}
		where $S_n$ in \eqref{eq: def_sn} is the subgradient of $L_0(\gamma)$ at $\mathbf 0$. The choice of $\lambda^*$ in \eqref{eq: opt_lambda} implies 
		$\text{pr}\left(\lambda^*\geq c_0\Psi^d(S_n)\right)= 1-\alpha_0$,    which means 
		\begin{equation}\label{ineq:gamma_Sn}
			L_0(\hat{\gamma})-L_0(\mathbf 0) \geq -\frac{\lambda^*}{c_0}\Psi(\hat{\gamma})
			\geq -\frac{\lambda^*}{c_0}\left( \Psi_{\Omega}(\hat{\gamma}) + \Psi_{\bar\Omega}(\hat{\gamma}) \right)
		\end{equation}
		holds with probability $1-\alpha_0$. Combining \eqref{ineq:Lgamma>L0} and \eqref{ineq:gamma_Sn}, we can see that
		$$
		\text{pr}\left(\Psi_{\bar\Omega}(\hat{\gamma})\leq \frac{c_0+1}{c_0-1}\Psi_{\Omega}(\hat{\gamma})\right) = 1-\alpha_0,
		$$ 
		which completes the proof.
	\end{proof}
	
	The second lemma provides a high-probability tail bound for the dual norm $\Psi^d(\cdot)$ at the empirical gradient $S_n$ in \eqref{eq: def_sn}, which helps justify the proposed regularization parameter.
	\begin{lemma}\label{lemma:lambda} 
		Suppose Assumption~\ref{assumption_X} holds. Then for any $t > 0$,
		\[
		\operatorname{pr}\left(\Psi^d(S_n) \geq t\right) \leq 2p \exp\left(-\frac{nt^2}{32b_1^2 c_p^2}\right).
		\]
	\end{lemma}
	\begin{proof}
		By writing $S_n = (s_1,\ldots,s_p)^{\intercal}$, we have
		$$
		s_k = \frac{1}{n(n-1)}\sum_{i=1}^n\sum_{j\ne i}(X_{jk}-X_{ik})\ {\rm sign}(\epsilon_i-\epsilon_j), \quad k\in[p].
		$$ 
		For any $t>0$. The concentration inequality for U-statistics \citep[Lemma A]{Wang2020ATR} together with Assumption~\ref{assumption_X} implies that
		$
		\text{pr}\left(|s_k|\geq {t}/{c_p} \right) \leq 2\exp\left(-{nt^2}/{(32b_1^2 c_p^2)}\right)
		$.
		Then we have that
		\begin{equation*}
			\text{pr}\left(\Psi^d(S_n)\geq  t\right)\leq \text{pr}\left(\|S_n\|_{\infty}\geq  \frac{t}{c_p}\right)\leq \sum_{k=1}^{p}\text{pr}\left(|s_k|\geq \frac{t}{c_p} \right) \leq 2p\exp\left(-\frac{nt^2}{32b_1^2 c_p^2}\right),
		\end{equation*}
		where the first and second inequalities follow from \eqref{ineq:cp} and the union bound, respectively.
		~\hfill\halmos
	\end{proof}
	
	We next provide a technical lemma that is also used in the proof of Theorem~\ref{th:consistency}.
	\begin{lemma}\label{lemma:conv and prob}
		Given a convex function $g:\mathbb{R}^p \rightarrow \mathbb{R}$, a cone $\Gamma$ in $\mathbb{R}^p$, and any $\delta > 0$. Suppose
		$
		\inf\limits_{y\in\Gamma}\left\{ g(y)\ \vert \ \|y\|_2 = \delta \right\} > g(\mathbf 0)$, then we have 
		$
		\inf\limits_{y\in\Gamma}\left\{ g(y)\ \vert \ \|y\|_2 > \delta \right\} > g(\mathbf 0)$.
		
	\end{lemma}
	\begin{proof}
		For any $y\in \Gamma$ satisfying $\|y\|_2 > \delta$. Since $\Gamma$ is a cone,  we have $({\delta}/{\|y\|_2})y\in \Gamma$.
		By the convexity of $g(\cdot)$, we have 
		$
		\left(1 - {\delta}/{\|y\|_2}\right)g(\mathbf 0) + ({\delta}/{\|y\|_2}) g(y) \geq g\left(({\delta}/{\|y\|_2}) y\right)
		$. 
		Rearranging this inequality, we obtain
		\[
		\frac{\delta}{\|y\|_2}\left( g(y) - g(\mathbf 0)\right)\geq g\left(\frac{\delta}{\|y\|_2}y\right) - g(\mathbf 0)\geq \inf\limits_{x\in\Gamma}\left\{ g(x)\ \vert \ \|x\|_2=\delta \right\} - g(\mathbf 0) > 0,
		\]
		which indicates $g(y) - g(\mathbf 0) > 0$. This further implies that
		\[
		g(y) - g(\mathbf 0) \geq \frac{\delta}{\|y\|_2}\left( g(y) - g(\mathbf 0) \right) \geq \inf\limits_{x\in\Gamma}\left\{ g(x)\ \vert \ \|x\|_2=\delta \right\} - g(\mathbf 0) > 0.
		\]
		
		By taking the infimum of $y$ over all $y\in\Gamma$ with $\|y\|_2 > \delta$, we have 
		\[
		\inf\limits_{x\in\Gamma}\left\{ g(x)\ \vert \ \|x\|_2>\delta \right\} - g(\mathbf 0)\geq \inf\limits_{x\in\Gamma}\left\{ g(x)\ \vert \ \|x\|_2=\delta \right\} - g(\mathbf 0) > 0,
		\]
		and the conclusion follows.    \hfill
		~\halmos
	\end{proof}
	
	
	We then give the proof of the main theorem as follows.
	
	
	
	\begin{proof}{\textit{of Theorem \ref{th:consistency}}}  ~Define the sets 
		\begin{equation*}
			\Gamma = \{\gamma \in \mathbb{R}^p: \Psi_{\bar\Omega}( \gamma)\leq (\bar{c}-1)\Psi_{\Omega}( \gamma)\} \text{~and~} \Gamma^* = \{\gamma \in \Gamma: \|\gamma\|_2 = \Delta c_{|\Omega|} c_p{\cal E}(n,p)\},
		\end{equation*}
		where $\Delta  =  \frac{17b_1 c_0 \bar{c}}{b_2b_3} $.
		Define the estimation error $\hat{\gamma}(\lambda^*) := \hat{\beta}(\lambda^*) - \beta^*$. Then we have that
		\begin{equation*}
			\begin{aligned}
				&\mathrm{pr}\left(\|\hat{\beta}(\lambda^*)-\beta^*\|_2 \leq  \Delta c_{|\Omega|} c_p{\cal E}(n,p)\right)  = 1- \mathrm{pr}\left(\|\hat{\gamma}(\lambda^*)\|_2 > \Delta c_{|\Omega|} c_p {\cal E}(n,p)\right) \\
				&\qquad \geq  1- \mathrm{pr}\left(\{\|\hat{\gamma}(\lambda^*)\|_2 > \Delta c_{|\Omega|} c_p{\cal E}(n,p)\}\cap \{\hat{\gamma}(\lambda^*) \in \Gamma\}\right) -  \mathrm{pr}(\hat{\gamma} (\lambda^*)\notin \Gamma)\\
				&\qquad \geq  1- \mathrm{pr}\left(\inf_{\gamma\in \Gamma} \left\{
				F_{\lambda^*}(\gamma)\ \middle \vert \ \|\gamma\|_2 > \Delta c_{|\Omega|} c_p{\cal E}(n,p) 
				\right\} \leq F_{\lambda^*}(0)\right) - \mathrm{pr}(\hat{\gamma}(\lambda^*) \notin \Gamma),
			\end{aligned}
		\end{equation*}
		where $F_{\lambda}(\cdot)$ is defined in \eqref{eq: gamma_prob}, and the last inequality follows from the fact that when $\|\hat{\gamma}(\lambda^*)\|_2 > \Delta c_{|\Omega|} c_p{\cal E}(n,p)$ and $\hat{\gamma}(\lambda^*) \in \Gamma$, we have
		$$
		\inf_{\gamma\in \Gamma} \left\{
		F_{\lambda^*}(\gamma)\ \middle \vert \ \|\gamma\|_2 > \Delta c_{|\Omega|} c_p{\cal E}(n,p) 
		\right\} = F_{\lambda^*}(\hat\gamma(\lambda^*)) = \inf_{\gamma\in \mathbb{R}^p}F_{\lambda^*}(\gamma)\leq F_{\lambda^*}(\mathbf 0).
		$$
		Moreover, we can see that
		\begin{equation}\label{ineq:probilibity}
			\begin{aligned}
				& \text{pr}\left(\|\hat{\beta}(\lambda^*)-\beta^*\|_2 \leq  \Delta c_{|\Omega|} c_p{\cal E}(n,p)\right) \\
				&\qquad \geq \text{pr}\left(\inf_{\gamma\in \Gamma} \left\{
				F_{\lambda^*}(\gamma)\ \middle \vert \ \|\gamma\|_2 > \Delta c_{|\Omega|} c_p{\cal E}(n,p)
				\right\} > F_{\lambda^*}(0)\right) - \alpha_0  \\
				& \qquad \geq  \text{pr}\left(\inf_{\gamma \in \Gamma^*} F_{\lambda^*}(\gamma) > F_{\lambda^*}(0)\right) - \alpha_0,  
		\end{aligned} \end{equation}
		where the first inequality comes from Lemma \ref{lemma:Psi<gamma with prob}, and the second inequality follows from Lemma \ref{lemma:conv and prob}. 
		Denote $E\left( L_0(\cdot)\right)$ as $\bar{L}_0(\cdot)$, where $L_0(\cdot)$ is defined in \eqref{eq: def_L0}. By the definition of $F_{\lambda}({\cdot})$ in \eqref{eq: gamma_prob} and the triangle inequality,  we have  
		\begin{equation}
			\begin{aligned}
				&\inf_{\gamma\in\Gamma^*}\{F_{\lambda^*}(\gamma)-F_{\lambda^*}(\mathbf 0)\}  =  \inf_{\gamma\in\Gamma^*}\{L_0(\gamma)+ \lambda^*\Psi(\gamma+\beta^*)- L_0(\mathbf 0) - \lambda^*\Psi(\beta^*)\}\\
				& \qquad = \inf_{\gamma\in\Gamma^*}\{\bar{L}_0(\gamma)-\bar{L}_0(\mathbf 0) + L_0(\gamma)- L_0(\mathbf 0) -(\bar{L}_0(\gamma)-\bar{L}_0(\mathbf 0)) + \lambda^*(\Psi(\gamma + \beta^*)- \Psi(\beta^*))\} \\
				&\qquad \geq\underbrace{\inf_{\gamma\in\Gamma^*}\{\bar{L}_0(\gamma)-\bar{L}_0(\mathbf 0)\}}_{\text{part I}} -  \underbrace{\eta(\Gamma^*)}_{\text{part II}} -  \underbrace{\lambda^*\sup_{\gamma\in\Gamma^*}|\Psi(\gamma+\beta^*)- \Psi(\beta^*)|}_{\text{part III}},\label{ineq:F-L0 with l2norm}
			\end{aligned}
		\end{equation}
		where
		$\eta(\Gamma^*): =\sup\limits_{\gamma\in\Gamma^*}\left|L_0(\gamma)- L_0(\mathbf 0) -\left(\bar{L}_0(\gamma)-\bar{L}_0(\mathbf 0)\right)\right|
		$.
		
		Next, we analyze the bounds for three parts one by one. 
		To facilitate the subsequent analysis, for any $\gamma\in \Gamma$, we define an auxiliary function $h_{\gamma}:\mathbb{R}^2 \rightarrow \mathbb{R}$ as:
		\begin{equation*}
			h_{\gamma}(\epsilon_i,\epsilon_j) := \left|(\epsilon_i-\epsilon_j) - {(X_i-X_j)}^{\intercal}\gamma\right| - |\epsilon_i - \epsilon_j|.
		\end{equation*} 
		Then we can see that, for each $i,j\in [p]$,
		\begin{equation}
			|h_{\gamma}(\epsilon_i,\epsilon_j)|\leq |{(X_i-X_j)}^{\intercal}\gamma| \leq |\Psi^d(X_i-X_j)
			\Psi(\gamma)|\leq c_p\|X_i-X_j\|_{\infty}\Psi(\gamma), \label{eq: bound_h}
		\end{equation}
		where the last inequality follows from \eqref{ineq:co}. For any $\gamma \in \Gamma$, due to decomposition \eqref{eq:decomposition} and inequality \eqref{ineq:cp}, we have
		\begin{equation}\label{ineq:Psi<gamma2}
			\Psi(\gamma) = \Psi_{\Omega}(\gamma) +\Psi_{\bar\Omega}(\gamma) 
			\leq \bar{c}\Psi_{\Omega}(\gamma) 
			\leq \bar{c}c_{|\Omega|}\|\gamma\|_2.
		\end{equation}
		This, together with \eqref{eq: bound_h} and Assumption~\ref{assump:error}, further implies
		\begin{equation}\label{ineq:|h|}
			|h_{\gamma}(\epsilon_i,\epsilon_j)|\leq 2b_1\bar{c}c_{|\Omega|}c_p\|\gamma\|_2.
		\end{equation}
		
		\textbf{Bound for $ \inf\limits_{\gamma\in\Gamma^*}\{\bar{L}_0(\gamma)-\bar{L}_0(\mathbf 0)\}$ in part I.} By the definition of $\bar{L}_0(\cdot)$, for any $\gamma\in \Gamma^*$, we have
		\begin{equation*}
			\bar{L}_0(\gamma) - \bar{L}_0(\mathbf 0) = E\left( \frac{1}{n(n-1)}\sum\limits_{i=1}^n\sum\limits_{j \ne i}h_{\gamma}(\epsilon_i,\epsilon_j)\right).
		\end{equation*}
		According to Knight's identity  \citep{Koenker_2005}, it can be seen that
		$$
		h_{\gamma}(\epsilon_i,\epsilon_j) = - {(X_i-X_j)}^{\intercal}\gamma\left[1 -2\mathbb{I}\{\zeta_{ij} < 0\}\right] + 2\int_0^{{(X_i-X_j)}^{\intercal}\gamma}[\mathbb{I}\{\zeta_{ij} \leq s\} -\mathbb{I}\{\zeta_{ij} \leq 0\}] \mathrm{d}s,
		$$
		where $\zeta_{ij} = \epsilon_i-\epsilon_j$, and $\mathbb{I}\{\cdot\}$ denotes the indicator function, which equals $1$ if the condition inside the braces holds, and $0$ otherwise. 
		Since $\epsilon_i$'s are i.i.d., the random variable $\zeta_{ij}$ follows a symmetric distribution with cumulative distribution function $F^*(\cdot)$, implying that $E\left( \mathbb{I}\{\zeta_{ij} \leq 0\}\right)=0.5$. Thus, we can see that for any $\gamma \in \Gamma^*$,
		\begin{equation*}
			\begin{aligned}
				& \bar{L}_0(\gamma) - \bar{L}_0(\mathbf 0) =  \frac{2}{n(n-1)}\sum\limits_{i=1}^n\sum\limits_{j \ne i}\int_0^{{(X_i-X_j)}^{\intercal}\gamma}[F^*(s) -F^*( 0)]{\rm d}s \\ 
				& \qquad = \frac{2}{n(n-1)}\sum\limits_{i=1}^n\sum\limits_{j \ne i}\int_0^{{(X_i-X_j)}^{\intercal}\gamma}[F^*(s) -F^*( 0)] \mathbb{I}\{{(X_i-X_j)}^{\intercal}\gamma>0\}{\rm d}s \\
				& \qquad \quad + \frac{2}{n(n-1)}\sum\limits_{i=1}^n\sum\limits_{j \ne i}\int_0^{{(X_i-X_j)}^{\intercal}\gamma}[F^*(s) -F^*(0)]\mathbb{I} \{ {(X_i-X_j)}^{\intercal}\gamma \leq 0 \}{\rm d}s.
			\end{aligned}
		\end{equation*}
		By the Mean Value Theorem, for $\gamma \in \Gamma^*$ with ${(X_i-X_j)}^{\intercal}\gamma>0$, there exists  $\zeta_{ij}\in[0,{(X_i-X_j)}^{\intercal}\gamma]$ such that  
		\begin{equation*}
			\int_0^{{(X_i-X_j)}^{\intercal}\gamma}[F^*(s)- F^*(0)]{\rm d}s =\int_0^{{(X_i-X_j)}^{\intercal}\gamma}s f^*(\zeta_{ij}){\rm d}s.
		\end{equation*}
		From \eqref{eq: bound_h} and \eqref{ineq:|h|}, we know that 
		$$
		|\zeta_{ij}| \leq |{(X_i-X_j)}^{\intercal}\gamma|\leq 2b_1\bar{c}c_{|\Omega|}c_p\|\gamma\|_2 =  \frac{34 b_1^2 c_0 \bar{c}^2}{b_2 b_3}   c_p^2 c_{|\Omega|}^2 \mathcal{E}(n, p).
		$$ 
		Thus, from Assumption~\ref{assumption_error} , we know that $f^*(\zeta_{ij}) \geq b_3$, which yields
		$$\int_0^{{(X_i-X_j)}^{\intercal}\gamma}[F^*(s)- F^*(0)]{\rm d}s \geq \frac{b_3}{2} [(X_i - X_j)^{\intercal} \gamma]^2.
		$$
		A similar bound holds when $\gamma \in \Gamma^*$ with $(X_i - X_j)^{\intercal} \gamma \leq 0$. Therefore, for any $\gamma \in \Gamma^*$, we have
		\begin{equation*}
			\begin{aligned}
				& \bar{L}_0(\gamma) - \bar{L}_0(\mathbf 0) \\ 
				& \quad \geq   \frac{b_3}{n(n-1)} \sum\limits_{i=1}^n\sum\limits_{j \ne i} [{(X_i-X_j)}^{\intercal}\gamma]^2
				= \frac{b_3}{n(n-1)}\sum\limits_{i=1}^n\sum_{j \neq i} \left[(X_i^\intercal \gamma)^2 + (X_j^\intercal \gamma)^2 - 2\gamma^{\intercal}X_iX_j^\intercal \gamma  \right] \\
				&\quad = \frac{b_3}{n(n-1)}\left(2(n-1) \sum_{i=1}^n (X_i^\intercal \gamma)^2 - \sum_{i=1}^n  \left( \sum_{j \neq i} 2\gamma^\intercal  X_i  X_j ^\intercal \gamma +  2\gamma^\intercal  X_i  X_i ^\intercal \gamma - 2 (X_i ^\intercal \gamma)^2
				\right)  \right)\\
				&\quad = \frac{b_3}{n(n-1)}\left(2n \sum_{i=1}^n (X_i^\intercal \gamma)^2 - 2\gamma^\intercal \left( \sum_{i=1}^n X_i \right)\left( \sum_{j =1}^n X_j \right)^\intercal \gamma
				\right)\\
				&\quad =\frac{2b_3}{n-1} \sum_{i=1}^n (X_i^\intercal \gamma)^2  > \frac{2b_3}{n} \sum_{i=1}^n (X_i^\intercal \gamma)^2,
			\end{aligned}
		\end{equation*}
		where the last equality follows from $\sum_{i=1}^n X_i = \mathbf 0$ in Assumption~\ref{assumption_X}. Applying Assumption~\ref{assumption_X} then yields 
		\begin{equation} \label{ineq:bound1}
			\inf\limits_{\gamma\in\Gamma^*}\{ \bar{L}_0(\gamma) - \bar{L}_0(\mathbf 0)\}
			> \frac{2b_3}{n} \inf\limits_{\gamma\in\Gamma^*} \sum_{i=1}^n (X_i^\intercal \gamma)^2
			\geq  2b_2b_3\inf\limits_{\gamma\in\Gamma^*}\|\gamma\|_2^2= 2b_2b_3\Delta^2 c_{|\Omega|}^2 c_p^2{\cal E}^2(n,p).
		\end{equation} 
			
			\textbf{Bound for $\eta(\Gamma^*)$ in part II. }
			By the bounded difference inequality \citep{McDiarmid_1989}, for any $\delta>0$, we have
			\begin{equation}
				\text{pr}\left(\eta(\Gamma^*)- E\left( \eta(\Gamma^*)\right) >  \delta\right) 
				\leq \exp\left(-\frac{2\delta^2}{ \sum_{l=1}^n|\eta(\Gamma^*)-\eta^{l}(\Gamma^*)|^2 }\right),\label{eq: delta_gamma}
			\end{equation}
			where $\eta^l(\Gamma^*)$ is the value of $\eta(\Gamma^*)$ when the $l$-th observation $\epsilon_l$ is replaced by $\varepsilon_l$, with all other observations held fixed. In this case, the loss function $L_0(\cdot)$ is updated to $L^l_0(\cdot)$. For any $\gamma\in \Gamma^*$, the function $h_{\gamma}(\epsilon_i, \epsilon_j)$ is similarly affected and replaced by $h^l_{\gamma}(\epsilon_i, \epsilon_j)$, where $h^l_{\gamma}(\epsilon_i, \epsilon_j) = h_{\gamma}(\varepsilon_l, \epsilon_j)$ if $i = l$, or $h_{\gamma}(\epsilon_i, \varepsilon_l)$ if $j = l$; otherwise, it remains unchanged. Then we have 
			that for each $l\in[n]$,
			\begin{equation*}\begin{aligned}
					&|\eta(\Gamma^*) - \eta^l(\Gamma^*)| \leq  \sup\limits_{\gamma\in\Gamma^*}\left|L_0(\gamma) - L_0(\mathbf 0)- L^l_0(\gamma) + L^l_0(\mathbf 0) \right|\\
					&\quad =   \frac{1}{n(n-1)} \sup\limits_{\gamma\in\Gamma^*}\sum_{i=1}^{n}\sum\limits_{j\ne i}
					\left| h_{\gamma}(\epsilon_i, \epsilon_j)  -   h^l_{\gamma}(\epsilon_i, \epsilon_j)\right|
					=\frac{2}{n(n-1)}\sup\limits_{\gamma\in\Gamma^*}\sum_{j \neq l}
					|h_{\gamma}(\epsilon_l, \epsilon_j)  -   h_{\gamma}(\varepsilon_l, \epsilon_j)| \\
					&\quad \leq \frac{2}{n(n-1)}\sup\limits_{\gamma\in\Gamma^*}4(n-1)b_1\bar{c}c_{|\Omega|}c_p\|\gamma\|_2 \leq  \frac{8b_1 c_0 \bar{c}\Delta}{n}c^2_{|\Omega|}c^2_p{\cal E}(n,p) ,
			\end{aligned}\end{equation*}
			where the first inequality follows from the triangle inequality, the second inequality comes from  \eqref{ineq:|h|}, and the last inequality comes from the definition of $\Gamma^*$ and the fact that $c_0>1$. By taking $\delta =  4\sqrt{2}b_1c_0\bar{c}\Delta c_{|\Omega|}^2 c_p^2{\cal E}^2(n,p)$ in \eqref{eq: delta_gamma}, we obtain
			\begin{equation}\label{ineq:pro_wn-w0}
				\operatorname{pr}\left(\eta(\Gamma^*)-E\left( \eta(\Gamma^*)\right) > 4\sqrt{2}b_1c_0 \bar{c}\Delta c_{|\Omega|}^2 c_p^2{\cal E}^2(n,p)\right) \leq \exp\left(-n {\cal E}^2(n,p)\right).
			\end{equation}
			
			Then we proceed to bound $E\left(  \eta(\Gamma^*)\right)$. Define $M_n$ to be the smallest integer not less than $n/2$. For any $\gamma\in \Gamma^*$, we can see that
			$$
			L_0(\gamma)-L_0(\mathbf 0)=   \frac{1}{n}\sum_{i=1}^{n}\frac{1}{n-1}\sum\limits_{j\ne i} h_{\gamma}(\epsilon_i, \epsilon_j)  = \frac{1}{n!}\sum_{\pi\in P_n}M_n^{-1}\sum_{i=1}^{M_n}h_{\gamma}(\epsilon_{\pi(i)},\epsilon_{\pi(M_n+i)}),
			$$
			where $P_n$ denotes the set of all permutations of $[n]$.
			It then follows that 
			\begin{equation*}\begin{aligned} 
					&E\left(  \eta(\Gamma) \right) = E\left( \sup_{\gamma\in\Gamma^*}\frac{1}{n!}\left|
					\sum_{\pi\in P_n}M_n^{-1}\sum_{i=1}^{M_n} \left( h_{\gamma}(\epsilon_{\pi(i)},\epsilon_{\pi(M_n+i)})- E\left( h_{\gamma}(\epsilon_{\pi(i)},\epsilon_{\pi(M_n+i)}) \right)  \right)\right|\right)\\
					&\qquad =  E_{\epsilon}\left( \sup_{\gamma\in\Gamma^*}\frac{1}{n!}\left|E_{\epsilon'} \left(
					\sum_{\pi\in P_n}M_n^{-1}\sum_{i=1}^{M_n} \left( h_{\gamma}(\epsilon_{\pi(i)},\epsilon_{\pi(M_n+i)})- h_{\gamma}(\epsilon'_{\pi(i)},\epsilon'_{\pi(M_n+i)})\right) \right)   \right| \right) \\
					&\qquad \leq \frac{1}{n!}\sum_{\pi\in P_n} E_{\epsilon,\epsilon'}\left( \sup_{\gamma\in\Gamma^*}  \left|
					M_n^{-1}\sum_{i=1}^{M_n} \left( h_{\gamma}(\epsilon_{\pi(i)},\epsilon_{\pi(M_n+i)})- h_{\gamma}(\epsilon'_{\pi(i)},\epsilon'_{\pi(M_n+i)})\right)    \right|\right).  
			\end{aligned}\end{equation*} 
			Let $\{\sigma_i\}_{i=1}^n$ be an i.i.d. Rademacher sequence, where each $\sigma_i$ takes values $ \pm 1$ with equal probability.  Since the distribution of $h_{\gamma}(\epsilon_{\pi(i)},\epsilon_{\pi(M_n+i)})-h_{\gamma}(\epsilon'_{\pi(i)},\epsilon'_{\pi(M_n+i)})$ is symmetric, the random variables
			\begin{equation*}
				\sigma_i \left[ h_{\gamma}(\epsilon_{\pi(i)},\epsilon_{\pi(M_n+i)})- h_{\gamma}(\epsilon'_{\pi(i)},\epsilon'_{\pi(M_n+i)})\right]
				\text{~and~} h_{\gamma}(\epsilon_{\pi(i)},\epsilon_{\pi(M_n+i)})- h_{\gamma}(\epsilon'_{\pi(i)},\epsilon'_{\pi(M_n+i)})
			\end{equation*}
			share the same distribution. Therefore,
			\begin{equation*}\begin{aligned}  
					E\left( \eta(\Gamma^*) \right) &\leq \frac{1}{n!}\sum_{\pi\in P_n} E_{\epsilon,\epsilon',\sigma}\left( \sup_{\gamma\in\Gamma^*}  \left|
					M_n^{-1}\sum_{i=1}^{M_n} \sigma_i \left( h_{\gamma}(\epsilon_{\pi(i)},\epsilon_{\pi(M_n+i)})- h_{\gamma}(\epsilon'_{\pi(i)},\epsilon'_{\pi(M_n+i)})\right)    \right|\right)\\
					& \leq  \frac{2}{n!}\sum_{\pi\in P_n} E_{\epsilon,\sigma}\left( \sup_{\gamma\in\Gamma^*}  \left|
					M_n^{-1}\sum_{i=1}^{M_n} \sigma_i h_{\gamma}(\epsilon_{\pi(i)},\epsilon_{\pi(M_n+i)})   \right|\right),
			\end{aligned}\end{equation*} 
			where the last inequality follows from the triangle inequality.  
			Since $h_{\gamma}(\epsilon_i,\epsilon_j)$  is a contraction mapping satisfying $|h_{\gamma}(\epsilon_i,\epsilon_j)|\leq |{(X_i-X_j)}^{\intercal}\gamma|$,  an application of the comparison properties for Rademacher averages \citep[Theorem 4.12]{Ledoux_1991} yields
			\begin{equation}
				\begin{aligned}
					& E\left( \eta(\Gamma^*) \right) \leq  \frac{4}{n!}\sum_{\pi\in P_n}  E_{\sigma}\left( \sup_{\gamma\in\Gamma^*}\left| M_n^{-1}\sum_{i=1}^{M_n} \sigma_i
					(X_{\pi(i)}-X_{\pi(M_n+i)})^{\intercal}\gamma \right|\right)\\
					&\qquad \leq \frac{4}{n!}\sum_{\pi\in P_n} E_{\sigma}\left( \Psi^d\left(M_n^{-1}\sum_{i=1}^{M_n} 
					\sigma_i (X_{\pi(i)}-X_{\pi(M_n+i)}   )\right)   \sup_{\gamma\in\Gamma^*}\Psi(\gamma )   \right) \\
					&\qquad \leq \frac{4\bar{c}}{n!}\Delta c_{|\Omega|}^2 c_p^2 {\cal E}(n,p) \sum_{\pi\in P_n}E_{\sigma}\left(\left\|M_n^{-1}\sum_{i=1}^{M_n} 
					\sigma_i (X_{\pi(i)}-X_{\pi(M_n+i)}   )\right\|_{\infty}\right),\label{ineq:EW_n}
				\end{aligned}
			\end{equation}
			where the last inequality follows from \eqref{ineq:Psi<gamma2} and \eqref{ineq:co}. By applying \citet[Lemma 14.12]{B_hlmann_2011}, we obtain that  
			\begin{equation*}
				\begin{aligned}
					& E_{\sigma}\Big(\max_{1\leq j\leq p}\Big|M_n^{-1}\sum_{i=1}^{M_n}\sigma_i(X_{\pi(i),j} 
					-X_{\pi(M_{n}+i),j})\Big|\Big)  \leq   2b_1M_n^{-1}\log(p+1)  + \sqrt{2M_n^{-1}\log(p+1)} \\ 
					& \qquad \leq  5b_1\left( n^{-1}\log p  + \sqrt{n^{-1}\log p}\right) = 5b_1 {\cal E}(n,p). 
				\end{aligned}
			\end{equation*}
			Combining this bound with \eqref{ineq:EW_n}, we have 
			$$
			E\left( \eta(\Gamma^*) \right)\leq  20b_1 \bar{c} \Delta c_{|\Omega|}^2 c_p^2{\cal E}^2(n,p) \leq 20b_1 c_0\bar{c} \Delta c_{|\Omega|}^2c_p^2{\cal E}^2(n,p) ,
			$$
			which together with \eqref{ineq:pro_wn-w0}, implies
			\begin{equation}\label{ineq:bound2}
				\text{pr}\left( \eta(\Gamma^*) \leq (20+4\sqrt{2})  b_1 c_0 \bar{c}\Delta  c_{|\Omega|}^2 c_p^2{\cal E}^2(n,p) \right)
				\geq 1-  \exp\left( -n {\cal E}^2(n,p) \right).
			\end{equation}
			
			\textbf{Bound for $ \lambda^* \sup\limits_{\gamma\in\Gamma^*} \{|\Psi(\gamma+\beta^*)- \Psi(\beta^*)|\}$ in part III. } 
			It can be seen that
			\begin{equation}
				\lambda^* \sup_{\gamma\in\Gamma^*} \{ |\Psi(\gamma+\beta^*) - \Psi(\beta^*)|  \}
				\leq  \lambda^* \sup_{\gamma\in\Gamma^*} \Psi(\gamma) 
				\leq  \lambda^* \bar{c} c_{|\Omega|} \sup_{\gamma\in\Gamma^*} \|\gamma\|_2 
				=  \lambda^* \bar{c} \Delta  c_{|\Omega|}^2 c_p 
				{\cal E}(n,p), 
				\label{eq: bound3}
			\end{equation}
			where the second inequality follows from \eqref{ineq:Psi<gamma2}. By setting $t =8 b_1 c_p{\cal E}(n,p)$ in Lemma \ref{lemma:lambda},   we have 
			\begin{equation*}
				\text{pr}\left(\Psi^d(S_n)\geq 8 b_1 c_p{\cal E}(n,p) \right) 
				\leq 2p\exp\left(- 2 n{\cal E}^2(n,p)\right)
				\leq 2p\exp\left(-2\log p\right) 
				= 2/p \leq  \alpha_0.
			\end{equation*}
			By the definition of $\lambda^*$ in \eqref{eq: opt_lambda}, we have 
			$ \lambda^* \leq 8 b_1 c_0 c_p{\cal E}(n,p)$,
			which along with \eqref{eq: bound3} indicates that
			\begin{equation}\label{ineq:bound3}
				\lambda^* \sup_{\gamma\in\Gamma^*} \{ |\Psi(\gamma+\beta^*) - \Psi(\beta^*)|  \}
				\leq   8 b_1 c_0 \bar{c} \Delta c_{|\Omega|}^2 c_p^2 {\cal E}^2(n,p) .
			\end{equation}
			
			Finally, by combining  \eqref{ineq:bound1}, \eqref{ineq:bound2}, \eqref{ineq:bound3}, and \eqref{ineq:F-L0 with l2norm}, we can see the following inequality 
			\begin{equation*}
				\begin{aligned}
					& \inf_{\gamma\in\Gamma^*}\{F_{\lambda^*}(\gamma)-F_{\lambda^*}(0)\}
					> \left( 2b_2b_3\Delta  -  (20+4\sqrt{2}) b_1 c_0\bar{c}-8 b_1 c_0 \bar{c} \right)
					\Delta c_{|\Omega|}^2 c_p^2 {\cal E}^2(n,p)\\
					&\qquad =  \left( 34 b_1 c_0 \bar{c} -  (28+4\sqrt{2}) b_1 c_0\bar{c}\right) 
					\Delta c_{|\Omega|}^2 c_p^2 {\cal E}^2(n,p)>0,
				\end{aligned}
			\end{equation*}
			holds with probability at least $1-  \exp\left( -n {\cal E}^2(n,p) \right)$. Based on  \eqref{ineq:probilibity}, we conclude that
			$$
			\text{pr}\left(\|\hat{\beta}(\lambda^*)-\beta^*\|_2 
			\leq  \Delta c_{|\Omega|} c_p{\cal E}(n,p)\right)
			\geq  1-\exp\left( -n {\cal E}^2(n,p) \right)- \alpha_0.
			$$
			This completes the proof.
			\hfill~\halmos
		\end{proof}

		\section{Implementation Details of the Semismooth Newton Method}\label{sec:implementation}
		
		We recall that $L(\cdot)$ defined in \eqref{loss:Wilcoxon}, also known as the \emph{clustered Lasso regularizer}, has been studied in \cite{lin2019efficient}, where the proximal mapping and its generalized Jacobian were characterized. Building on these results, we present the following proposition.  
		For convenience, we define the polyhedral cone $
		\mathcal{D} := \{ s \in \mathbb{R}^n \mid \Delta s \geq 0 \}$, where $\Delta \in \mathbb{R}^{(n-1)\times n}$ is the difference operator given by $\Delta s = (s_1 - s_2, \dots, s_{n-1} - s_n)^\intercal$ for any $s\in \mathbb{R}^n$.
		\begin{proposition}\label{prop:Jacobian-L}
			Given any $s \in \mathbb{R}^n$, let $s^\downarrow$ denote its nonincreasing rearrangement, with the permutation matrix $P_s$ such that $s^\downarrow = P_s s$. Then the following results hold.
			\begin{itemize}
				\item[(i)] The proximal mapping of $L(\cdot)$ can be computed by
				\begin{equation*}
					\mathrm{Prox}_{L}(s) = P_s^\intercal \Pi_{\mathcal{D}}(P_s s - \varrho),
				\end{equation*}
				where $\varrho_k = \dfrac{2n - 4k + 2}{n(n-1)}$ for $k \in [n]$, and  $\Pi_{\mathcal{D}}$ is the projection onto the cone $\mathcal{D}$, computable via the pool-adjacent-violators algorithm \citep{Best1990}. 
				\item[(ii)] For $x\in \mathbb{R}^n$, we define the active set $\mathcal{I}_{\mathcal{D}}(x) := \left\{ i \in [n-1] \mid (\Delta \Pi_{\mathcal{D}}(x))_i = 0 \right\}$, the multiplier set $\mathcal{M}_{\mathcal{D}}(x) = \left\{ \lambda \in {\mathbb{R}^{n-1}_- } \;\middle|\;\Delta^\intercal \lambda = x - \Pi_{\mathcal{D}}(x), \;\lambda_i = 0 \ \text{for } i\notin \mathcal{I}_{\mathcal{D}}(x) \right\}$, and the index set
				\[
				\mathcal{K}_{\mathcal{D}}(x) := \left\{ K \subseteq [n-1] \mid \exists \lambda \in \mathcal{M}_{\mathcal{D}}(x) \text{ s.t. } \operatorname{supp}(\lambda) \subseteq K \subseteq \mathcal{I}_{\mathcal{D}}(x) \right\}.
				\]
				Then we have the following generalized Jacobian of $\mathrm{Prox}_{L}(s)$:
				\[
				\partial \mathrm{Prox}_{L}(s) = \left\{ \Lambda \in \mathbb{R}^{n \times n} \middle\vert \Lambda = I_n - P_s^{\intercal} \Delta_{K}^{\intercal} (\Delta_{K} \Delta_{K}^{\intercal})^{-1} \Delta_{K} P_s,~ K \in \mathcal{K}_{\mathcal{D}}(P_s s - \varrho) \right\},
				\]
				where $\Delta_{K}$ is the matrix consisting of the rows of $\Delta$ indexed by $K$.
				\item[(iii)] $\mathrm{Prox}_{L}(\cdot)$ is strongly semismooth with respect to $\partial \mathrm{Prox}_{L}(\cdot)$.
				\item[(iv)] Define $\theta\in \mathbb{R}^{n-1}$ as $\theta_i = 1$ if $i \in \mathcal{I}_{\mathcal{D}}(P_s s - \varrho)$, and $0$ otherwise. Consider the block decomposition of $\theta$ into $N$ consecutive blocks: 
				$ \theta =(\theta^{(1)}; \dots ;\theta^{(N)})$,     where each $\theta^{(i)} =0_{n_i}$ or $1_{n_i}$, with $\sum_{i=1}^N n_i = n - 1$ and adjacent blocks taking different values. Define the active block index set $ \mathcal{J}^s = \{i\in[N] \mid \theta^{(i)} = 1_{n_i}\}$, listed in increasing order with its $j$th element denoted as $\mathcal{J}^s_j$.
				Then we can construct: 
				\[
				\Lambda(s) = {\rm Diag}(P_s^{\intercal}\theta_s) + P_s^{\intercal} \Gamma_s \Gamma_s^{\intercal} P_s \in \partial \mathrm{Prox}_{L}(s),
				\]
				where $\theta_s = (\theta_s^{(1)}; \dots ;\theta_s^{(N)})  \in \mathbb{R}^{n}$ is defined as:
				\[
				\theta_s^{(i)} = 
				\begin{cases} 
					\mathbf 0_{n_i+1} & \text{if } i \in \mathcal{J}^s, \\
					\mathbf 1_{n_i} & \text{if } i\in \{1,N\}\setminus\mathcal{J}^s , \\
					\mathbf 1_{n_i-1} & \text{otherwise},
				\end{cases}
				\]
				and the matrix $\Gamma_s \in \mathbb{R}^{n \times |\mathcal{J}^s|}$ is defined by:
				\[
				(\Gamma_s)_{\ell,j} = 
				\begin{cases} 
					\dfrac{1}{\sqrt{n_{\mathcal{J}^s_j}+1}} & \text{if } \sum\limits_{t=1}^{\mathcal{J}^s_j-1} n_t + 1 \leq \ell \leq \sum\limits_{t=1}^{\mathcal{J}^s_j} n_t, \\[2.5ex]
					0 & \text{otherwise},
				\end{cases}\quad l\in [n], j\in \left[|\mathcal{J}^s|\right].
				\]					
			\end{itemize}
		\end{proposition}
		
		For the group Lasso regularizer $\Psi(\cdot)$, both the proximal operator and its generalized Jacobian admits a separable structure across groups. This result is stated in the following proposition, adapted from \cite{Zhang2020}.	
		\begin{proposition}\label{prop:Jacobian-P}
			For any $\beta\in\mathbb{R}^{p}$,   the following results hold.
			\begin{itemize}
				\item[(i)] The  proximal operator of the group Lasso regularizer $\Psi(\cdot)$ satisfies:
				\[
				\left({\rm Prox}_{\Psi}(\beta)\right)_{\mathcal{G}_l} =  {\rm Prox}_{ w_l\|\cdot\|_2}(\beta_{{\cal G}_l})= \beta_{{\cal G}_l}- \Pi_{\bar{\mathcal{B}}_2^{ w_l}}(\beta_{{\cal G}_l}),\quad l\in [g],
				\]
				where $\Pi_{\bar{\mathcal{B}}_{2}^{ w_l}}$ is the projection onto the closed $\ell_2$ ball $\bar{\mathcal{B}}_{2}^{ w_l} \coloneq \{\varpi \mid \|\varpi\|_2 \leq w_l\}$, computable as
				$\Pi_{\bar{\mathcal{B}}_{2}^{ w_l}}(\varpi)  = \varpi/\max\left\{1, \|\varpi\|_2/w_l\right\}$.
				\item[(ii)] For $l\in[g]$, define the linear operator ${\cal P}_l : \mathbb{R}^p \rightarrow \mathbb{R}^{|{\cal G}_l|}$ by ${\cal P}_l x = x_{{\cal G}_l}$ and set ${\cal P}:= ({\cal P}_1;\dots;{\cal P}_g)$. Then  the  Clarke generalized Jacobian  of ${\rm Prox}_{\Psi}(\beta)$ is: 
				\[
				\partial {\rm Prox}_{\Psi} (\beta) = \left\{  I_p - {\cal P}^{\intercal}\Sigma {\cal P} \ \middle| \ \Sigma = {\rm Diag}(\Sigma_1, \dots, \Sigma_g),\ \Sigma_l \in \partial \Pi_{\bar{\mathcal{B}}_2^{ w_l}}(\beta_{{\cal G}_l}),\ l\in[g] \right\},
				\]
				where $$
				\partial \Pi_{\bar{\mathcal{B}}_2^{ w_l}}(\varpi) = 
				\begin{cases} 
					\frac{ w_l}{\|\varpi\|_2} \left( I_{|\mathcal{G}_l|} - \frac{ \varpi \varpi ^{\intercal} }{\|\varpi\|_2^2} \right),  & \text{if } \|\varpi\|_2 > w_l,\\
					\left\{ I_{|\mathcal{G}_l|} - t \frac{\varpi \varpi^{\intercal}}{( w_l)^2} \ \middle| 0\leq t\leq 1 \right\},  & \text{if } \|\varpi\|_2 = w_l,\\
					I_{|\mathcal{G}_l|}, & \text{otherwise}.
				\end{cases}
				$$
				\item[(iii)]  ${\rm Prox}_{\Psi}(\cdot)$ is strongly semismooth with respect to  $\partial{\rm Prox}_{\Psi}(\cdot)$.
				\item[(iv)] We construct $V(\beta)= \sum_{ l\in \mathcal{R}^{\beta}} V_l \in \partial {\rm Prox}_{\Psi} (\beta)$, where $\mathcal{R}^{\beta} :=\{ l\in [g]\mid  \|\beta_{{\cal G}_l}\|_2 >  w_l \}$, and
				$$
				V_l =
				\left( \mathbf 1 - \frac{ w_l}{\|\beta_{{\cal G}_l}\|_2}   \right) {\rm Diag}(\eta^l)  +  \frac{ w_l}{\|\beta_{{\cal G}_l}\|_2^3} \left( {\cal P}_l^{\intercal}\beta_{{\cal G}_l}\right) \left( {\cal P}_l^{\intercal}\beta_{{\cal G}_l}\right)^{\intercal}, \quad l\in \mathcal{R}^\beta, 
				$$
				with $\eta^l\in \mathbb{R}^p$ defining as $(\eta^l)_i = 1$ if $i\in \mathcal{G}_l$ and $0$ otherwise. 
			\end{itemize}
		\end{proposition}
		
		With the established results in Propositions~\ref{prop:Jacobian-L} and \ref{prop:Jacobian-P}, we now present the computational details for solving the Newton equation \eqref{eq:Newton}. For brevity, we write $\tilde{s} = s^k/\sigma_k + z^{(j)}$,  $\tilde{\beta} = \beta^k/(\sigma_k\lambda) - X^{\intercal}z^{(j)}/\lambda$. By taking 
		\begin{equation*}
			\Lambda^{(j)} = \Lambda(\tilde{s})  \in \partial{\rm Prox}_L(\tilde{s}),\qquad V^{(j)}  = V(\tilde{\beta}) \in {\rm Prox}_{\Psi}(\tilde{\beta}),
		\end{equation*}
		the linear operator $H^{(j)}$ in \eqref{eq:Newton} becomes
		\begin{equation*}
			\begin{aligned}
				H^{(j)} &= \frac{\tau}{\sigma_k} I_n+\sigma_k\Lambda(\tilde{s}) + \sigma_k XV(\tilde{\beta}) X^{\intercal} \\
				& =\frac{\tau}{\sigma_k} I_n + \sigma_k\left({\rm Diag}({\cal P}_{\tilde{s}}^{\intercal}\theta_{\tilde{s}}) + {\cal P}_{\tilde{s}}^{\intercal} \Gamma_{\tilde{s}} \Gamma_{\tilde{s}}^{\intercal} P_{\tilde{s}} \right) \\
				&\quad +\sigma_k \sum_{l\in \mathcal{R}^{\tilde{s}}} X \left(\left( 1 - \frac{ w_l}{\|\tilde{\beta}_{{\cal G}_l}\|_2}   \right) {\rm Diag}(\eta^l)  +  \frac{ w_l}{\|\tilde{\beta}_{{\cal G}_l}\|_2^3} \left( {\cal P}_l^{\intercal}\tilde{\beta}_{{\cal G}_l}\right) \left( {\cal P}_l^{\intercal}\tilde{\beta}_{{\cal G}_l}\right)^{\intercal} \right)X^{\intercal}.
			\end{aligned}
		\end{equation*}
		Noting the nature of the linear operators $\mathcal{P}_l$, $l\in [g]$ and the $0$-$1$ structure of the vectors $\eta^l$, $l\in \mathcal{R}^{\tilde{s}}$, we can see that
		\begin{equation}
			H^{(j)} = \frac{\tau}{\sigma_k} I_n +\sigma_k {\rm Diag}(P_{\tilde{s}}^{\intercal}\theta_{\tilde{s}}) + \sigma_k [\Theta, \Xi,\Upsilon] \left[\begin{matrix}
				\Theta^{\intercal}  \\
				\Xi^{\intercal}\\
				\Upsilon^{\intercal}
			\end{matrix}\right],\label{eq: one_H}
		\end{equation}
		where $\Theta = {\cal P}_{\tilde{s}}^{\intercal} \Gamma_{\tilde{s}}\in \mathbb{R}^{n\times |\mathcal{J}^{\tilde{s}}|}$, and
		\begin{equation*}
			\begin{aligned}
				\Xi &= [\Xi_1,\dots, \Xi_{|\mathcal{R}^{\tilde{\beta}}|}]\in\mathbb{R}^{n\times \sum_{l\in \mathcal{R}^{\tilde{\beta}}} |\mathcal{G}_l|} \quad \text{with} \quad \Xi_l = \sqrt{  1 - \frac{ w_l}{\|\tilde{\beta}_{{\cal G}_l}\|_2} } X_l \in\mathbb{R}^{n\times |{\cal G}_l|}, \quad l\in \mathcal{R}^{\tilde{\beta}},\\
				\Upsilon &= [\Upsilon_1,\dots, \Upsilon_{|\mathcal{R}^{\tilde{\beta}}|}]\in\mathbb{R}^{n\times |\mathcal{R}^{\tilde{\beta}}|} \qquad \text{with} \quad \Upsilon_l = \sqrt{\frac{ w_l}{\|\tilde{\beta}_{{\cal G}_l}\|_2^3}} \left( X_l\tilde{\beta}_{{\cal G}_l}\right) \in\mathbb{R}^{n}, \quad l\in \mathcal{R}^{\tilde{\beta}},
			\end{aligned}
		\end{equation*}
		with $X_l\in \mathbb{R}^{n\times |\mathcal{G}_l|}$ being the submatrix of $X$ with those columns in $\mathcal{G}_l$.
		
		Then the computational complexity of solving the linear system \eqref{eq:Newton} with coefficient matrix $H^{(j)}$ in \eqref{eq: one_H} can be separated into two stages. In the first stage, the required components in $H^{(j)}$ are prepared: forming  the vector $P_{\tilde{s}}^{\intercal}\theta_{\tilde{s}}$ and the matrix $\Theta$ requires no arithmetic cost as they can be obtained by reordering, and computing the matrices $\Xi$ and $\Upsilon$ requires $\mathcal{O}(n\sum_{l\in \mathcal{R}^{\tilde{\beta}}}|\mathcal{G}_l|)$ operations. In the second stage, the system \eqref{eq:Newton} can be solved by several alternative strategies, depending on the relative values of $n$ and 
		$$
		r:=|\mathcal{J}^{\tilde{s}}| + \sum_{l\in \mathcal{R}^{\tilde{\beta}}} |\mathcal{G}_l| + |\mathcal{R}^{\tilde{\beta}}|,
		$$
		which characterizes the problem sparsity: 
		$|\mathcal{J}^{\tilde{s}}|$ reflects the sparsity induced by the rank loss, 
		$\sum_{l \in \mathcal{R}^{\tilde{\beta}}} |\mathcal{G}_l|$ corresponds to the solution sparsity, 
		and $|\mathcal{R}^{\tilde{\beta}}|$ captures the group sparsity. The most direct approach is to explicitly form $H^{(j)}$, which costs $\mathcal{O}(n^2r)$, followed by a direct solver with complexity $\mathcal{O}(n^3)$. When $n$ is large, a more suitable option is the CG method, where each matrix-vector multiplication $H^{(j)}d$ requires $\mathcal{O}(nr)$ operations. If $r \ll n$, the low-rank structure of $H^{(j)}$ can be further exploited to solve the linear system through the Sherman-Morrison-Woodbury formula, leading to an overall complexity of $\mathcal{O}(nr^2 + r^3)$. Overall, the sparsity structure encoded in $r$ plays a key role in reducing computational cost and enables the Newton system to be solved efficiently.

\end{document}